\documentclass[10pt,twocolumn,letterpaper]{article}

\usepackage{cvpr}
\usepackage{times}
\usepackage{epsfig}
\usepackage{graphicx}
\usepackage{amsmath}
\usepackage{amsthm}
\usepackage{amssymb}
\usepackage{multicol}
\usepackage{lipsum}
\usepackage{placeins}
\usepackage{multirow}

\usepackage{bm}
\usepackage{float}
\usepackage[caption = false]{subfig}
\newtheorem{lemma}{Lemma}
\usepackage{soul}
\usepackage[square,sort,comma,numbers]{natbib}
\usepackage{algorithm}
\usepackage{algorithmic}
\usepackage{multirow}
\usepackage[utf8]{inputenc}
\usepackage[english]{babel}
\usepackage{soul}
\usepackage{color}
\usepackage{enumitem}
\usepackage{mathtools}

\DeclarePairedDelimiter{\abs}{\lvert}{\rvert}
\usepackage{tabularx} % for tables
\usepackage{rotating}
\newcolumntype{L}[1]{>{\hsize=#1\hsize\raggedright\arraybackslash}X}%
\newcolumntype{R}[1]{>{\hsize=#1\hsize\raggedleft\arraybackslash}X}%
\newcolumntype{C}[1]{>{\hsize=#1\hsize\centering\arraybackslash}X}%
\newcommand\RotText[1]{\rotatebox{90}{\parbox{2cm}{\centering#1}}}

\usepackage[pagebackref=true,breaklinks=true,letterpaper=true,colorlinks,bookmarks=false]{hyperref}

\cvprfinalcopy % *** Uncomment this line for the final submission

\def\cvprPaperID{6717} % *** Enter the CVPR Paper ID here

\begin{document}
	\newcommand{\bc}{\mathbf{c}}
	\newcommand{\bu}{\mathbf{u}}
	\newcommand{\br}{\mathbf{r}}
	\newcommand{\bx}{\mathbf{x}}
	\newcommand{\by}{\mathbf{y}}
	\newcommand{\bz}{\mathbf{z}}
	\newcommand{\bR}{\mathbf{R}}
	\newcommand{\bK}{\mathbf{K}}
	\newcommand{\bZ}{\mathbf{Z}}
	\newcommand{\bH}{\mathbf{H}}
	\newcommand{\bQ}{\mathbf{Q}}
	\newcommand{\rbox}{\mathbb{B}}
	\newcommand{\vel}{\bm{\omega}}
	\newcommand{\upbnd}{\overline{C}}
	\newcommand{\event}{\bm{e}}
	\newcommand{\cE}{\mathcal{E}} 
	\newcommand{\cF}{\mathcal{F}}
	\newcommand{\cD}{\mathcal{D}}
	\newcommand{\cL}{\mathcal{L}}
	\newcommand{\cG}{\mathcal{G}}
	\newcommand{\cM}{\mathcal{M}}
	\newcommand{\bM}{\mathbf{M}}
	\newcommand{\cT}{\mathcal{T}}
	\newcommand{\bT}{\mathbf{T}}
	\newcommand{\cV}{\mathcal{V}}
	\newcommand{\au}{\vec{\mathbf{u}}}
	\newcommand{\bA}{\mathbf{A}}
	\newcommand{\algorithmicbreak}{\textbf{break}}
	\newcommand{\BREAK}{\STATE \algorithmicbreak}

	\newtheorem{thm}{Theorem}[section]
	\newtheorem{lem}[thm]{Lemma}
	\newtheorem{prop}[thm]{Proposition}
	\newtheorem{cor}{Corollary}
	\newtheorem{conj}{Conjecture}[section]
	\newtheorem{defn}{Definition}[section]
	\newtheorem{exmp}{Example}[section]
	\newtheorem{rem}{Remark}
	\newtheorem{prob}{Problem}
	\newtheorem{theorem}{Theorem}
	
	\newcommand{\norm}[1]{\left\lVert#1\right\rVert}

	\graphicspath{{figure/}}
	
	%%%%%%%%% TITLE
	\title{Globally Optimal Contrast Maximisation for Event-based Motion Estimation}

	\author{Daqi Liu~~~~~~~Álvaro Parra~~~~~~~Tat-Jun Chin\\
		School of Computer Science, The University of Adelaide\\
		{\tt\small \{daqi.liu, alvaro.parrabustos, tat-jun.chin\}@adelaide.edu.au}
	}

	\maketitle

	\begin{abstract}
		Contrast maximisation estimates the motion captured in an event stream by maximising the sharpness of the motion-compensated event image. To carry out contrast maximisation, many previous works employ iterative optimisation algorithms, such as conjugate gradient, which require good initialisation to avoid converging to bad local minima. To alleviate this weakness, we propose a new globally optimal event-based motion estimation algorithm. Based on branch-and-bound (BnB), our method solves rotational (3DoF) motion estimation on event streams, which supports practical applications such as video stabilisation and attitude estimation. Underpinning our method are novel bounding functions for contrast maximisation, whose theoretical validity is rigorously established. We show concrete examples from public datasets where globally optimal solutions are vital to the success of contrast maximisation. Despite its exact nature, our algorithm is currently able to process a $50,000$-event input in $\approx 300$ seconds (a locally optimal solver takes $\approx 30$ seconds on the same input). The potential for GPU acceleration will also be discussed.
	\end{abstract}
	
	\section{Introduction}
	
	By asynchronously detecting brightness changes, event cameras offer a fundamentally different way to detect and characterise physical motion. Currently, active research is being conducted to employ event cameras in many areas, such as robotics/UAVs~\cite{delbruck2013robotic,mueggler2014event}, autonomous driving\cite{weikersdorfer2014event,mueggler2017event,vidal2018ultimate}, and spacecraft navigation~\cite{chin2019star,chin2019event}. While the utility of event cameras extends beyond motion perception, e.g., object recognition and tracking~\cite{ramesh2019dart,stoffregen2019event}, the focus of our work is on estimating visual motion using event cameras.
	
	Due to the different nature of the data, new approaches are required to extract motion from event streams. A recent successful framework is contrast maximisation (CM)~\cite{gallego2018unifying}. Given an event stream, CM aims to find the motion parameters that yield the sharpest motion-compensated event image; see Fig.~\ref{fig:event_image}. Intuitively, the correct motion parameters will align corresponding events, thereby producing an image with high contrast. We formally define CM below.
	
	\begin{figure}[t]\centering
		\begin{center}
			\subfloat[Event stream (w/o polarity).]{\label{fig:stream}\includegraphics[width=0.48\linewidth]{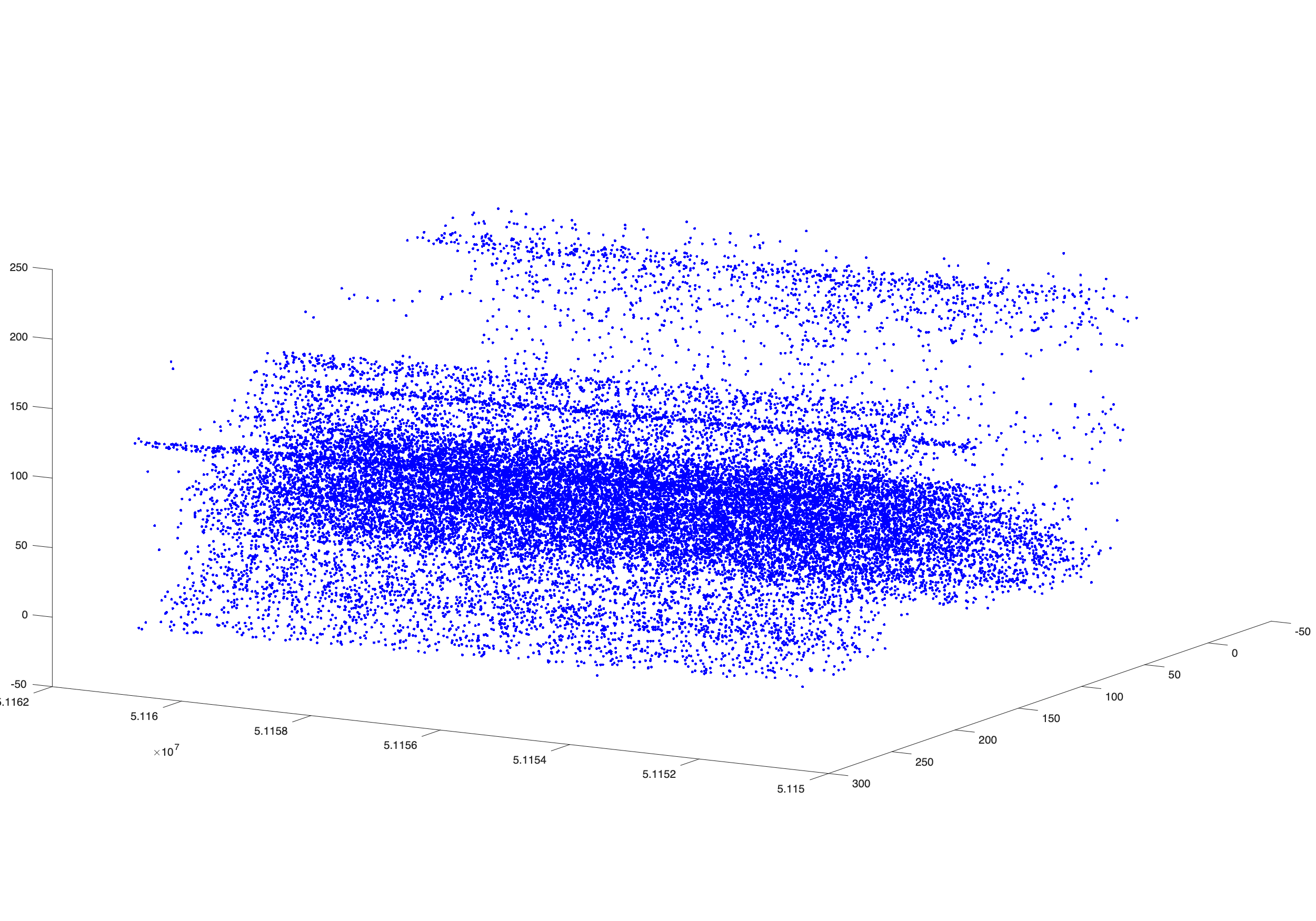}}
			\subfloat[Contrast = $0.9993$ (identity).]{\label{fig:unwarp}\includegraphics[width=0.48\linewidth]{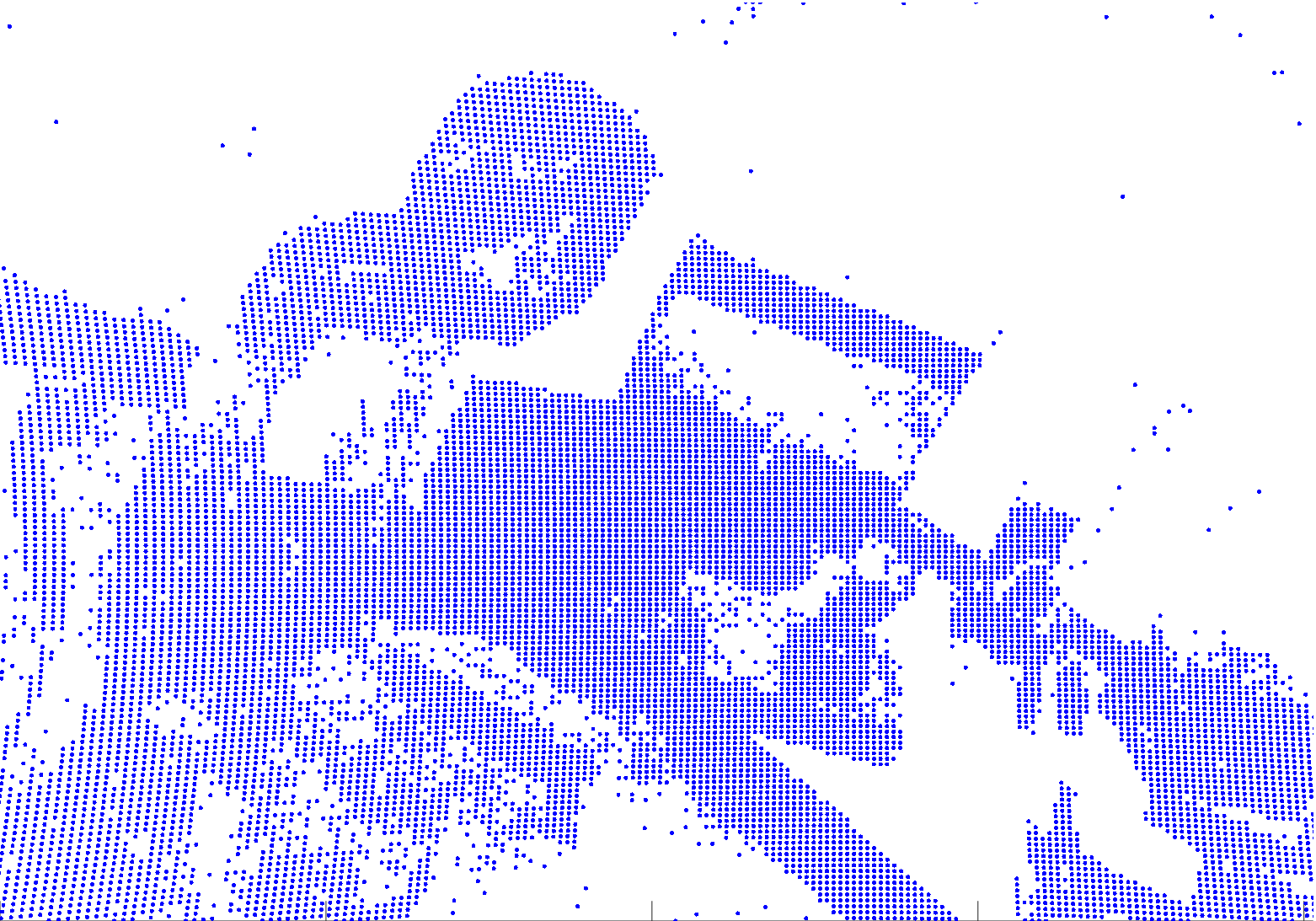}}\\
			\vspace{-1em}
			\subfloat[Contrast = $1.0103$ (local).]{\label{fig:warp_local}\includegraphics[width=0.48\linewidth]{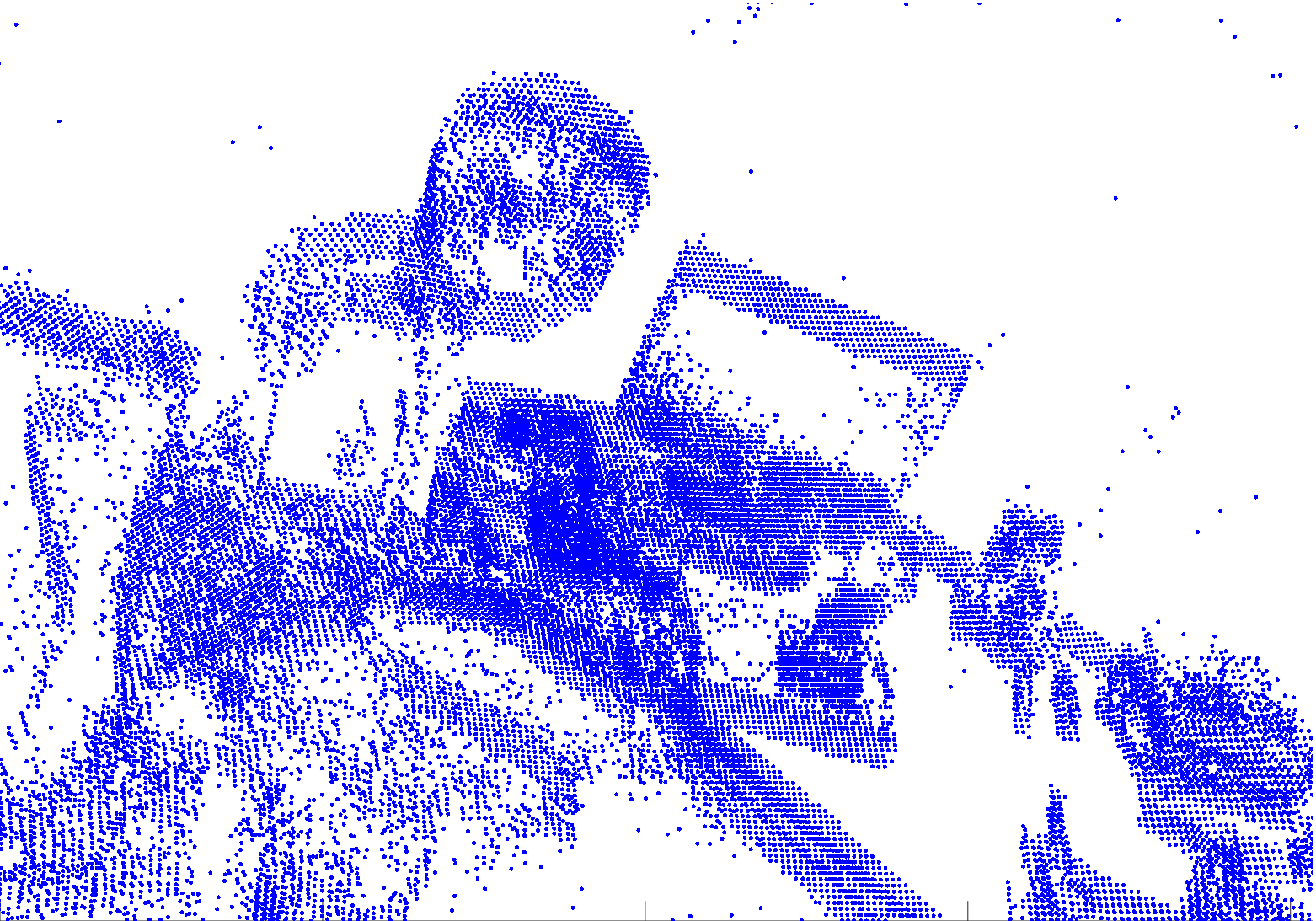}}
			\subfloat[Contrast = $1.9748$ (global).]{\label{fig:warp_global}\includegraphics[width=0.48\linewidth]{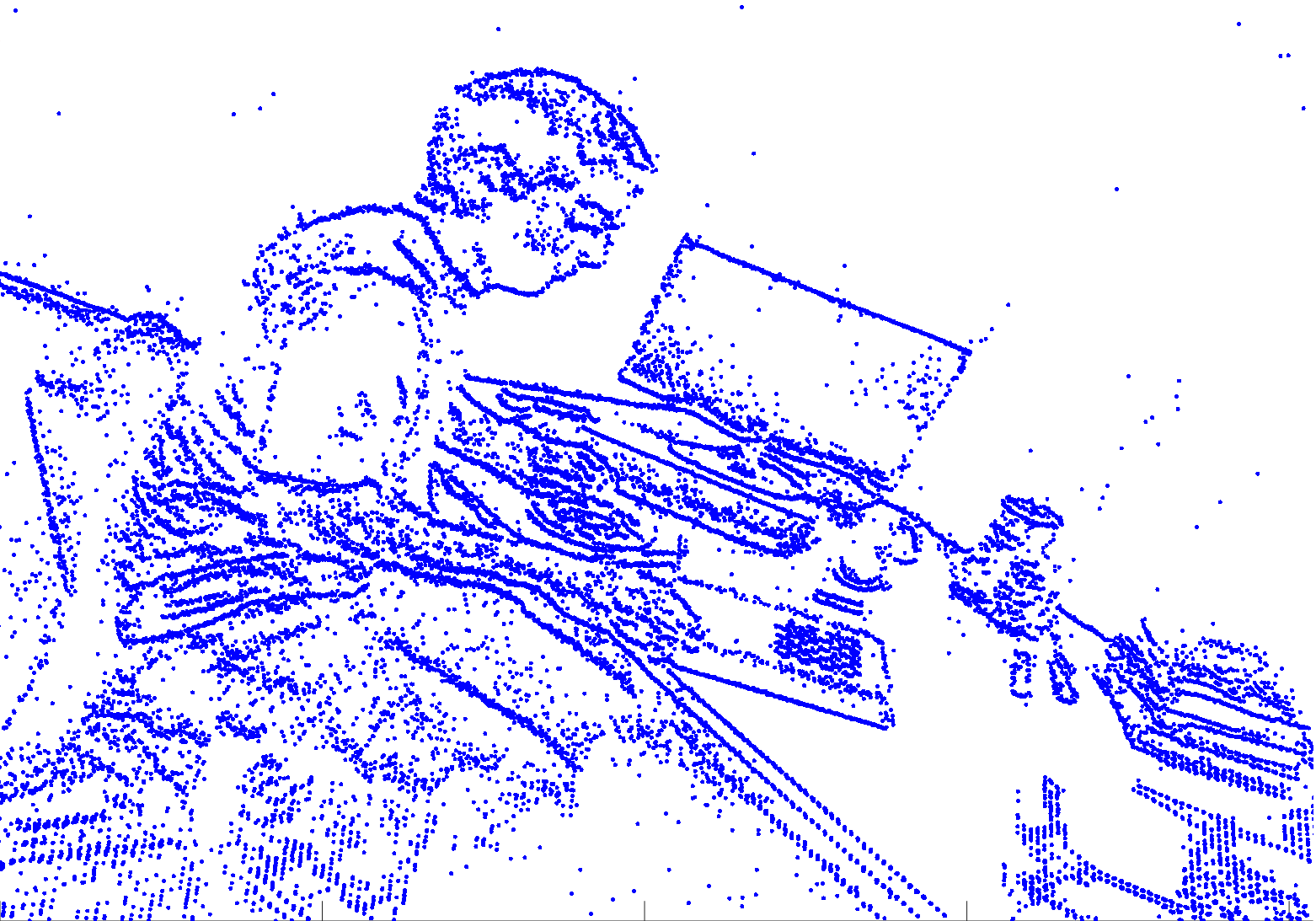}}\\
			\vspace{0.5em}
			\caption{(a) $10$ ms event stream under rotational motion~\cite{gallego2017accurate}. Since event polarity is not used in our work, the events are plotted in the same colour. (b) Event image without motion compensation (identity transformation). (c)(d) Event images produced with locally and globally optimal contrast maximisation.}
			\label{fig:event_image}
		\end{center}
	\end{figure}
	
	\vspace{-1em}
	\paragraph{Event image}
	
	Let $\cE = \{\event_i\}_{i}^N$ be an event stream recorded over time duration $\cT = [0,t_{\max}]$. Each event $\event_i = (\bu_i,t_i,p_i)$ contains an image position $\bu_i$, time stamp $t_i \in \cT$, and polarity $p_i\in\{-1,+1\}$. We assume $\cE$ was produced under camera motion $\cM$ over a 3D scene, thus each $\event_i$ is associated with a scene point that triggered the event.

	We parameterise $\cM$ by a vector $\vel \in \Omega$, and let $X = \{\bx_j\}_{j=1}^{P}$ be the centre coordinates of the pixels in the image plane of the event sensor. Under CM, the event image $H_c$ is defined as a function of $\vel$, and the intensity at pixel $\bx_j$ is
	\begin{equation}\label{eq:rt}
	H_c(\bx_j;\vel) = \sum_{i=1}^N\bm{\delta}(\bx_j-f(\bu_i, t_i ;\, \vel)),
	\end{equation}
	
	where $\bm{\delta}$ is a kernel function (e.g., Gaussian).  Following~\cite[Sec.~2.1]{gallego2018unifying}, we do not use event polarities in~\eqref{eq:rt}. $H_c$ is regarded to be captured at time $0$, and the function
	\begin{equation}\label{eq:warp}
	\bu'_i = f(\bu_i, t_i;\, \vel)
	\end{equation}
	warps $\bu_i$ to $\bu'_i$ in $H_c$ by ``undoing" the motion $\cM$ between time $0$ and $t_i$. Intuitively, $\bu'_i$ is the image position of the 3D scene point that triggered $\event_i$, if it was observed at time $0$.
	
	In practice, the region of support of the kernel $\bm{\delta}$ in~\eqref{eq:rt} is small w.r.t.~image dimensions, e.g., Gaussian kernels with bandwidth $\sigma = 1$ pixel were used in~\cite[Sec.~2]{gallego2018unifying}. This motivates the usage of ``discrete" event images
	\begin{equation}\label{eq:rt_discrete}
	H_d(\bx_j;\vel) = \sum_{i=1}^N \mathbb{I}(f(\bu_i, t_i ;\, \vel)~\text{lies in pixel}~\bx_j),
	\end{equation}
	where $\mathbb{I}$ returns $1$ if the input predicate is true, and $0$ otherwise. As we will show later, conducting CM using $H_c$ (with small bandwidth) and $H_d$ yields almost identical results.
	
	\vspace{-1em}
	\paragraph{Contrast maximisation}
	
	The \emph{contrast} of an event image $H$ (continuous or discrete) is the variance of its pixel values. Since $H$ depends on $\vel$, the contrast is also a function of $\vel$
	\begin{equation}\label{eq:c}
	C(\vel) =\dfrac{1}{P}\sum_{j=1}^P (H(\bx_j;\vel)-\mu(\vel))^2,
	\end{equation}
	where $\mu(\vel)$ is the mean intensity
	\begin{equation}
	\mu(\vel)=\dfrac{1}{P}\sum_{j=1}^P H(\bx_j;\vel).
	\end{equation}
	CM~\cite{gallego2018unifying} estimates $\cM$ by maximising the contrast of $H$, i.e.,
	\begin{align}\label{eq:cm}
	\max_{\vel \in \Omega}~~C(\vel).
	\end{align}
	The intuition is that the correct $\vel$ will allow $\cM$ to align events that correspond to the same scene points in $H$, thus leading to a sharp or high-contrast event image; see Fig.~\ref{fig:event_image}.
	
	\vspace{-1em}
	\paragraph{Global versus local solutions} By \emph{globally optimal} (or ``global") solution to CM~\eqref{eq:cm}, we mean $\vel^* \in \Omega$ such that
	\begin{align}
	C(\vel^*) \ge C(\vel) \;\; \forall \vel \in \Omega.
	\end{align}
	A solution $\hat{\vel}$ is \emph{approximate} if $C(\hat{\vel}) < C(\vel^*)$. Also, a solution is \emph{locally optimal} (or ``local") if it is the maximiser of its neighbourhood~\cite[Chap.~2]{nocedal_and_wright}. All global solutions are locally optimal, but the converse is not true.
	
	\subsection{Previous works}
	
	Previous works on CM (e.g.,~\cite{gallego2018unifying,gallego2019focus,stoffregen2019event,Stoffregen_2019_CVPR}) apply nonlinear optimisation (e.g., conjugate gradient) to solve~\eqref{eq:cm}. Given an initial solution $\vel^{(0)}$, the solution is successively updated until convergence to a locally optimal solution. In practice, if the local solution is a bad approximate solution, there can be significant differences in its quality compared to the global solution; see Fig.~\ref{fig:event_image}. This can occur when $\vel^{(0)}$ is too distant from good solutions, or $C(\vel)$ is too nonconcave (e.g., when $\bm{\delta}$ has a very small bandwidth). Thus, algorithms that can find $\vel^*$ are desirable.
	
	Recent improvements to CM include modifying the objective function to better suit the targeted settings~\cite{gallego2019focus,Stoffregen_2019_CVPR}. However, the optimisation work horse remains locally optimal methods. Other frameworks for event processing~\cite{censi2014low,kim2016real,kueng2016low,gallego2017event,kim2008simultaneous,chin2019event} conduct filtering, Hough transform, or specialised optimisation schemes; these are generally less flexible than CM~\cite{gallego2018unifying}. There is also active research in applying deep learning to event data~\cite{zhu2018ev,zhu2019unsupervised,ye2019learning,sekikawa2019eventnet}, which require a separate training phase on large datasets.
	
	\vspace{-1em}
	\paragraph{Contributions}
	We focus on estimating rotational motion from events, which is useful for several applications, e.g., video stabilisation~\cite{gallego2017accurate} and attitude estimation~\cite{chin2019star}.
	
	Specifically, we propose a BnB method for \emph{globally optimal} CM for rotation estimation. Unlike previous CM techniques, our algorithm does not require external initialisations $\vel^{(0)}$, and can guarantee finding the global solution $\vel^*$ to~\eqref{eq:cm}. Our core contributions are novel bounding functions for CM, whose theoretical validity are established. As we will show in Sec.~\ref{sec:results}, while local methods generally produce acceptable results~\cite{gallego2017accurate,gallego2018unifying}, they often fail during periods with fast rotational motions. On the other hand, our global method always returns accurate results.
	%\textcolor{red}{[AP: You should explain the fast rotation periods (e.g., last 15 seconds of \emph{poster}, \emph{boxes} and \emph{dynamic} sequences) somewhere in experiments.]}

	%\hl{Besides in general local methods for CM produce acceptable results on findings rotation parameters}~\cite{gallego2017accurate,gallego2018unifying}\hl{, they often fail during periods with fast rotation as we will show using event data from public datasets in Sec.}~\ref{sec:results}\hl{. Therefore, globally optimal solutions are vital to the success of CM.}
	
	%(e.g., last 15 seconds of \emph{poster}, \emph{boxes} and \emph{dynamic} sequences). Hence globally optimal solutions are necessary.
	
	%local optimal solution for fast rotation (e.g., last 15 seconds of \emph{poster}, \emph{boxes} and \emph{dynamic} sequences). Hence globally optimal solutions are necessary.
	
	\section{Rotation estimation from events}
	If duration $\cT$ is small (e.g., $t_{\max} = 10ms$), a fixed axis of rotation and constant angular velocity can be assumed for $\cM$~\cite{gallego2018unifying}. Following~\cite[Sec.~3]{gallego2018unifying}, $\cM$ can be parametrised as a $3$-vector $\vel$, where the direction of $\vel$ is the axis of rotation, and the length $\| \vel \|_2$ of $\vel$ is the angular rate of change. Between time $0$ and $t$, the rotation undergone is
	\begin{equation}
	\bR(t; \vel) = \text{exp}([\vel \, t]_\times),
	\end{equation}
	where $\vel t$ is the axis-angle representation of the rotation, $[\vel t]_\times$ is the skew symmetric form of $\vel t$, and $\text{exp}$ is the exponential map (see~\cite{wiki:exp} for details).
	
	Let $\bK$ be the $3 \times 3$ intrinsic matrix of the event camera ($\bK$ is known after calibration~\cite{zhang2000flexible,EventResource}). The warp~\eqref{eq:warp} is thus
	\begin{equation}\label{eq:motionmodel}
	f(\bu_i,t_i;\, \vel) = \frac{\bK^{(1:2)} \bR(t_i; \vel) \tilde{\bu}_i}{\bK^{(3)} \bR(t_i; \vel) \tilde{\bu}_i},
	\end{equation}
	where $\tilde{\bu}_i = [\bu_i^T~1]^T$ is the homogeneous version of $\bu_i$, and $\bK^{(1:2)}$ and $\bK^{(3)}$ are respectively the first-two rows and third row of $\bK$. Intuitively,~\eqref{eq:motionmodel} rotates the ray that passes through $\bu_i$ using $\bR(t_i; \vel)$, then projects the rotated ray onto $H$.
	
	Following~\cite[Sec.~3]{gallego2018unifying}, we also assume a known maximum angular rate $r_{\max}$. The domain is thus an  $r_{\max}$-ball
	\begin{align}
	\Omega = \{ \vel \in \mathbb{R}^3 \mid \| \vel \|_2 \le r_{\max} \},
	\end{align}
	and our problem reduces to maximising $C(\vel)$ over this ball, based on the rotational motion model~\eqref{eq:motionmodel}.

	%Our BnB method is able to process \hl{$N = 50,000$} events in \hl{XXXX} minutes (a locally optimal solver takes \hl{XXXX} minutes on the same input), and has the potential to be further speeded-up using GPUs.
	
	\subsection{Main algorithm}
	
	Algorithm~\ref{al:bnb} summarises our BnB algorithm to achieve globally optimal CM for rotation estimation. Starting from the tightest bounding cube $\rbox$ on the $r_{\max}$-ball $\Omega$ (the initial $\rbox$ is thus of size $(2r_{\max})^3$), the algorithm recursively subdivides $\rbox$ and prunes the subcubes until the global solution is found. A lower bound $\underline{C}$ and upper bound $\overline{C}(\rbox)$ are used to prune each $\rbox$. When the difference between the bounds is smaller than $\tau$, the algorithm terminates with $\hat{\vel}$ being the global solution $\vel^*$ (up to error $\tau$, which can be chosen to be arbitrarily small). See~\cite{horst90,hartley2009global} for details of BnB.
	
	As alluded to above, our core contributions are novel and effective bounding functions for CM using BnB. We describe our bounding functions in the next section.
	
	\begin{algorithm}
		\caption{BnB for rotation estimation from events.}\label{al:bnb}
		\begin{algorithmic}[1]
			\REQUIRE Event stream $\cE = \{ \event_i \}^{N}_{i=1}$, maximum angular rate of change $r_{\max}$, convergence threshold $\tau$.
			\STATE $q \leftarrow$ Initialise priority queue.
			\STATE $\rbox \leftarrow$ Cube in $\mathbb{R}^3$ of size $(2r_{\max})^3$ centred at origin.
			\STATE $\vel_c \leftarrow$ Centre of $\rbox$.
			\STATE $\hat{\vel} \leftarrow \vel_c$.
			\STATE Insert $\rbox$ into $q$ with priority $\overline{C}(\rbox)$.
			\WHILE{$q$ is not empty}
			\STATE $\rbox \leftarrow$ Dequeue top item from $q$.
			\STATE If $\overline{C}(\rbox) - C(\hat{\vel}) \leq \tau$, then terminate.
			\STATE $\vel_c \leftarrow$ Centre of $\rbox$.
			\STATE If $C(\vel_c) \geq C(\hat{\vel})$, then $\hat{\vel} \leftarrow \vel_c$.
			\STATE Uniformly subdivide $\rbox$ into $8$ subcubes $\rbox_1,\ldots, \rbox_8$.
			\FOR{$i = 1,\cdots,8$}
			\IF{$\overline{C}(\rbox_i) \geq C(\hat{\vel})$}
			\STATE Insert $\rbox_i$ into $q$ with priority $\overline{C}(\rbox_i)$.
			\ENDIF
			\ENDFOR
			\ENDWHILE
			\RETURN $\hat{\vel}$ as $\vel^*$.
		\end{algorithmic}
	\end{algorithm}
	
	\begin{figure*}[t]\centering
		\centering
		\subfloat[]{\includegraphics[height=8.5em]{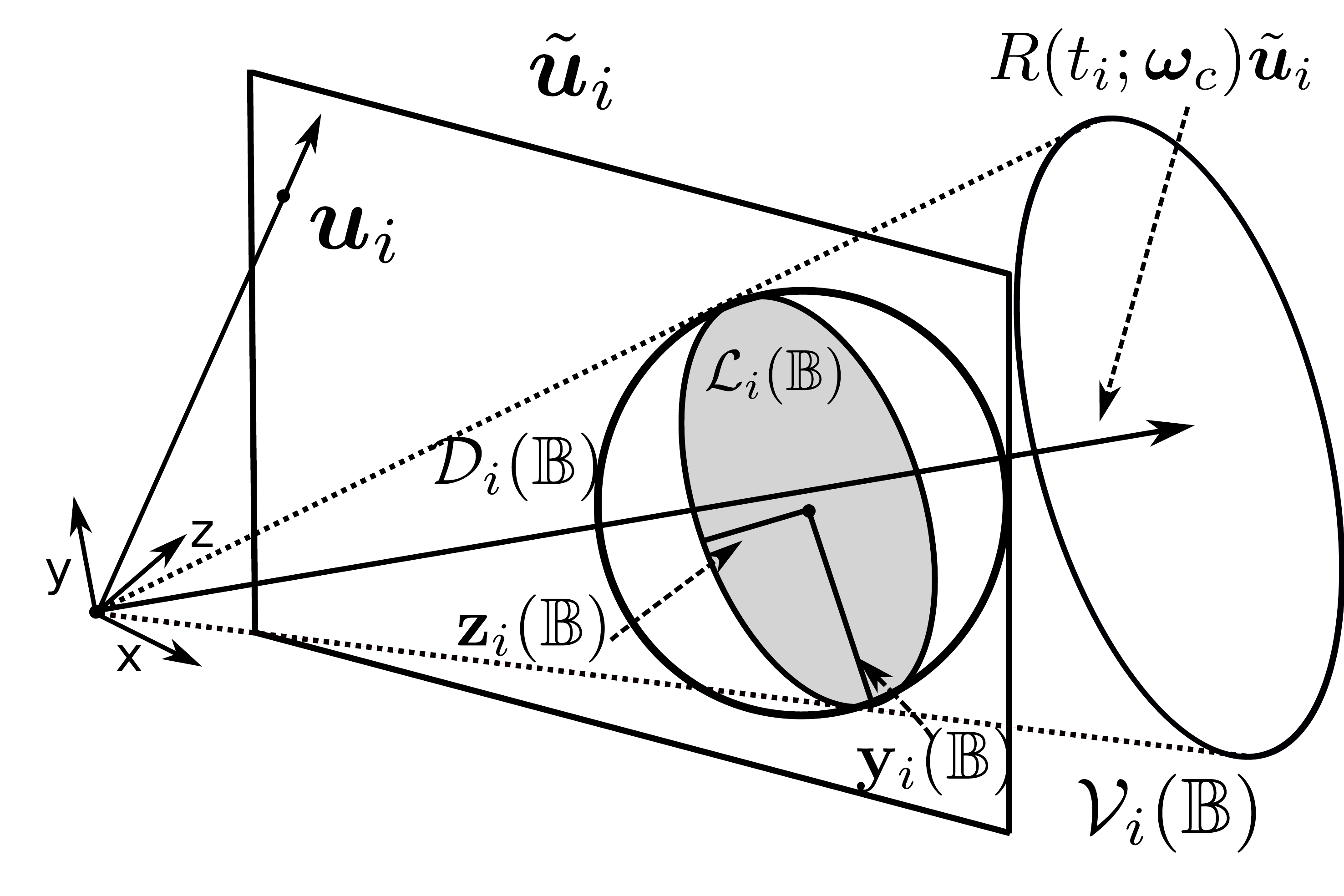}\label{fig:cone}} \hfill
		\subfloat[]{\includegraphics[height=8.5em]{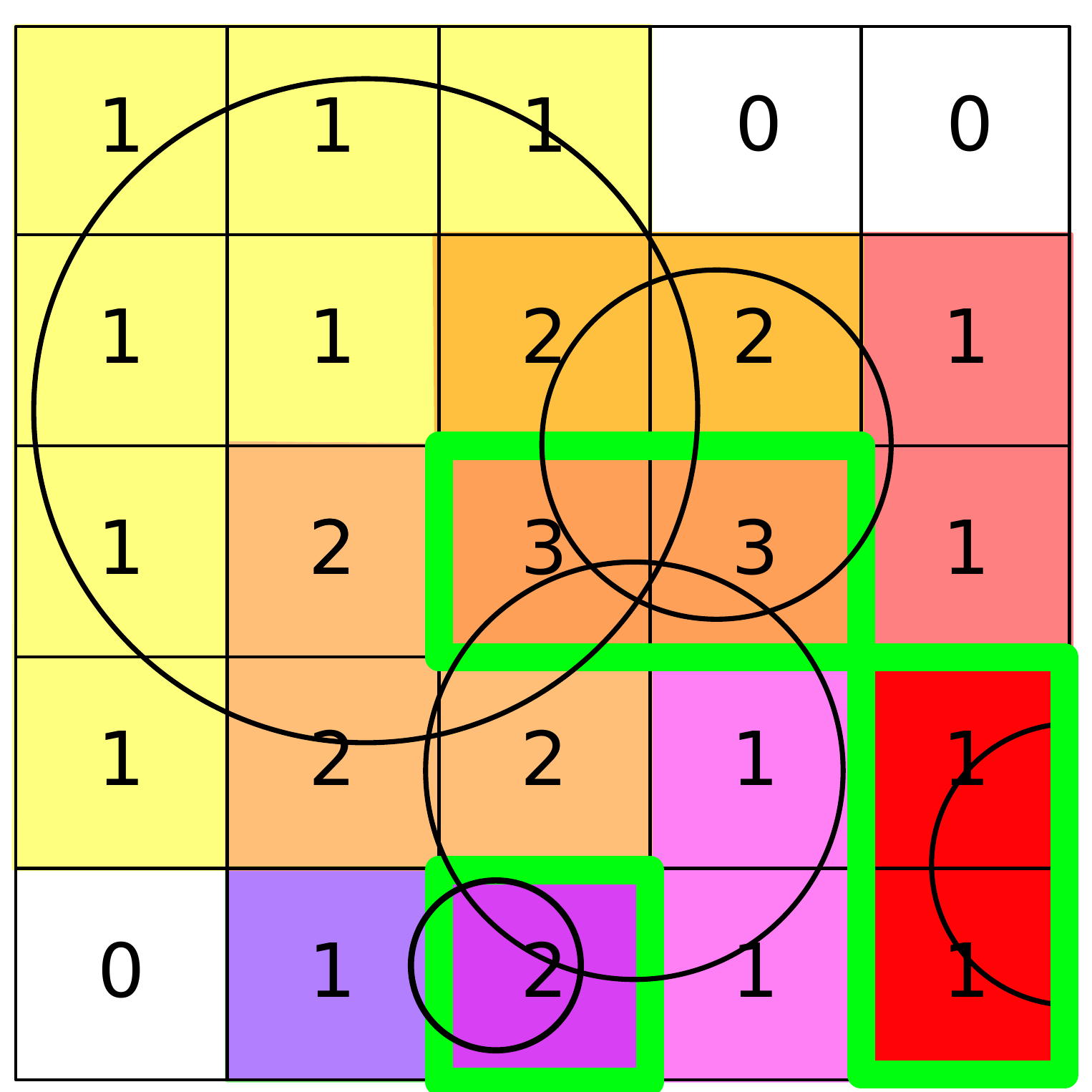}\label{fig:CC}} \hfill
		\subfloat[]{\includegraphics[height=8em]{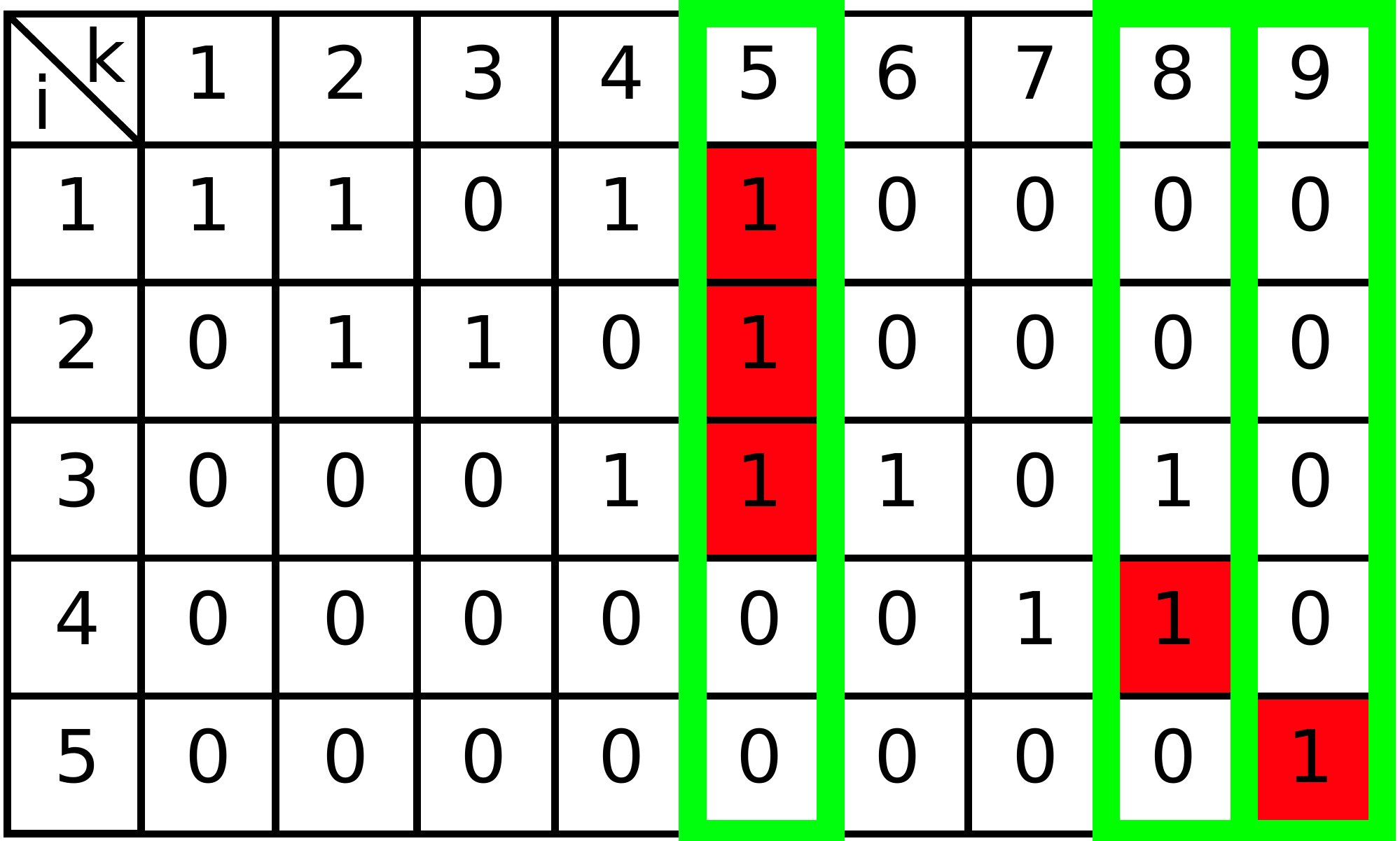}\label{fig:qip}} \hfill
		\subfloat[]{\includegraphics[height=8em]{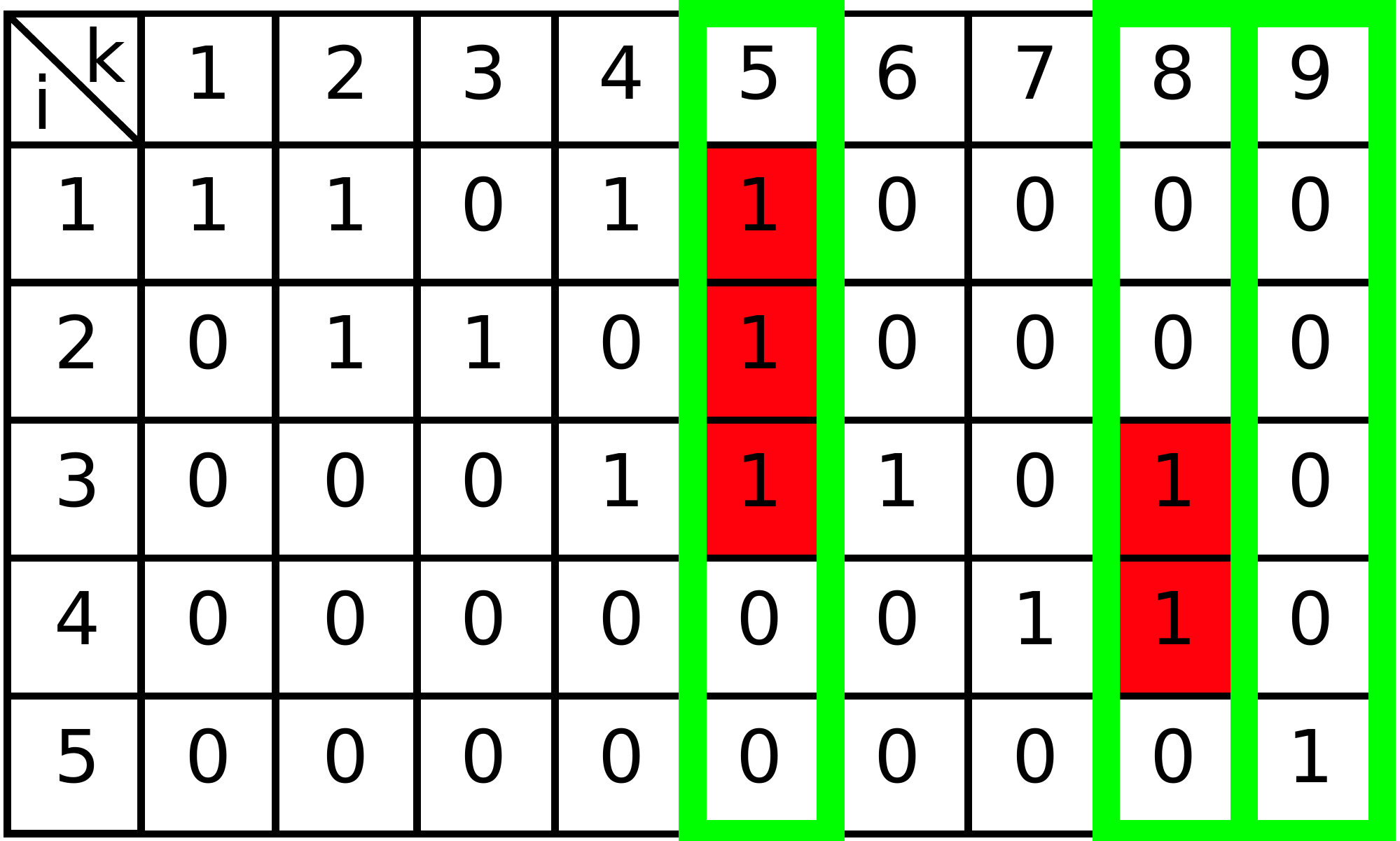}\label{fig:rqip}}    
		\caption{(a) The vector $R(t_i,\vel)\tilde{\bu}_i$ for all $\vel \in \rbox$ lies in the cone $\cV_i(\rbox)$, and the projection of $\cV_i(\rbox)$ onto the image plane is an elliptical region $\cL_i(\rbox)$. (b) Example with $5 \times 5$ pixel image and $N = 5$ events; there are thus $N = 5$ discs $\{ \cD_i(\rbox) \}_{i=1}^5$ on the image. The value in a pixel is the number of discs that intersect the pixel. Pixels with the same color are from the same connected component (CC); there are $9$ CCs in this example. (c)(d) Incidence matrix $\bM$ corresponding to the example in (b). The solution of~\ref{eq:qip} and~\ref{eq:rqip} are indicated respectively in (c) and (d), where entries $(i,k)$ of $\bZ$ that are set to $1$ are marked in red. The dominant columns of $\bM$ are outlined in green, and their corresponding CCs are also outlined in green in (b).}
	\end{figure*}
	
	\section{Bounds for contrast maximisation}\label{sec:bounding}
	
	To search for the maximum of $C(\vel)$ using BnB, a lower and upper bound on the objective are required.
	
	The lower bound $\underline{C}$ must satisfy the condition
	\begin{align}
	\underline{C} \le \max_{\vel \in \Omega}~C(\vel),
	\end{align}
	which is trivially achieved by any (suboptimal) solution. In Algorithm~\ref{al:bnb}, the current best solution $\hat{\vel}$ is used to provide $\underline{C}$, which is iteratively raised as the search progresses.
	
	The upper bound $\upbnd(\rbox)$ is defined over a region (a subcube) $\rbox$ of $\Omega$, and must satisfy the condition
	\begin{equation}\tag{A1}\label{eq:cond1}
	\upbnd(\rbox) \ge \max_{\vel \in \rbox}~C(\vel).
	\end{equation}
	Also, as $\rbox$ collapses to a single point $\vel$, $\upbnd(\rbox)$ should equate to $C(\vel)$; more formally, 
	\begin{equation}\tag{A2}\label{eq:cond2}
	\upbnd(\rbox) \to C(\vel)~~\text{when}~~\rbox \to \vel.
	\end{equation}
	See~\cite{horst90} for the rationale of the above conditions for BnB.
	
	%Concatenating all pixel intensities in $H$ to yield vector
	%\begin{equation}
	%\bH(\vel) = \left[ \begin{matrix} H(\bx_1;\vel) & \dots & H(\bx_P;\vel) %\end{matrix} \right]^T,
	%\end{equation}
	%the contrast function~\eqref{eq:3} \hl{can be rewritten as}
	%\begin{equation}\label{eq:vec_contrast}
	%C(\vel) = \dfrac{1}{P}\bH(\vel)^T\bH(\vel) - \mu(\vel)^2.
	%\end{equation}
	
	Deriving the upper bound is a more involved process. Our starting point is to rewrite~\eqref{eq:c} as
	\begin{equation}\label{eq:vec_contrast}
	C(\vel) = \dfrac{1}{P}\sum_{j=1}^P H(\bx_j;\vel)^2 - \mu(\vel)^2,
	\end{equation}
	which motivates a bound based on two components
	\begin{equation}\label{eq:upbnd}
	\upbnd(\rbox) := \dfrac{1}{P}\overline{S}(\rbox) - \underline{\mu}(\rbox)^2,
	\end{equation}
	where $\overline{S}(\rbox)$ is an upper bound
	\begin{equation}
	\label{eq:SOSB}
	\overline{S}(\rbox) \geq \max_{\vel\in \rbox}~\sum_{j=1}^P H(\bx_j;\vel)^2
	\end{equation}
	on the ``sum of squares (SoS)" component, and
	\begin{equation}\label{eq:muunder}
	\underline{\mu}(\rbox) \le  \min_{\vel\in \rbox}~\mu(\vel)
	\end{equation}
	is a lower bound of the mean pixel value. Given~\eqref{eq:SOSB} and~\eqref{eq:muunder}, then~\eqref{eq:upbnd} satisfies~\ref{eq:cond1}. If equality holds in~\eqref{eq:SOSB} and~\eqref{eq:muunder} when $\rbox$ is singleton, then~\eqref{eq:upbnd} also satisfies~\ref{eq:cond2}.
	
	In Secs.~\ref{sec:method1} and~\ref{sec:method2}, we develop $\overline{S}(\rbox)$ for continuous and discrete event images, before deriving $\underline{\mu}(\rbox)$ in Sec.~\ref{sec:meanbound}.
	
	\subsection{SoS bound for continuous event image}\label{sec:method1}
	
	For the continuous event image $H_c$~\eqref{eq:rt}, our SoS upper bound (denoted $\overline{S}_c$) is defined as
	\begin{equation}
	\overline{S}_c(\rbox) := \sum_{j=1}^P \overline{H}_c(\bx_j;\mathbb{B})^2,
	\end{equation}
	where $\overline{H}_c(\bx_j;\rbox)$ is an upper bound on the value of $H_c$ at $\bx_j$. To obtain $\overline{H}_c(\bx_j;\rbox)$, we bound the position
	\begin{align}\label{eq:uncertainty}
	\{ \bu^\prime_i = f(\bu_i,t_i;\, \vel) \mid \vel~\in \rbox \}    
	\end{align}
	of each warped event, under all possible $\vel \in \rbox$ for the warping function~\eqref{eq:motionmodel}. To this end, let $\vel_\bc$ be the centre of a cube $\rbox$, and $\vel_\mathbf{p}$ and $\vel_\mathbf{q}$ be opposite corners of $\rbox$. Define
	\begin{align}
	\alpha_i(\rbox) := 0.5\|\vel_\mathbf{p}t_i-\vel_\mathbf{q}t_i\|_2.
	% \angle (\bR(t_i ; \vel_\bc), \bR(t_i ; \vel)) := 2\arcsin{(\dfrac{2}{\sqrt{2}}\|\bR(t_i ; \vel_\bc)- \bR(t_i ; \vel)\|)}
	\end{align}
	Then, the following inequality can be established
	\begin{equation}\label{eq:axisanglebound}
	\angle (\bR(t_i ; \vel_\bc) \tilde{\bu}_i, \bR(t_i ; \vel)\tilde{\bu}_i) \leq \alpha_i(\rbox),
	%\angle (\bR(t_i ; \vel_\bc) \tilde{\bu}_i, \bR(t_i ; \vel)\tilde{\bu}_i) \leq  \angle (\bR(t_i ; \vel_\bc), \bR(t_i ; \vel_{\mathbf{q}}))
	\end{equation}
	which is an extension of ~\cite[Lemma 3.2]{hartley2009global}. Intuitively,~\eqref{eq:axisanglebound} states that the rotated vector $\bR(t_i ; \vel)\tilde{\bu}_i$ under all $\vel \in \rbox$ must lie within the cone
	\begin{align}\label{eq:cones}
	\cV_i(\rbox) := \left\{ \tilde{\bu} \in \mathbb{R}^3 \mid \angle( \bR(t_i ; \vel_\bc) \tilde{\bu}_i, \tilde{\bu}   ) \le \alpha_i(\rbox) \right\}.
	\end{align}
	Fig.~\ref{fig:cone} illustrates the cone $\cV_i(\rbox)$. Now, the pinhole projection of all the rays in $\cV_i(\rbox)$ yields the 2D region
	\begin{align}
	\cL_i(\rbox) = \left\{ \bx = \frac{\bK^{(1:2)}\tilde{\bu}}{\bK^{(3)}\tilde{\bu}} \,\middle\vert\,  \tilde{\bu} \in \cV_i(\rbox) \right\},
	\end{align}
	which is an elliptical region~\cite[Chap.~2]{mumford1995algebraic}; see Fig.~\ref{fig:cone}. Further, the centre $\bc_i(\rbox)$, semi-major axis $\by_i(\rbox)$ and semi-minor axis $\bz_i(\rbox)$ of
	$\cL_i(\rbox)$ can be analytically determined (see the supplementary material). We further define
	\begin{align}\label{eq:discs}
	\cD_i(\rbox) = \left\{ \bx \in \mathbb{R}^2 \,\middle\vert\, \| \bx - \bc_i(\rbox) \| \le \|\by_i(\rbox)\| \right\},
	\end{align}
	i.e., the smallest disc that contains $\cL_i(\rbox)$.
	
	By construction, $\cD_i(\rbox)$ fully contains the set of positions that $\bu_i^\prime$ can take for all $\vel \in \rbox$, i.e., the set~\eqref{eq:uncertainty}. We thus define the upper bound on the pixel values of $H_c$ as
	\begin{equation}\label{eq:pixelvalupbnd}
	\overline{H}_c(\bx_j;\rbox) = \sum_{i=1}^N\bm{\delta}\left(\max\left(\|\bx_j-\bc_i(\rbox)\|-\|\by_i(\rbox)\|,0\right)\right).
	\end{equation}
	Intuitively, we take the distance of $\bx_j$ to the boundary of $\cD_i(\rbox)$ to calculate the intensity, and if $\bx_j$ is within the disc then the distance is zero.
	
	\begin{lemma}\label{lem:pix_upbnd}
		\begin{align}
		\overline{H}_c(\bx_j;\rbox) \ge \max_{\vel\in \mathbb{B}}~H_c(\bx_j;\vel)
		\end{align}
		with equality achieved if $\rbox$ is singleton, i.e., $\rbox = \{ \vel \}$.
	\end{lemma}
	\begin{proof}
		See supplementary material.
	\end{proof}
	
	Given Lemma~\ref{lem:pix_upbnd}, it is clear that $\overline{S}_c(\rbox)$ satisfies the conditions (see Sec.~\ref{sec:bounding}) to be a valid component in the upper bound~\eqref{eq:upbnd} for the continuous event image $H_c$.
	
	\subsection{SoS bound for discrete event image}\label{sec:method2}
	
	Given the $N$ discs $\{ \cD_i(\rbox) \}^{N}_{i=1}$ associated with the $N$ events, define the intersection matrix $\bT \in \{0,1\}^{N \times P}$:
	\begin{align}\label{eq:intersections}
	\bT_{i,j} = \begin{cases} 1 & \cD_i(\rbox)~\text{intersects pixel}~\bx_j;  \\ 0 & \text{otherwise}. \end{cases}
	\end{align}
	The disc-pixel intersections can be computed efficiently using established techniques~\cite{van1984efficient,foley1996computer}. We assume $\sum_{j=1}^P \bT_{i,j} > 0$ for all $i$, i.e., each disc intersects at least one pixel. If there are discs that lie beyond the image plane, we ignore these discs without loss of generality.
	
	A direct extension of $\overline{H}_c$~\eqref{eq:pixelvalupbnd} to the discrete case would be to calculate the pixel upper bound value as
	\begin{equation}\label{eq:pix_upbnd_disc}
	\overline{H}_d(\bx_j;\rbox) = \sum_{i=1}^N \bT_{i,j},
	\end{equation}
	i.e., number of discs that intersect the pixel; see Fig.~\ref{fig:CC}. This can however be overly pessimistic, since the pixel value for the discrete event image~\eqref{eq:rt_discrete} satisfies
	\begin{equation}
	\sum_{j=1}^P H_d(\bx_j;\vel) \leq N  \implies \sum_{j=1}^P H_d(\bx_j;\vel)^2 \le N^2,
	\end{equation}
	whereas by using~\eqref{eq:pix_upbnd_disc},
	\begin{equation}\label{eq:badbound}
	\sum_{j=1}^P \overline{H}_d(\bx_j;\rbox) \leq PN  \implies \sum_{j=1}^P \overline{H}_d(\bx_j;\vel)^2 \le (PN)^2.
	\end{equation}
	Note that $P$ is the number of pixels (e.g., $P = 240 \times 320 = 76k$ for IniVation Davis 240C~\cite{bazin2012branch}), thus $(PN)^2 \gg N^2$.
	
	To get a tighter bound, we note that the discs $\{ \cD_i(\rbox) \}^{N}_{i=1}$ partition $X$ into a set of connected components (CC)
	\begin{align}\label{eq:cc}
	\{ \cG_k \}^{K}_{k=1},
	\end{align}
	where each $\cG_k$ is a connected set of pixels that are intersected by the same discs; see Fig.~\ref{fig:CC}. Then, define the incidence matrix $\bM \in \{0,1\}^{N \times K}$, where
	\begin{align}\label{eq:incidence}
	\bM_{i,k} = \begin{cases} 1 & \exists \bx_j \in \cG_k \; \text{such that}~\bT_{i,j} = 1; \\ 0 & \text{otherwise.} \end{cases}
	\end{align}
	In words, $\bM_{i,k} = 1$ if $\cD_i(\rbox)$ is a disc that intersect to form $\cG_k$. We then formulate the integer quadratic program
	\begin{align}\tag{IQP}\label{eq:qip}
	\begin{aligned}
	\overline{S}^*_d(\rbox) = \max_{\bZ \in \{0,1\}^{N \times K}} \quad & \sum^K_{k=1} \left( \sum_{i=1}^N \bZ_{i,k} \bM_{i,k} \right)^2 \\
	\text{s.t.} \quad & \bZ_{i,k} \le \bM_{i,k}, \;\; \forall i,k, \\
	& \sum_{k=1}^K \bZ_{i,k} = 1, \;\; \forall i.
	\end{aligned}
	\end{align}
	In words, choose a set of CCs that are intersected by as many discs as possible, while ensuring that each disc is selected exactly once. Intuitively, \ref{eq:qip} warps the events (under uncertainty $\vel \in \rbox$) into ``clusters" that are populated by as many events as possible, to encourage fewer clusters and higher contrast. See Fig.~\ref{fig:qip} for a sample solution of~\ref{eq:qip}.

	%Computing~\eqref{eq:cc} is a classic \hl{CC labelling problem}~\cite{wiki_cc_labelling_problem}; by treating each $\bT_{:,j}$ as a binary string and converting it into a numeric value, efficient techniques can applied to compute $\{ \cG_k \}^{K}_{k=1}$.
	
	%In the course of finding the CCs, we obtain the incidence matrix $\bM \in \{0,1\}^{N \times K}$, where
	%\begin{align}\label{eq:incidence}
	%\bM_{i,k} = \begin{cases} 1 & \exists \bx_j \in \cG_k \; \text{such that}~\bT_{i,j} = 1; \\ 0 & \text{otherwise.} \end{cases}
	%\end{align}
	%In words, $\bM_{i,k} = 1$ if $\cD_i(\rbox)$ is a disc that intersect to form $\cG_k$. Note that explicitly computing the CCs and $\bM$ is costly---these objects are invoked for conceptual explanations and their explicit computations are not required.

	%Note that $\{ \cG_k \}^{K}_{k=1}$ and $\bM$ are \emph{conceptual objects} to help define our bounding operation below, and the explicit computation and storage of $\{ \cG_k \}^{K}_{k=1}$ and $\bM$ are not required.
	
	\begin{lemma}
		\label{lem:2}
		\begin{align}
		\overline{S}^*_d(\rbox) \ge \max_{\vel\in \rbox}~\sum_{j=1}^P H_d(\bx_j;\vel)^2,
		\end{align}
		with equality achieved if $\rbox$ is singleton, i.e., $\rbox = \{ \vel \}$.
	\end{lemma}
	\begin{proof}
		See supplementary material.
	\end{proof}
	
	Solving~\ref{eq:qip} is challenging, not least because $\{ \cG_k \}^{K}_{k=1}$ and $\bM$ are costly to compute and store (the number of CCs is exponential in $N$). To simplify the problem, first define the \emph{density} of a CC $\cG_k$ (corresponding to column $\bM_{:,k}$) as
	\begin{align}
	\Delta_k = \sum_{i=1}^N \bM_{i,k}.
	\end{align}
	We say that a column $\bM_{:,\eta}$ of $\bM$ is \emph{dominant} if there exists a subset $\Lambda \subset \{1,\dots,K\}$ (including $\Lambda = \emptyset$) such that
	\begin{align}
	\bM_{i,k} \le \bM_{i,\eta} \;\; \forall i \in \{1,\dots,N\}, \forall k \in \Lambda,
	\end{align}
	whereas for all $k \notin \Lambda$, the above does not hold. In words, the $1$ elements of columns in $\Lambda$ is a subset of the $1$ elements of $\bM_{:,\eta}$. Geometrically, a dominant column $\bM_{:,\eta}$ corresponds to a CC $\cG_\eta$ such that for all discs that intersect to form the CC, $\cG_\eta$ is the densest CC that they intersect with; mathematically, there exists $\cD_i(\rbox) \supseteq \cG_\eta$ such that
	\begin{align}\label{eq:M'}
	\max_{\bx_j\in\cD_i(\rbox)}\sum_{i=1}^N\bT_{i,j} = \sum_{i=1}^N\bM_{i,\eta}.
	\end{align}
	Figs.~\ref{fig:CC} and~\ref{fig:qip} illustrate dominant columns.
	
	\begin{algorithm}[t]\centering
		\caption{Computing dominant columns $\bM'$.}\label{al:M'}
		\begin{algorithmic}[1]
			\REQUIRE Pixels $\{ \bx_j \}^{P}_{j=1}$, set of discs $\{ \cD_i(\rbox) \}^{N}_{i=1}$~\eqref{eq:discs}.
			\STATE $\bT \leftarrow$ $N\times P$ intersection matrix~\eqref{eq:intersections} from discs.
			\STATE $\left\{ \overline{H}_d(\bx_j;\rbox) \right\}^{P}_{j=1} \leftarrow $ Pixel upper bound image~\eqref{eq:pix_upbnd_disc}.
			\STATE $\{ a_j \}^{P}_{j=1} \leftarrow$ Array of $P$ elements initialised to $0$.
			%\STATE $c_{max}\leftarrow0$, $\bc_{col}\leftarrow NULL$ and $\mathcal{C}_{set}\leftarrow NULL$.
			\STATE $\bM' \leftarrow [~]$ (empty matrix).
			\FOR{$i = 1,\dots,N$}
			\STATE $c_{max} \leftarrow \max_{\bx_j \in 	\cD_i(\rbox)}\overline{H}_d(\bx_j;\rbox)$.
			\STATE $\mathcal{R} \leftarrow \left\{\bx_j\in\cD_i(\rbox) \mid \overline{H}_d(\bx_j;\rbox) = c_{max},a_j =0 \right\}$.
			%\vspace{-1em}
			\WHILE {$\mathcal{R}$ is not empty}
			\STATE Pick a pixel $\bx_j \in \mathcal{R}$ and $a_j \leftarrow 1$.
			\STATE $\bM' \leftarrow  \left[ \begin{matrix} \bM' & \bT_{:,j} \end{matrix} \right]$ and $\mathcal{R} \leftarrow \mathcal{R}\setminus \{ \bx_j \}$.
			\FOR {$\bx_\ell \in \mathcal{R}$}
			\IF {$ \bT_{:,\ell} = \bT_{:,j}$}
			\STATE $a_\ell \leftarrow 1$ and $\mathcal{R} \leftarrow \mathcal{R}\setminus \{\bx_\ell\}$.
			\ENDIF
			\ENDFOR
			\ENDWHILE
			\ENDFOR
			\RETURN $\bM'$.
		\end{algorithmic}
	\end{algorithm}
	
	Let $\bM^\prime \in \{0,1\}^{N \times K^\prime}$ contain only the dominant columsn of $\bM$. Typically, $K^\prime \ll K$, and $\bM^\prime$ can be computed directly without first building $\bM$, as shown in Algorithm~\ref{al:M'}. Intuitively, the method loops through the discs and incrementally keeps track of the densest CCs to form $\bM^\prime$.
	
	\begin{lemma}
		Problem \ref{eq:qip} has the same solution if $\bM$ is replaced with $\bM^\prime$.
	\end{lemma}
	\begin{proof}
		See supplementary material.
	\end{proof}
	
	It is thus sufficient to formulate~\ref{eq:qip} based on the dominant columns $\bM^\prime$. Further, we relax~\ref{eq:qip} into
	\begin{align}\tag{R-IQP}\label{eq:rqip}
	\begin{aligned}
	\overline{S}_d(\rbox) = \max_{\bZ \in \{0,1\}^{N \times K^\prime}} \quad & \sum^{K^\prime}_{k=1} \left( \sum_{i=1}^N \bZ_{i,k}\bM^\prime_{i,k} \right)^2 \\
	\text{s.t.} \quad & \bZ_{i,k} \le \bM^\prime_{i,k}, \;\; \forall i,k, \\
	& \sum_{k=1}^{K^\prime} \sum^{N}_{i=1} \bZ_{i,k} = N.
	\end{aligned}
	\end{align}
	where we now allow discs to be selected more than once. Since enforcing $\sum_{k=1}^{K^\prime} \bZ_{i,k} = 1 $ for all $i$ implies $\sum_{k=1}^{K^\prime}\sum^{N}_{i=1} \bZ_{i,k} = N$, \ref{eq:rqip} is a valid relaxation. See Fig.~\ref{fig:rqip} for a sample result of~\ref{eq:rqip}, and cf.~Fig.~\ref{fig:qip}.
	
	\begin{lemma}
		\begin{align}
		\overline{S}_d(\rbox) \ge \overline{S}^*_d(\rbox)
		\end{align}
		with equality achieved if $\rbox$ is singleton, i.e., $\rbox = \{ \vel \}$.
	\end{lemma}
	\begin{proof}
		See supplementary material.
	\end{proof}
	
	% and
	%\begin{align}
	%\overline{S}_d(\rbox) \ge \overline{S}^*_d(\rbox).
	%\end{align}
	%Moreover, if $\rbox = \{ \vel \}$ is singleton, all $\cG_k$ collapse into single pixel $\bx_j$, in other words, one events can only be assigned to one connected conponent/pixel, which already satisfied the second constrain in~\eqref{eq:qip}.
	%
	% \hl{xxxxxxxxxxxxxxxxxx} and $\overline{S}_d(\rbox)$ equates to $\overline{S}^*_d(\rbox)$. Thus, $\overline{S}_d(\rbox)$ satisfies the conditions (see Sec.~\ref{sec:bounding}) to be a valid SoS upper bound for the discrete event image $H_d$.
	
	\paragraph{Bound computation and tightness}
	
	\ref{eq:rqip} admits a simple solution. First, compute the densities $\{ \Delta_k \}^{K^\prime}_{k=1}$ of the columns of $\bM^\prime$. Let $\Delta_{(k)}$ be the $k$-th highest density, i.e.,
	\begin{align}
	\Delta_{(k_1)} \ge \Delta_{(k_2)} \;\; \text{if} \;\; k_1 < k_2.
	\end{align}
	Obtain $\gamma$ as the largest integer such that
	\begin{align}
	\sum_{k=1}^{\gamma} \Delta_{(k)} < N.
	\end{align}
	Then, the SoS upper bound for the discrete event image is
	\begin{align}
	\overline{S}_d(\rbox) = \sum_{k=1}^\gamma \Delta^2_{(k)} + \left(N - \sum_{k=1}^{\gamma} \Delta_{(k)}\right)^2.
	\end{align}
	Intuitively, the procedure greedily takes the densest CCs while ensuring that the quota of $N$ discs is not exceeded. Then, any shortfall in the number of discs is met using the next largest CC partially. Given $\bM^\prime$, the costliest routine is just the sorting of the column sums of $\bM^\prime$.
	
	%A tighter bound can be counterproductive if the cost of computing the bound is nontrivial. Fortunately,
	
	Given the final constraint in~\ref{eq:rqip}, it is clear that $\overline{S}_d(\rbox) \le N^2$. This represents a much tighter SoS upper bound than $\sum_{j=1}^P \overline{H}_d(\bx_j;\vel)^2$; see~\eqref{eq:badbound}.

	\subsection{Lower bound of mean pixel value}\label{sec:meanbound}
	
	For the continuous event image~\eqref{eq:rt}, the lower bound of the pixel value is the ``reverse" of the upper bound~\eqref{eq:pixelvalupbnd}, i.e.,
	\begin{align}
	\label{eq:pixelvaluelwbnd}
	\underline{H}_c(\bx_j ; \rbox) = \sum_{i=1}^N \bm{\delta}\left(\|\bx_j-\bc_i(\rbox)\|+\|\by_i(\rbox)\|\right),
	\end{align}
	whereby for each $\cD_i(\rbox)$, we take the maximum distance between $\bx_j$ and a point on the disc. Then, the lower bound of the mean pixel value is simply 
	\begin{equation}\label{eq:pixelvallwbnd}
	\underline{\mu}_c(\rbox) = \frac{1}{P}\sum^P_{j=1} \underline{H}_c(\bx_j ; \rbox).
	\end{equation}
	In the discrete event image~\eqref{eq:rt_discrete}, %if a pixel is not intersected by any disc $\cD_i(\rbox)$, then that pixel cannot have a nonzero value under all $\vel \in \rbox$. This leads to computing the lower bound of the pixel value as
	if all the $N$ discs lie fully in the image plane, the lower bound can be simply calculated as $N/P$. However, this ideal case rarely happens,  hence the the lower bound on the mean pixel vale is
	%\begin{align}
	%\underline{H}_d(\bx_j ; \rbox) = \mathbb{I}\left( \sum_{i=1}^N \bT_{i,j} > 0\right),
	%\end{align}
	%and the lower bound on the mean pixel value as
	%\begin{align}
	%\underline{\mu}_d(\rbox) = \dfrac{1}{P}  \sum_{j=1}^P %\underline{H}_d(\bx_j ; \rbox).
	%\end{align}
	\begin{equation}
	\label{eq:lowerbound}
	\underline{\mu}_d(\rbox) = \dfrac{1}{P}  \sum_{i=1}^N\mathbb{I}(\cD_i~\text{fully lie in the image plane}).
	\end{equation}
	See the supplementary material for proofs of the correctness of the above lower bounds.
	
	\subsection{Computational cost and further acceleration}\label{sec:gpu}
	
	Our BnB method is able to process $N \approx 50,000$ events in $\approx 300$ seconds. While this does not allow online low latency event processing, it is nonetheless useful for event sensing applications that permit offline computations, e.g., video stabilisation with post-hoc correction. Note that a local method can take up to $30$ seconds to perform CM on the same input, which also does not enable online processing\footnote{Since the implementation of~\cite{gallego2018unifying} was not available, we used the conjugate gradient solver in \texttt{fmincon} (Matlab) to solve CM locally optimally. Conjugate gradient solvers specialised for CM could be faster, though the previous works~\cite{gallego2018unifying,gallego2019focus,stoffregen2019event,khoei2019asynchronous} did not report online performance.} (Sec.~\ref{sec:results} will present more runtime results). 
	
	There is potential to speed-up our algorithm using GPUs. For example, in the bound computations for the discrete event image case, the disc-pixel intersection matrix $\bT$~\eqref{eq:intersections} could be computed using GPU-accelerated ray tracing~\cite{popov2007stackless,carr2006fast}, essentially by backprojecting each pixel and intersecting the ray with the cones~\eqref{eq:cones} in parallel. We leave GPU acceleration as future work.
	
	%However, this is not straightforward.
	%For example, in bound computations for the discrete event image case, the disc-pixel intersection matrix $\bT$}

	%While one does not expect a global algorithm to be as fast as a local algorithm, note that a local method can take up to $30$ seconds to perform CM on the same input, which is also not continuous time\footnote{Since the implementation of~\cite{gallego2018unifying} was not available, we used the conjugate gradient solver in \texttt{fmincon} (Matlab) to solve CM locally optimally. Conjugate gradient solvers specialised for CM could be faster, though the previous works~\cite{gallego2018unifying,gallego2019focus,stoffregen2019event,khoei2019asynchronous} did not report continuous time performance.} (Sec.~\ref{sec:results} will present more runtime results). 
	
	%**************************************************
	%**************************************************
	\vspace{-1mm}
	\section{Results}\label{sec:results}
	%**************************************************
	%**************************************************
	
	We first examine the runtime and solution quality of our algorithms, before comparing against state-of-the-art methods in the literature. The results were obtained on a standard desktop with a 3.0GHz Intel i5 CPU and 16GB RAM.
	
	\subsection{Comparison of bounding functions}
	
	The aim here is to empirically compare the performance of BnB (Algorithm~\ref{al:bnb}) with continuous and discrete event images. We call these variants \textbf{CMBnB1} and \textbf{CMBnB2}.
	
	For this experiment, a $10$ ms subsequence (which contains about $N = 50,000$ events) of the \emph{boxes} data~\cite{gallego2017accurate} was used. The underlying camera motion was a pure rotation.
	
	For CMBnB1, a Gaussian kernel with bandwidth $1$ pixel was used (following~\cite[Sec.~2]{gallego2018unifying}). Fig.~\ref{fig:iteration} plots the upper and lower bound values over time in a typical run of Algorithm~\ref{al:bnb}. It is clear that the discrete case converged much faster than the continuous case; while CMBnB2 terminated at about $12k$ iterations, CMBnB1 requried no fewer than $30k$ iterations. It is evident from Fig.~\ref{fig:iteration} that this difference in performance is due to the much tighter bounding in the discrete case. The next experiment will include a comparison of the solution quality of CMBnB1 and CMBnB2.
	\begin{figure}[ht]\centering
		\centering
		\includegraphics[width=0.85\linewidth]{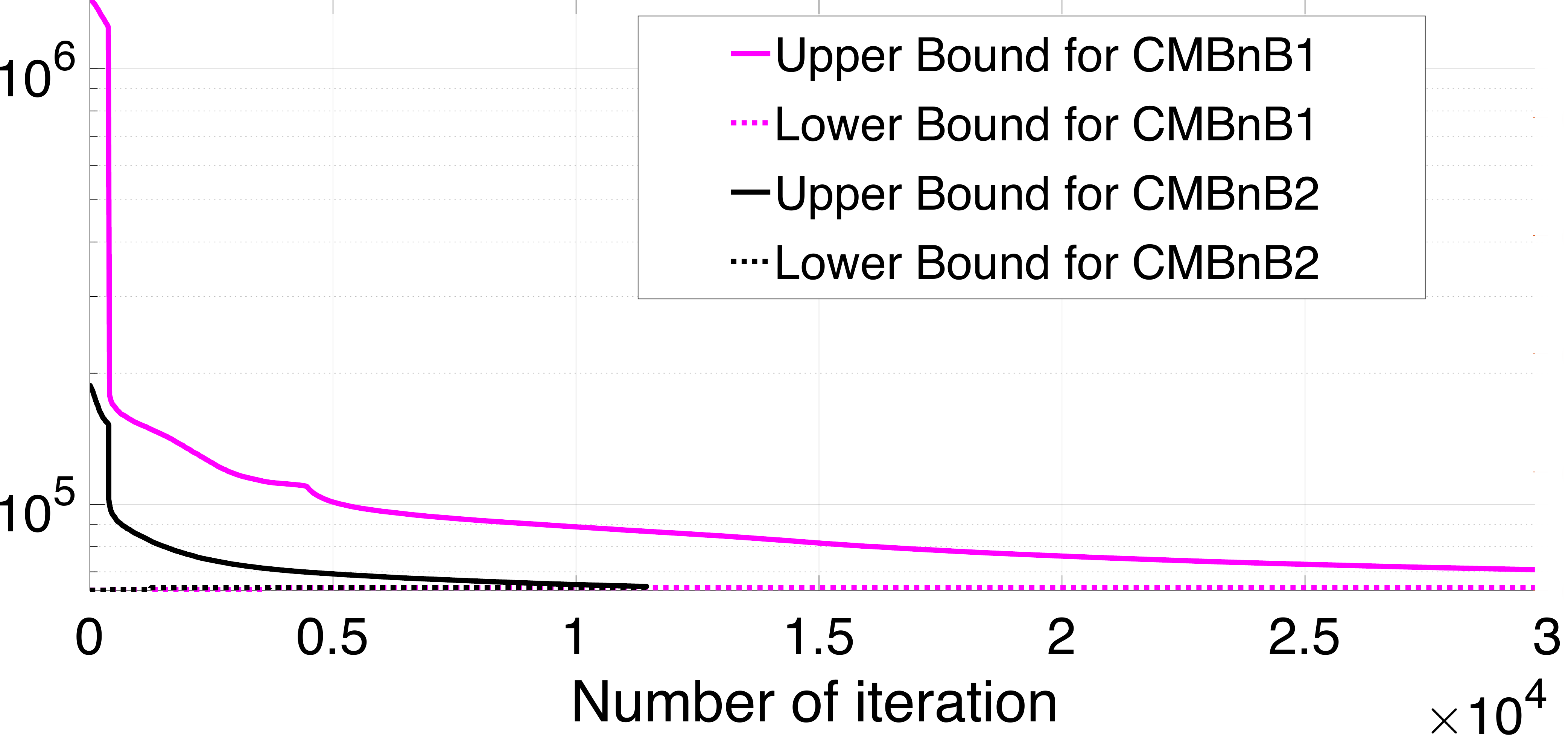}
		\caption{Upper and lower bound evolution in branch-and-bound.}
		\label{fig:iteration}
	\end{figure}

	\subsection{Qualitative comparisons}\label{sec:global_local_exp}
	
	To highlight the importance of globally optimal CM, we tested on select $10$ ms subsequences (about $N=50k$ events each) from the \emph{boxes} data~\cite{gallego2017accurate}---in the next experiment, a more comprehensive experiment and quantitative benchmarking will be described. Here, on the subsequences chosen, we compared BnB against the following methods:
	\begin{itemize}[leftmargin=1em,topsep=0.2em,parsep=0.2em,itemsep=0em]
		\item \textbf{CMGD1}: locally optimal solver (\texttt{fmincon} from Matlab) was used to perform CM with initialisation $\vel = 0$ (equivalent to identity rotation).
		\item \textbf{CMGD2}: same as above, but initialised with the optimised $\vel$ from the previous $10$ ms time window.
	\end{itemize}
	Both local methods were executed on the continuous event image with Gaussian kernel of bandwidth $1$ pixel.
	
	Fig.~\ref{fig:global_local} depicts motion compensated event images from two subsequences (Subseq 1 and Subseq 2); see supplementary material for more results. The examined cases show that the local methods (both CMGD1 and CMGD2) can indeed often converge to bad local solutions. Contrast this to BnB which always produced sharp event images.
	
	These results also show that CM based on continuous and discrete event images yield practically identical solutions. Since CMBnB2 usually converges much faster than CMBnB1, we use CMBnB2 in the remaining experiments.
	
	%smoothing function such as Gaussian (CMBNB1) and discretisation (CMBnB2) gives a similar result for global optimal solution. However, CMBnB2 has a much better processing time than CMBnB1 ($279$ seconds compares to $5$ hours)
	
	\begin{figure*}
		\renewcommand{\arraystretch}{0.7}
		\setlength{\tabcolsep}{1mm}
		\footnotesize
		\begin{tabularx}{\textwidth} { 
				>{\raggedright\arraybackslash}C{.05} 
				| >{\centering\arraybackslash}X 
				| >{\centering\arraybackslash}X 
				| >{\centering\arraybackslash}X 
				| >{\centering\arraybackslash}X  }
			& CMBnB1
			& CMBnB2
			& CMGD1
			& CMGD2\\
			%\hline
			\RotText{~~~~~~~~Subseq 1}
			& \includegraphics[width=\linewidth]{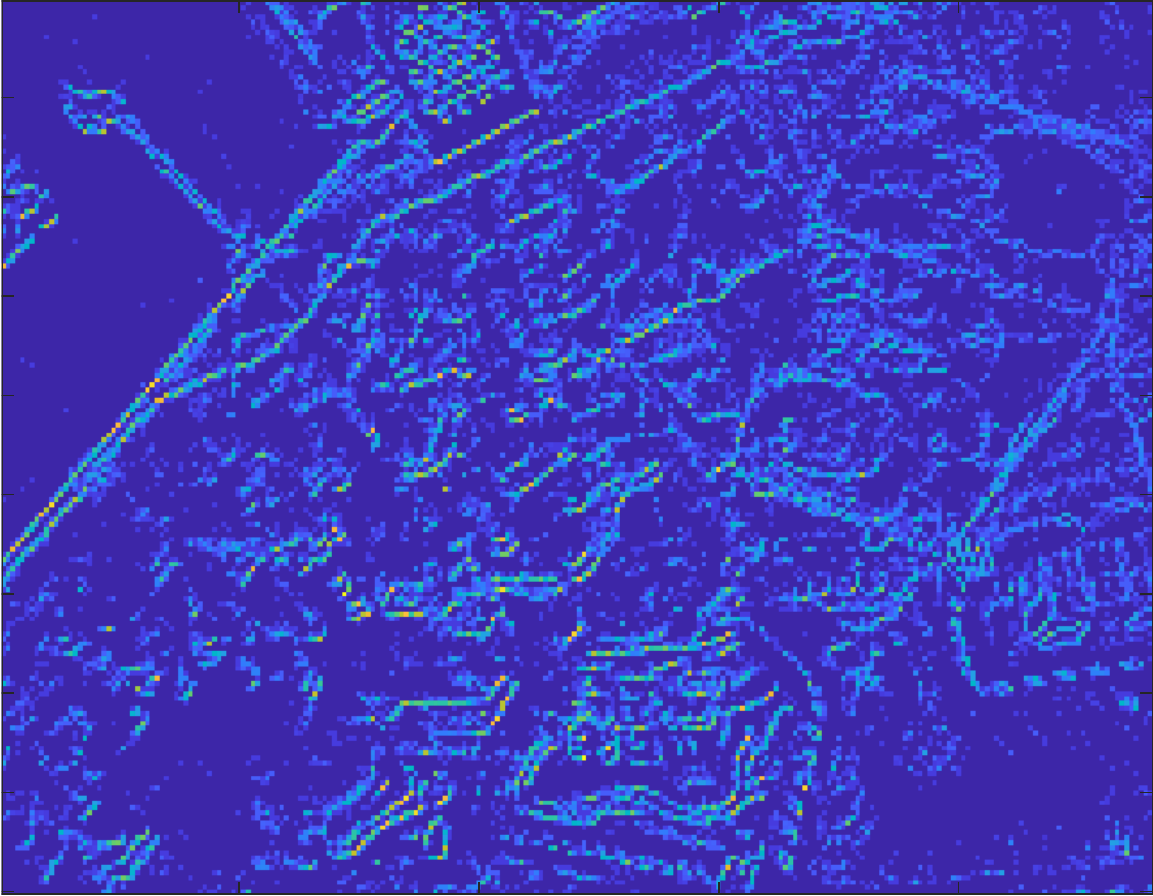}  
			& \includegraphics[width=\linewidth]{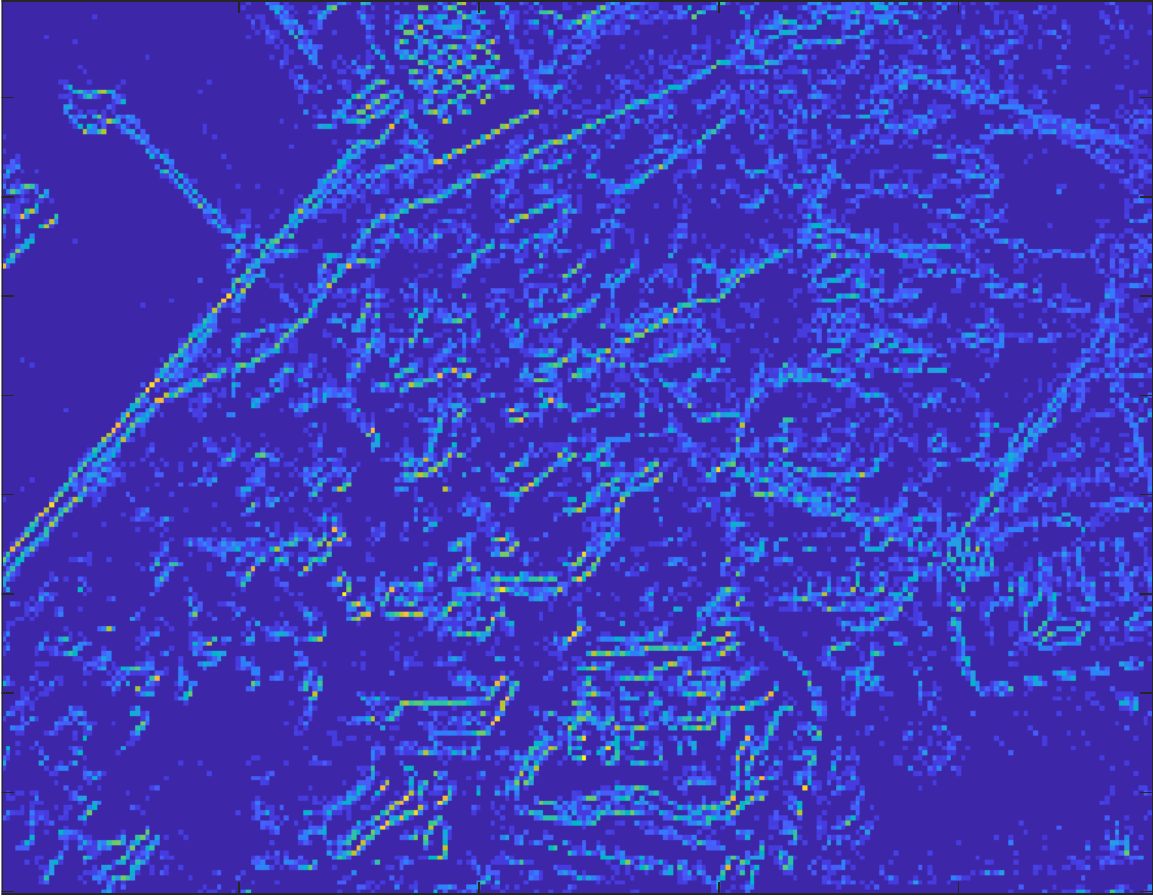}
			& \includegraphics[width=\linewidth]{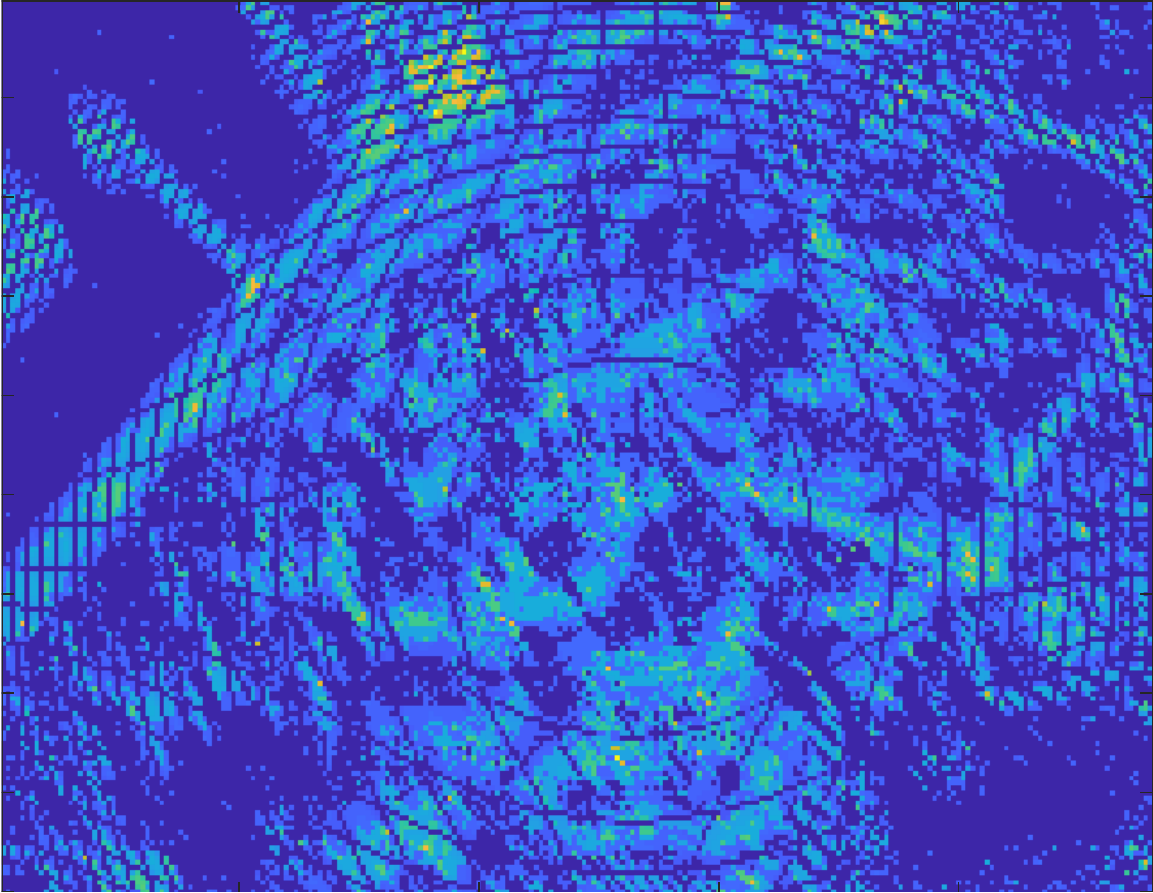}
			& \includegraphics[width=\linewidth]{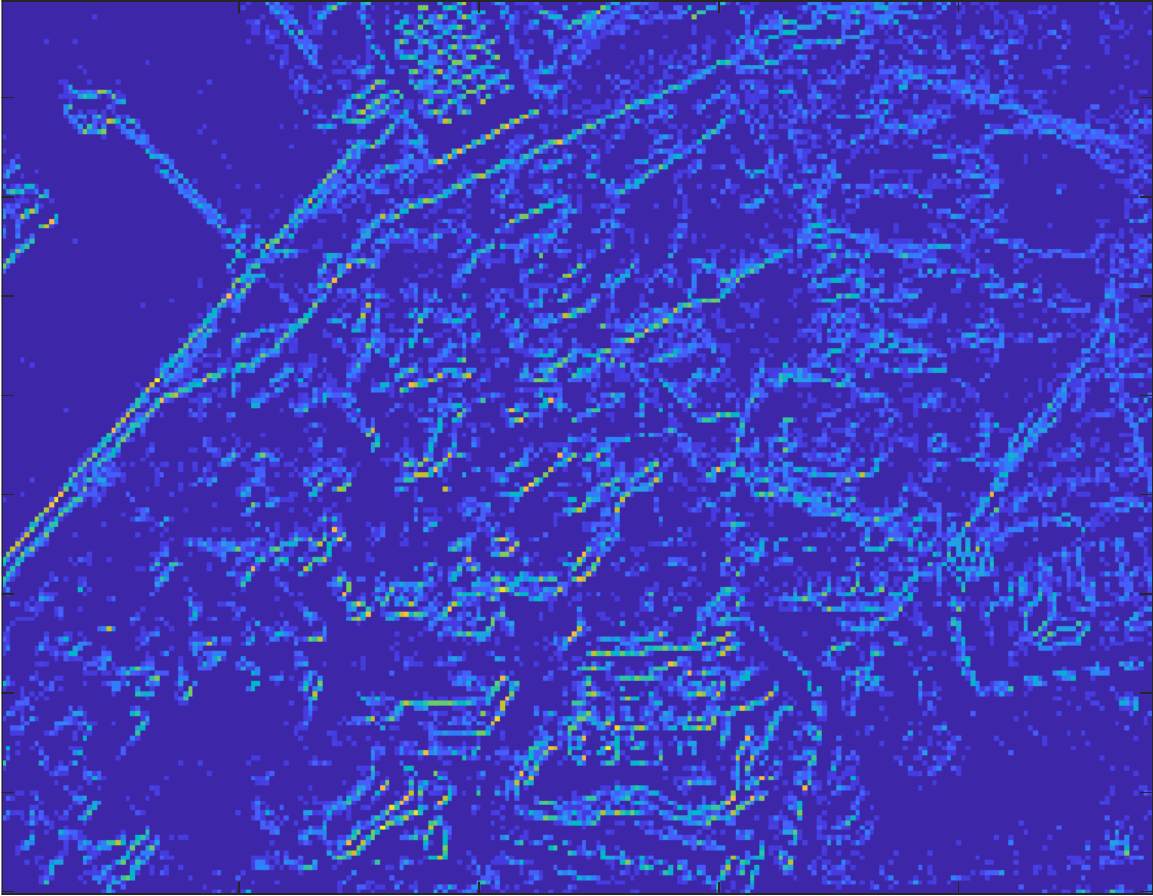}\\
			\hline
			&&&&\\[-.15cm]
			\RotText{~~~~~~~~~Subseq 2}
			& \includegraphics[width=\linewidth]{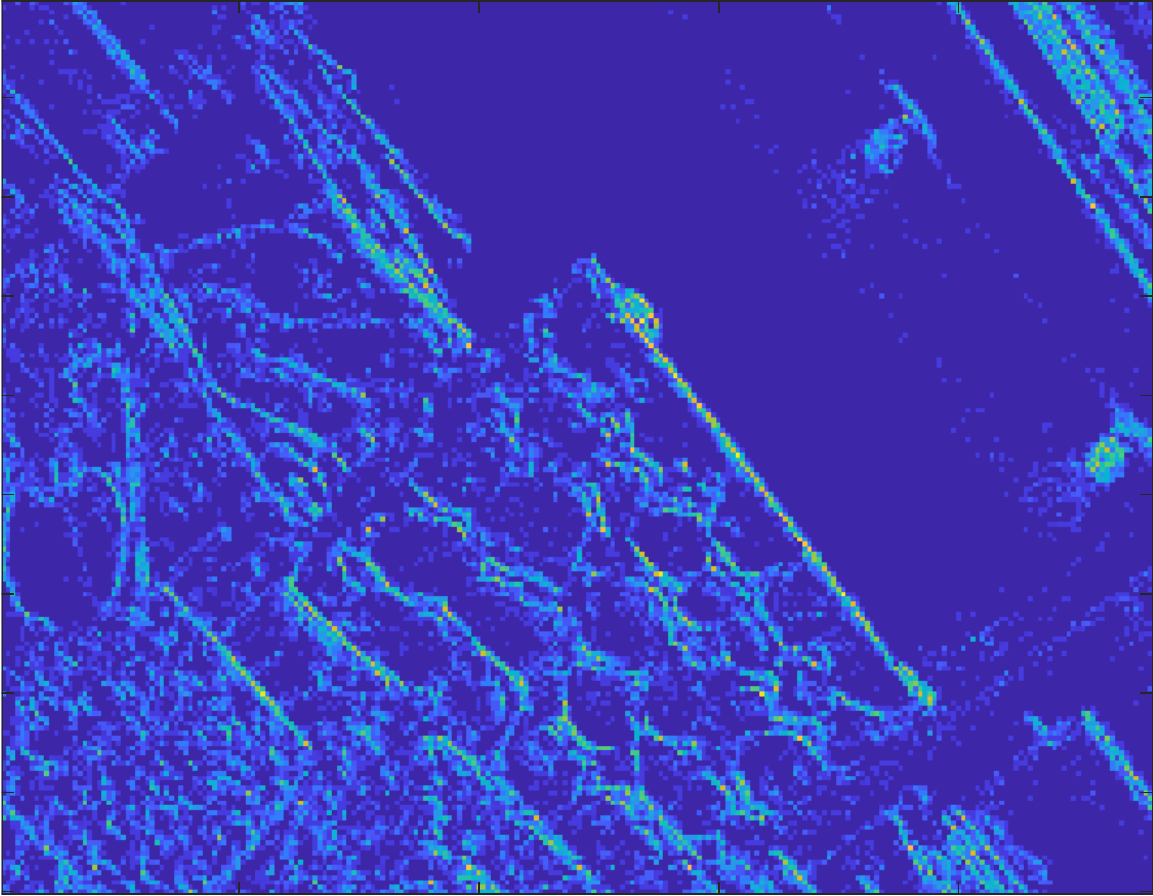}  
			& \includegraphics[width=\linewidth]{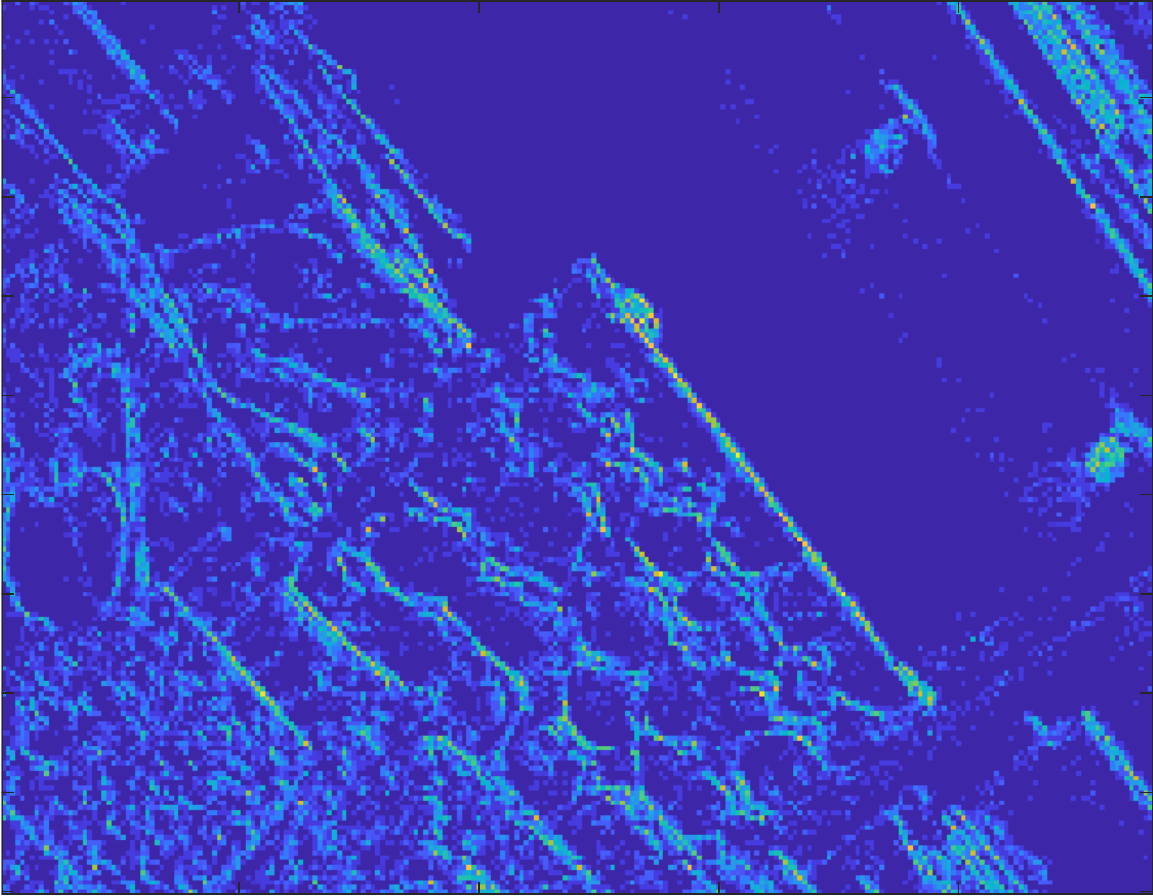}
			& \includegraphics[width=\linewidth]{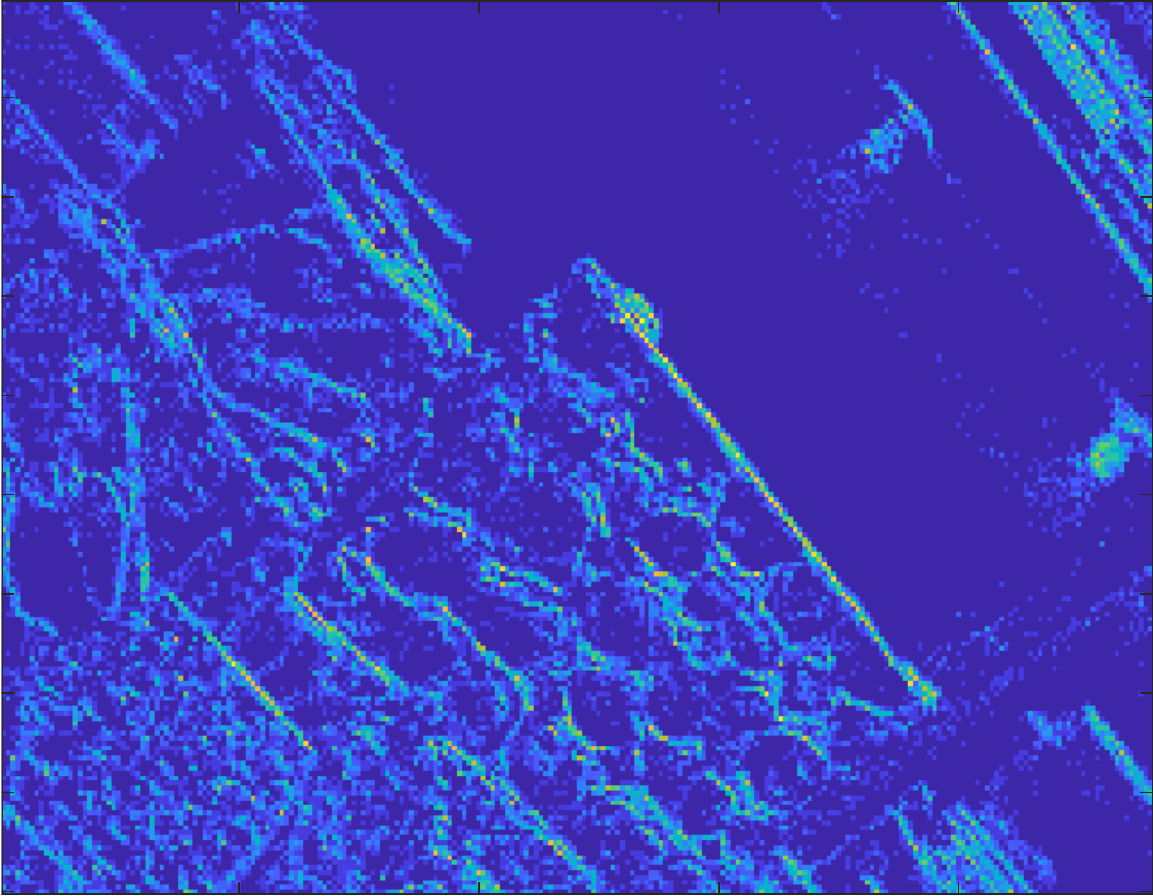}
			& \includegraphics[width=\linewidth]{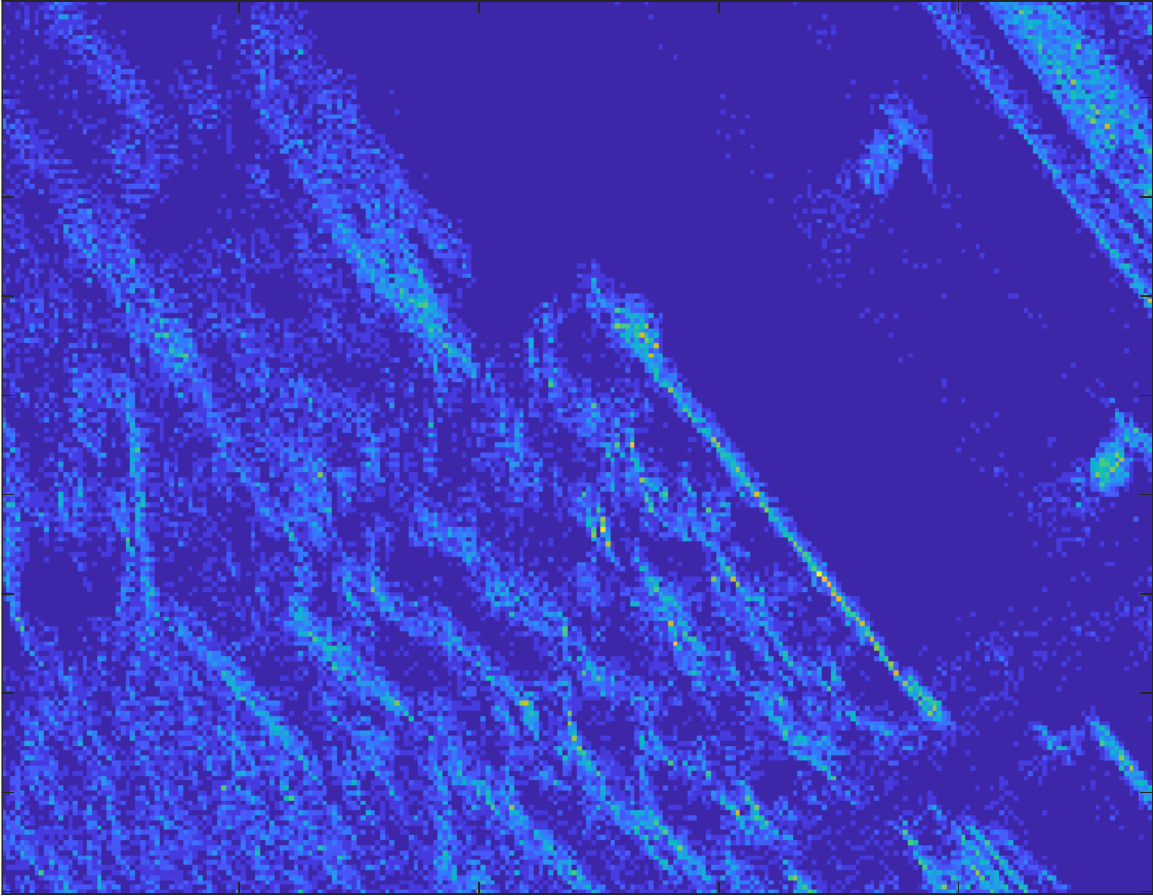}
		\end{tabularx}
		\caption{Qualitative results (motion compensated event images) from Subseq 1 and Subseq 2 of \emph{boxes}.}
		\label{fig:global_local}
	\end{figure*}
	
	\subsection{Quantitative benchmarking}
	
	We performed benchmarking using publicly available datasets~\cite{mueggler2017event,chin2019star}. We introduced two additional variants to CMGD1 and CMGD2:
	\begin{itemize}[leftmargin=1em,topsep=0.2em,parsep=0.2em,itemsep=0em]
		\item \textbf{CMRW1}: A variant of CM~\cite{Stoffregen_2019_CVPR} that uses a different objective function (called reward):
		\begin{equation*}    
		Rw(\vel) = C(\vel) + \dfrac{1}{P}\sum_{j=1}^P(e^{(-H(\bx_j;\vel))} +e^{(H(\bx_j;\vel))}).
		\end{equation*}
		The initial solution is taken as $\vel = 0$.
		\item \textbf{CMRW2}: Same as CMRW1 but initialised with the optimised $\vel$ from the previous subsequence.
		
		% \item \textbf{Reward fuction} proposed in~\cite{Stoffregen_2019_CVPR}, that optimized with gradient-based optimization. Same as CM, we compare with two initialization values, where we named CMRW1 and CMRW2. Same optimization library is used as CM.
		
		%\item \textbf{DNN} A learning approach for angular velocity estimation for event stream. We use our own implementation adapted from PointNet~\cite{qi2017pointnet} to our problem. EventNet~\cite{sekikawa2019eventnet}, which adapted from PointNet is used to solve the ego-motion estation problem, no open-source code by the author is available. Hence, we follow the framework of EventNet and modify PointNet for comparison.
		
		% \item A learning approach, PointNet~\cite{qi2017pointnet}. Unlike CM framework, DNNs is a model-free method to process events data. Most of the methods still rely on the presence of grey-scale image frames and apply 2D convolutions on it, e.g. EV-Flownet~\cite{zhu2018ev,zhu2019unsupervised} and ECN~\cite{ye2019learning}. While our CMBnB directly applies on event streams and no grey-scale images are required. Recently, DNNs that directly work for event streams have been proposed~\cite{qi2017pointnet,sekikawa2019eventnet}. As the open-source code by the author of Event-Net~\cite{sekikawa2019eventnet} is not published, we modify the implementation of PointNet to our problem. We believe that Event-Net will give a similar result.
	\end{itemize}
	We also compared against EventNet~\cite{sekikawa2019eventnet}, which is based on deep learning. However, similar to the error reported in~\cite{sekikawa2019eventnet}, we found that the error for EventNet was much higher than the error of the CM methods (e.g., the translated angular velocity error of the maximum of EventNet is 17.1\%, while it is around 5\% for CM). The lack of publicly available implementation also hampered objective testing of EventNet. We thus leave comparisons against deep learning methods as future work.
	
	\subsubsection{Rotational motion in indoor scene}

	We used event sequences \emph{poster}, \emph{boxes} and \emph{dynamic}  from~\cite{gallego2017accurate,mueggler2017event}, which were recorded using a Davis 240C~\cite{brandli2014240} under rotational motion over a static indoor scene. The ground truth motion was captured using a motion capture system. Each sequence has a duration of $1$ minute and around $100$ million events. For these sequences, the rotational motion was minor in a large part of the sequences (thereby producing trivial instances to CM), thus in our experiment we used only the final $15$ seconds of each sequence, which tended to have more significant motions.
	
	%We also evaluate the performance of different sequences of real data: \emph{poster}. \emph{boxes} and \emph{dynamic}~\cite{gallego2017accurate,mueggler2017event}. DAVIS~\cite{brandli2014240} is used to collect the event sets, which  has a spatial resolution of $240\times180$ pixels, and temporal resolution in microseconds. The ground truth pose was captured from a motion capture system (mocap), which operates at 200Hz. Each sequence of the real data has a length of 1 minute and has around 100 million events. Moreover, the angular speed increases with time. 

	%and a Nvidia RTX 2080Ti GPU. We used GPU only for DNN. 
	
	% While the bounding function in~\eqref{eq:pix_upbnd_disc} converges much slower and couldn't contact the termination criterion at $30,000$ iterations.
	%Our BnB algorithm terminates as soon as all the \emph{event regions} lie in one pixel or the upper bounds equals to the lower bounds, which takes too long. We manually set a threshold $\epsilon$, and our BnB algorithm terminates as soon as all the major-axis of the event ellipse is less than the threshold $\epsilon$. As we can see in Fig.~\ref{fig:iteration}, choosing the threshold to be $\epsilon = 0.1$ is good enough. More experiments on both synthesized and real data confirmed our observation. The majority computation in our method are rotation~\eqref{eq:warp} and disc of events~\eqref{eq:discs}. The calculation of upper bound~\eqref{eq:rqip} can be negligible as it's done greedily.   
	
	\begin{figure}[t]
		\centering
		\label{fig:RMS_result}
		\subfloat[\emph{boxes}]{\includegraphics[width=\linewidth]{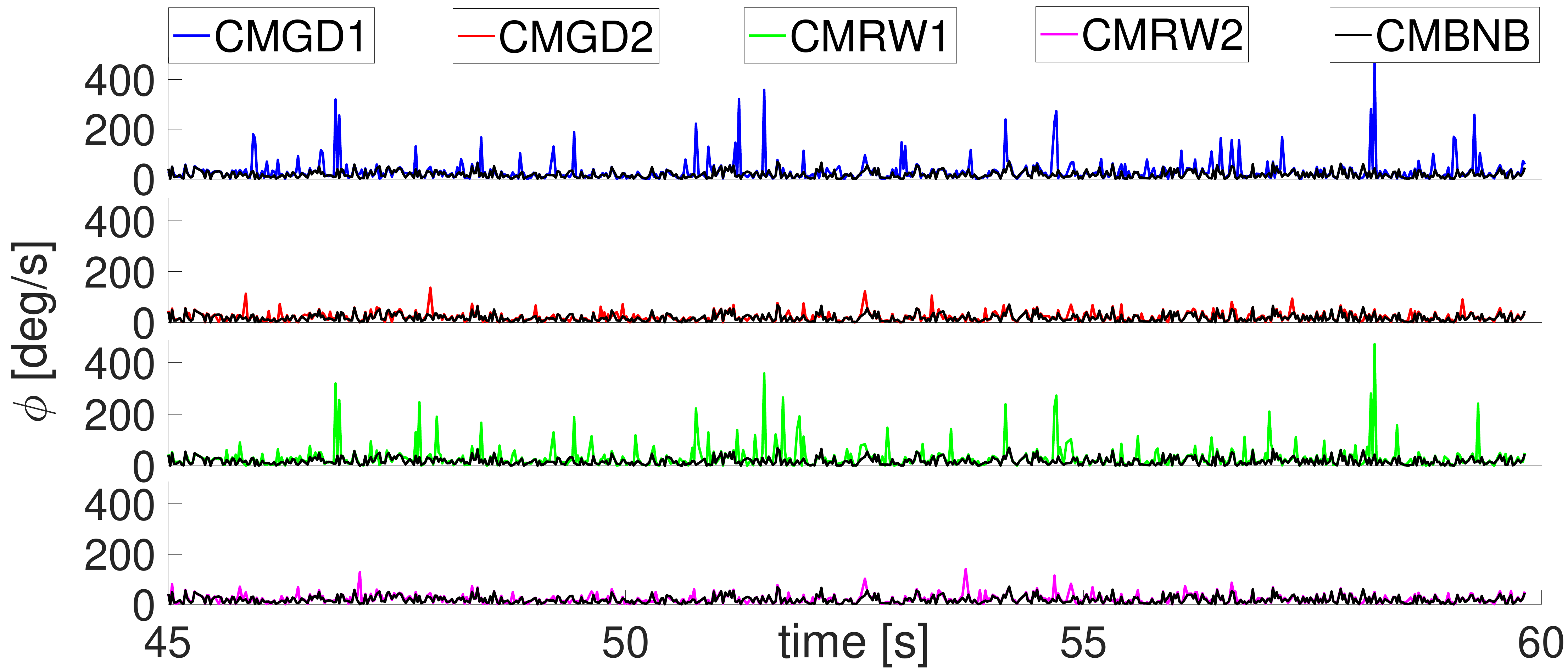}\label{fig:box_45_60}}\\
		\vspace{-1em}
		\subfloat[Sequence 1]{\includegraphics[width=0.95\linewidth]{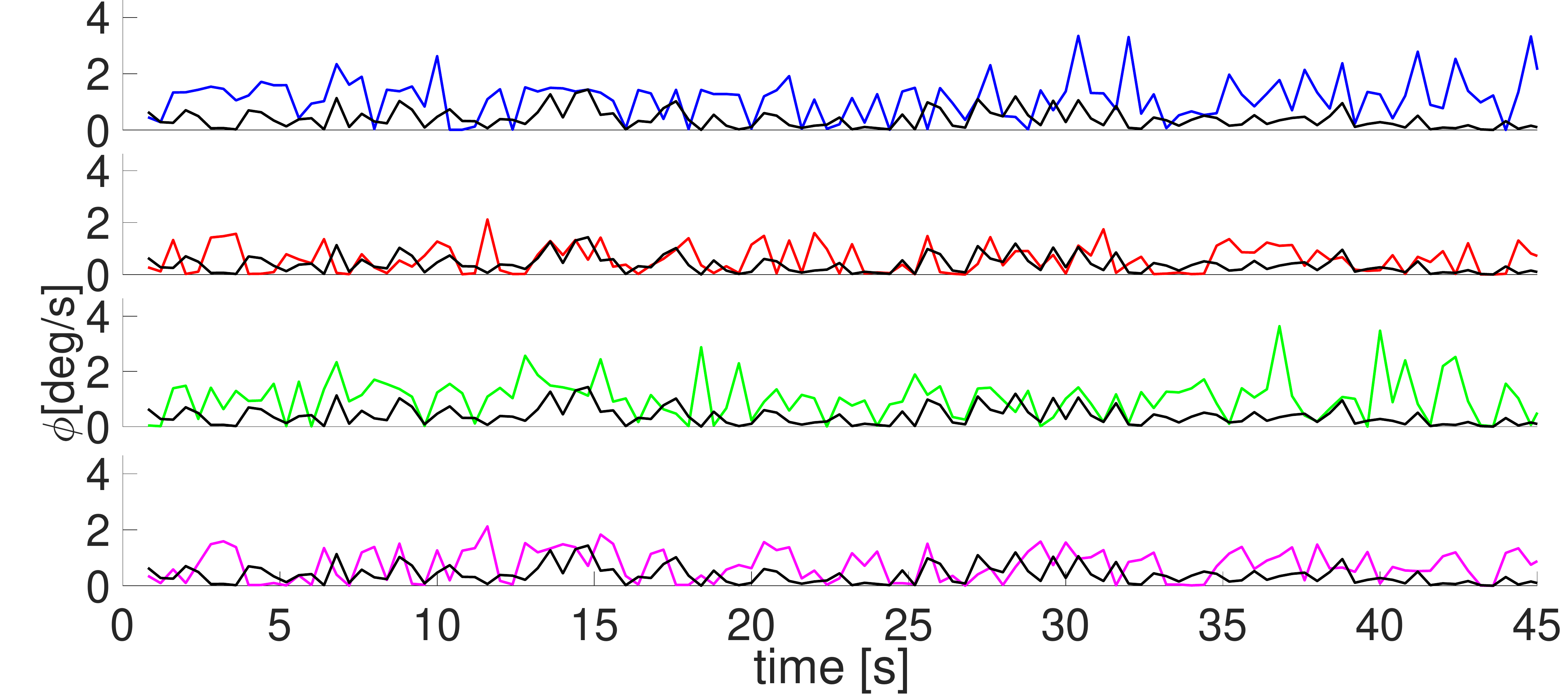}\label{fig:star_seq1}}
		\caption{Error $\phi$ of all CM methods (plotted against time) for (a) \emph{boxes} sequence~\cite{gallego2017accurate} and (b) Sequence 1 of~\cite{chin2019star}. In each subplot, the error of CMBNB is superimposed for reference.}
	\end{figure}

	We split each sequence into contiguous $10$ ms subsequences which were then subject to CM. For \emph{boxes} and \emph{poster}, each CM instance was of size $N\approx50k$, while for \emph{dynamic}, each instance was of size $N\approx25k$. For each CM instance, let $\bar{\vel}$ and $\hat{\vel}$ be the ground truth and estimated parameters. An error metric we used is
	\begin{align}
	\epsilon = \| \bar{\vel} - \hat{\vel} \|_2.
	\end{align}
	Our second error metric, which considers only differences in angular rate, is
	\begin{align}
	\phi = \abs{ \| \bar{\vel} \|_2 - \| \hat{\vel} \|_2}.
	\end{align}
	Fig.~\ref{fig:box_45_60} plots $\epsilon$ over all the CM instances for the \emph{boxes} sequence, while Table~\ref{tab:video} shows the average ($\mu$) and standard deviation ($\sigma$) of $\epsilon$ and $\phi$ over all CM instances.
	
	Amongst the local methods, the different objective functions did not yield significant differences in quality. The more important observation was that CMGD2/CMRW2 gave solutions of much higher quality than CMGD1/CMRW1, which confirms that good initialisation is essential for the local methods. Due to its exact nature, CMBnB provided the best quality in rotation estimation; its standard deviation of $\epsilon$ is also lower, indicating a higher stability over the sequences. Moreover, CMBnB does not require any initialisation, unlike the local methods.
	
	For the three sequences used (\emph{dynamic}, \emph{boxes}, \emph{poster}), the maximum absolute angular velocities are $500$, $670$ and $1000$ deg/s respectively~\cite{gallego2017accurate,mueggler2017event}. The average $\phi$ error of CMBnB of $10.09$, $17.97$ and $46.34$ deg/s thus translate into $2.2\%$, $2.7\%$ and $4.6\%$ of the maximum, respectively.
	
	% In this experiment, we fix the temporal window size to be 10 ms, where $N\approx50,000$ for \emph{box} and \emph{poster} sequence and $N\approx25000$ for \emph{dynamic} sequence The result summarized in Fig.~\ref{fig:result_45_60} and Tab.~\ref{tab:video} shows that CM optimized with gradient-based optimization has the fastest speed, however it's still not real time. The result neural network approach is not as good as other methods, which might be because of the lack of training data. We will further investigate the neural network approach in future work.  The result of the reward function proposed by~\cite{Stoffregen_2019_CVPR} is comparable to the original CM framework. Their performances are considerable with a good initialization — however, both at the starting point of each sequence and at the time they lose track, they must be initialized with 0. Moreover, even with a good initialization, it is still possible for the local optimization method to get into a wrong local optimal solution. Compare to local optimization; our CMBnB still has better accuracy in angular velocity, which never stuck in a wrong local optimal solution. Moreover, no prior information is required (initialization) for our CMBnB algorithm.
	
	\vspace{-1.5em}
	\paragraph{Runtime}
	
	The average runtimes of CMBnB over all instances in the three sequences (\emph{dynamic}, \emph{boxes}, \emph{poster}) were $163.2$, $278.3$ and $320.6$ seconds.
	% (the potential for GPU acceleration has been discussed in Sec.~\ref{sec:gpu}). 
	CMGD optimised with the conjugate gradient solver in \texttt{fmincon} has average runtimes of $20.2$, $31.1$ and $35.3$ seconds. 
	%See also Sec.~\ref{sec:gpu} on runtime comparisons.
	
	\begin{table*}
		\centering
		{\small
			\begin{tabular}{ |c|cccc|cccc|cccc| } 
				\hline
				\multirow{2}{*}{Method}&	\multicolumn{4}{c|}{\emph{dynamic}} & \multicolumn{4}{c|}{\emph{boxes}}  & \multicolumn{4}{c|}{\emph{poster}} \\
				\cline{2-13}
				& $\mu(\epsilon)$& $\mu(\phi)$ &$\sigma(\epsilon)$ &$\sigma(\phi)$& $\mu(\epsilon)$& $\mu(\phi)$ &$\sigma(\epsilon)$ &$\sigma(\phi)$& $\mu(\epsilon)$& $\mu(\phi)$ &$\sigma(\epsilon)$ &$\sigma(\phi)$\\
				\hline
				CMGD1 & 21.52 & 20.07 & 24.38 & 31.13 & 31.29 & 31.47 & 34.30 & 45.94 & 56.58 & 54.92 & 47.03 & 58.95 \\ 
				CMGD2 & 15.09 & 13.31 & 10.39 & 12.08 & 22.01 & 21.70 & 12.79 & 18.83 & 49.64 & 50.12 & 35.93 & 42.77\\ 
				CMRW1 & 21.03 & 18.59 & 25.41 & 28.83 & 32.28 & 32.23 & 36.11 & 46.01 & 59.03 & 58.71 & 49.49 & 60.87  \\ 
				CMRW2 & 14.55 & 12.29 & 9.85 & 11.21 & 21.95 & 21.41 & 13.71 & 18.42 & 49.49 & 50.04 & 37.51 & 43.35 \\
				CMBnB & \textbf{11.93} & \textbf{10.09} & \textbf{7.82} & \textbf{8.74} & \textbf{18.76} & \textbf{17.97} & \textbf{10.06}& \textbf{14.66} & \textbf{44.34} & \textbf{46.34} & \textbf{24.79} & \textbf{36.79} \\ 
				\hline
			\end{tabular}\\
		}
		\caption{Average and standard deviation of $\epsilon$ and $\phi$ over all CM instances in \emph{boxes}, \emph{dynamic}, and \emph{poster} (best result bolded).}
		\label{tab:video}
		
	\end{table*}
	\vspace{-1em}
	\subsubsection{Attitude estimation}
	
	We repeated the above experiment on the event-based star tracking (attitude estimation) dataset of~\cite{chin2019star,chin2019event}, which contains $11$ event sequences of rotational motion over a star field. Each sequence has a constant angular velocity of $4$ deg/s over a duration of $45$ seconds and around $1.5$ million events. We split each sequence into $400$ ms subsequences, which yielded $N \approx 15,000$ events per subsequence. Fig.~\ref{fig:star_seq1} plots the $\phi$ errors for Sequence 1 in the dataset. The average errors and standard deviation over all CM instances are shown in Table~\ref{tab:star}. Again, CMBnB gave the highest quality solutions; its average $\phi$ error of $0.234$ deg/s translate into $5.8\%$ of the maximum. The average runtime of CMBnB and CMGD over all instances were $80.7$ and $11.1$ seconds.
	
	\begin{table}\centering
		{\small
			\begin{tabular}{ |c|c c c c| } 
				\hline
				Method & $\mu(\epsilon)$& $\mu(\phi)$ &$\sigma(\epsilon)$ &$\sigma(\phi)$ \\
				\hline
				CMGD1 & 0.448 & 0.652 & 0.314 & 0.486 \\ 
				CMGD2 & 0.294 & 0.423 & 0.232 & 0.323 \\ 	
				CMRW1 & 0.429 & 0.601 & 0.346 & 0.468 \\ 
				CMRW2 & 0.318 & 0.461 & 0.234 & 0.341 \\
				CMBnB & \textbf{0.174} & \textbf{0.234} & \textbf{0.168} & \textbf{0.217} \\
				\hline
			\end{tabular}\\
		}
		\caption{Average and standard deviation of $\epsilon$ and $\phi$ over all CM instances in the star tracking dataset (best result bolded). 
		}
		\label{tab:star}
	\end{table}
	\section{Conclusions}
	We proposed a novel globally optimal algorithm for CM based on BnB. The theoretical validity of our algorithm has been established, and the experiments showed that it greatly outperformed local methods in terms of solution quality.
	
	\section*{Acknowledgements}
	This work was supported by ARC DP200101675. We thank S.~Bagchi for the implementation of CM and E.~Mueggler for providing the dataset used in experiments.
	\clearpage
	
	%This work is partly supported by

	%%%%%%%%% TITLE
	
	\onecolumn 
	\begin{center}
		\Large\textbf{{Supplementary Material:\\
				Globally Optimal Contrast Maximisation for Event-based Motion Estimation}}
	\end{center}
	\begin{multicols}{2} 
		\setcounter{section}{0}
		\newcommand{\bn}{\mathbf{n}}
		
		\renewcommand\thesection{\Alph{section}}

		\graphicspath{{figure/}}

		%%%%%%%%%%%%%%%%%%%%%%%%%%%%%%%%%%%%%%%%%%%%%%%%%%%%%%%%%%%%%%%%%%%%%%%%%%%%%%%%%%
		\section{Geometric derivations of the elliptical region}
		Here we present the analytic form of the centre $\bc$, semi-major axis $\by$, and semi-minor axis $\bz$ of the elliptical region $\cL$ (see Sec.~\ref{sec:bounding}) following the method in~\cite{liu20182d,clackdoyle2011centers} (subscript $i$  and explicit dependency on $\rbox$ are omitted for simplicity). See Fig.~\ref{fig:cone} in the main text for a visual representation of the aforementioned geometric entities.

		% We follow the method proposed by~\cite{liu20182d,clackdoyle2011centers} for calculating  the centre $\bc(\rbox)$, semi-major axis $\by(\rbox)$, and semi-minor axis $\bz(\rbox)$ of the elliptical region $\cL(\rbox)$). 
		\begin{enumerate}
			\item Calculate direction of the cone-beam %the unit rotated vector for $\vel_\bc$
			\begin{equation}
			\hat{\bu} = 
			\dfrac{\bR(t ; \vel_\bc) \tilde{\bu}}
			{\|\bR(t ; \vel_\bc) \tilde{\bu}\|_2},
			\end{equation}
			its radius %its uncertainty radius
			\begin{equation}
			r = \sin{\alpha({\rbox})},
			\end{equation}
			and the norm vector to the image plane $\hat{\bn} = [0\; 0\; 1]^T$.
			\item Calculate the semi-major axis direction within the cone-beam
			\begin{equation}
			\hat{\by} = \dfrac{\hat{\bu} \times (\hat{\bu} \times \hat{\bn})}{\|\hat{\bu} \times (\hat{\bu} \times \hat{\bn})\|}
			\end{equation}
			and semi-minor axis direction
			\begin{equation}
			\hat{\bz} = \dfrac{\hat{\by} \times \hat{\bn}}{\|\hat{\by} \times \hat{\bn}\|}.
			\end{equation}
			\item Calculate the intersecting points between the ray with the direction of the semi-major axis and the cone-beam% $\cV(\rbox)$
			\begin{equation}\label{eq:inty}
			\begin{aligned}
			\by^{(a)} &= \hat{\bu} - r \hat \by \\
			\by^{(b)} &= \hat{\bu} + r \hat \by, 
			\end{aligned}
			\end{equation}
			and the analogous points for the semi-minor axis
			%then the intersections between the vector of semi-minor axis $\overrightarrow{\by}$ and the cone beam $\cV(\rbox)$
			\begin{equation}\label{eq:intz}
			\begin{aligned}
			\bz^{(a)} &= \hat{\bu} - r \hat \bz \\
			\bz^{(b)} &= \hat{\bu} + r \hat \bz. 
			\end{aligned}
			\end{equation}
			%\item Project $\by^{(a)}$, $\by^{(b)}$, $\bz^{(a)}$ and $\bz^{(b)}$ into the image plane with the intrinsic matrix $\bK$ and get $\by(1)$, $\by(2)$, $\bz(1)$ and $\bz(2)$ respectively.
			
			\item Obtain $\by'^{(a)}$, $\by'^{(b)}$, $\bz'^{(a)}$ and $\bz'^{(b)}$ as the projection of~\eqref{eq:inty}~and~\eqref{eq:intz} into the image plane with the intrinsic matrix $\bK$.
			
			%\item Calculate $\bc = 0.5(\by(1) + \by(2))$, $\by = \|\by(1)-\by(2)\|_2$ and $\bz = \|\bz(1)-\bz(2)\|_2$.
			\item Calculate $\bc = 0.5(\by'^{(a)} + \by'^{(b)})$, $\by = \|\by'^{(a)}-\by'^{(b)}\|_2$, and $\bz = \|\bz'^{(a)}-\bz'^{(b)}\|_2$.
		\end{enumerate}

		%%%%%%%%%%%%%%%%%%%%%%%%%%%%%%%%%%%%%%%%%%%%%%%%%%%%%%%%%%%%%%%%%%%%%%%%%%%%%%%%%%
		\section{Proofs}
		%%%%%%%%%%%%%%%%%%%%%%%%%%%%%%%%%%%%%%%%%%%%%%%%%%%%%%%%%%%%%%%%%%%%%%%%%%%%%%%%%%
		We state our integer quadratic problem again. %as multiple times referenced in the following proofs.
		\begin{align}\tag{IQP}\label{eq:iqp}
		\begin{aligned}
		\overline{S}^*_d(\rbox) = \max_{\bZ \in \{0,1\}^{N \times K}} \quad & \sum^K_{k=1} \left( \sum_{i=1}^N \bZ_{i,k} \bM_{i,k} \right)^2 \\
		\text{s.t.} \quad & \bZ_{i,k} \le \bM_{i,k}, \;\; \forall i,k, \\
		& \sum_{k=1}^K \bZ_{i,k} = 1, \;\; \forall i.
		\end{aligned}
		\end{align}
		
		\subsection{Proof of Lemma~1 in the main text}
		\begin{lemma}\label{lem:pix_upbnd}
			\begin{align}\label{eq:lem1}
			\overline{H}_c(\bx_j;\rbox) \ge \max_{\vel\in \mathbb{B}}~H_c(\bx_j;\vel)
			\end{align}
			with equality achieved if $\rbox$ is singleton, i.e., $\rbox = \{ \vel \}$.
		\end{lemma}
		\begin{proof}
			This lemma can be demonstrated by contradiction. Let $\vel^*$ be the optimiser for the RHS of~\eqref{eq:lem1}. If 
			\begin{equation} \label{eq:lem1_cont}
			H_c(\bx_j;\vel^*) > \overline{H}_c(\bx_j;\rbox),
			\end{equation}
			it follows from the definition of pixel intensity (Eq.~\eqref{eq:rt} ) and its upper bound (Eq.~\eqref{eq:pixelvalupbnd}) that 
			\begin{equation}
			\label{eq:lem1_violet}
			\|\bx_j-f(\bu_i,t_i,\vel^*)\| < \max\left(\|\bx_j-\bc_i(\rbox)\|-\|\by_i(\rbox)\|,0\right),
			\end{equation}
			for at least one $i=1,\ldots, N$.
			
			% then after considering the definitions 
			% \begin{equation}\label{eq:rt}
			%     H_c(\bx_j;\vel) = \sum_{i=1}^N\delta(\|\bx_j-f(\bu_i, t_i ;\, \vel)\|_2),
			% \end{equation}
			% and 
			% \begin{equation}\label{eq:pixelvalupbnd}
			% \overline{H}_c(\bx_j;\rbox) = \sum_{i=1}^N\delta\left(\max\left(\|\bx_j-\bc_i(\rbox)\|-\|\by_i(\rbox)\|,0\right)\right).
			% \end{equation}
			% of the two sides of~\eqref{eq:lem1_cont}, we can establish that 
			% \begin{equation}
			%     \label{eq:lem1_violet}
			%     \|\bx_j-f(\bu_i,t_i,\vel^*)\| < \max\left(\|\bx_j-\bc_i(\rbox)\|-\|\by_i(\rbox)\|,0\right),
			% \end{equation}
			% for at least one $i=1,\ldots, N$.
			
			% In words, the shortest distance between $\bx_j$ and the disc $\cD_i(\rbox)$ is greater than the distance between $\bx_j$ and the optimal position $f(\bu_i,t_i,\vel^*)$. However, $f(\bu_i,t_i,\vel^*)$ is always inside the disc $\cD_i(\rbox)$, and hence Eq.~\eqref{eq:lem1_violet} cannot hold. If $\rbox = \{ \vel \}$, then from definition~\eqref{eq:pixelvalupbnd} $\overline{H}_c(\bx_j;\rbox) = H_c(\bx_j;\vel)$.
			
			In words, the shortest distance between $\bx_j$ and the disc $\cD_i(\rbox)$ is greater than the distance between $\bx_j$ and the optimal position $f(\bu_i,t_i,\vel^*)$. However, $f(\bu_i,t_i,\vel^*)$ is always inside the disc $\cD_i(\rbox)$, and hence Eq.~\eqref{eq:lem1_violet} cannot hold. If $\rbox = \{ \vel \}$, then from definition \eqref{eq:pixelvalupbnd} in the main text $\overline{H}_c(\bx_j;\rbox) = H_c(\bx_j;\vel)$.
		\end{proof}

		\subsection{Proof of Lemma~2 in the main text}
		\begin{lemma}\label{lem:IQP}
			\begin{align}
			\label{eq:qip_upper}
			\overline{S}^*_d(\rbox) \ge \max_{\vel\in \rbox}~\sum_{j=1}^P H_d(\bx_j;\vel)^2,
			\end{align}
			with equality achieved if $\rbox$ is singleton, i.e., $\rbox = \{ \vel \}$.
		\end{lemma}
		\begin{proof}
			We pixel-wisely reformulate~\ref{eq:iqp}:
			\begin{align}\tag{P-IQP}\label{eq:pqip}
			\begin{aligned}
			\overline{S}^*_d(\rbox) = \max_{\bQ \in \{0,1\}^{N \times P}} \quad & \sum^P_{j=1} \left( \sum_{i=1}^N \bQ_{i,j}  \right)^2 \\
			\text{s.t.} \quad & \bQ_{i,j} \le \bT_{i,j}, \;\; \forall i,j, \\
			& \sum_{j=1}^P \bQ_{i,j} = 1, \;\; \forall i,
			\end{aligned}
			\end{align}
			and we express the RHS of~\eqref{eq:qip_upper} as a mixed integer quadratic program:
			\begin{align}\tag{MIQP}\label{eq:mqip}
			\begin{aligned}
			\max_{\vel \in \rbox,\bQ \in \{0,1\}^{N \times P}} \quad & \sum^P_{j=1} \left( \sum_{i=1}^N \bQ_{i,j} \right)^2 \\
			\text{s.t.} \quad & \bQ_{i,j} = \mathbb{I}(f(\bu_i, t_i ;\, \vel)~\text{in}~\bx_j),\; \forall i,j.
			\end{aligned}
			\end{align}
			Problem \ref{eq:pqip} is a relaxed version of~\ref{eq:mqip} - hence~\eqref{eq:qip_upper} holds - as for every $\bm{e}_i$, the feasible pixel $\bx_j$ is in $\cD_i(\rbox)$; whereas for \ref{eq:mqip}, the feasible pixel is dictated by a single $\vel\in\rbox$. If $\rbox$ collapses into $\vel$, every event $\bm{e}_i$ can intersect only one pixel $\bx_j$, hence $\bT_{i,j} = \mathbb{I}(f(\bu_i, t_i ;\, \vel)~\text{in}~\bx_j),\; \forall i,j$;  $\sum_{j=1}^P \bT_{i,j} = 1, \; \forall i$; and  $\sum_{j=1}^P \bQ_{i,j} = 1  \implies \bQ_{i,j} = \bT_{i.j}, \forall i$; therefore, \ref{eq:mqip} is equivalent to~\ref{eq:pqip} if $\rbox = \{\vel\}$.

		\end{proof}

		\subsection{Proof of Lemma 3 in the main text}
		\begin{lemma}%\label{lem:M2}
			Problem~\ref{eq:iqp} has the same solution if $\bM$ is replaced with $\bM^\prime$.
		\end{lemma}
		\begin{proof}
			
			We show that removing an arbitrary non-dominant column from $\bM$ does not change the solution of~\ref{eq:iqp}. Without loss of generality, assume the last column of $\bM$ is non-dominant. Equivalent to solving~\ref{eq:iqp} on $\bM$ without its last column is the following~\ref{eq:iqp} reformulation:
			\begin{subequations}\label{eq:qip_remove_non}
				\begin{align}
				\overline{S}^*_d(\rbox) = \max_{\bZ \in \{0,1\}^{N \times K}} \quad & \sum^{K-1}_{k=1} \left( \sum_{i=1}^N \bZ_{i,k} \bM_{i,k} \right)^2 +\\
				& \left( \sum_{i=1}^N \bZ_{i,K} \bM_{i,K} \right)^2\\
				\text{s.t.} 
				\quad & \bZ_{i,k} \le \bM_{i,k}, \;\; \forall i,k, \label{eq:iqp2_c1}\\
				%& \bZ_{i,K} = \bM_{i,K} = 0,\;\; \forall i,\label{eq:remcol}\\
				& \sum_{k=1}^{K} \bZ_{i,k} = 1, \;\; \forall i, \label{eq:lm3_c3}\\
				& \bZ_{i,K} =  0,\;\; \forall i,\label{eq:remcol}
				\end{align}
				\label{eq:iqp2}
			\end{subequations}
			which is same as~\ref{eq:iqp} but with additional constraint~\eqref{eq:remcol}. Since $\bM_{:,K}$ is non-dominant, it must exists a dominant column $\bM_{:,\eta}$ such that
			\begin{equation}
			\label{eq:dominate}
			\bM_{i,K} \leq \bM_{i,\eta},\;\; \forall i.
			\end{equation}
			%where the additional constraint~\eqref{eq:remcol} forces all the elements in the last column of $\bM$ (and $\bZ$) to be $0$. Since $\bM_{:,K}$ is non-dominant, it must exists a dominant column $\bM_{:,\eta}$ such that
			%For an arbitrary \emph{non-dominant} column $\bM_{:,K}$, always exist a \emph{dominant} column $\bM_{:,\eta}$ such that
			%Hence, if $\bZ^*_{i_1,K} = \bM_{i_1,K} = 1$, where $\bZ^*$ the optimiser of~\ref{eq:iqp}. Then $\bZ_{i_i,\eta} = 0$ due to~\eqref{eq:lm3_c3}.
			Hence, if $\bM_{i,K} = 1$, then $\bM_{i,\eta} = 1$ must holds $\forall i$. Let $\bZ^*$ be the optimiser of~\ref{eq:iqp} with $\bZ^*_{i_a,K},\ldots,\bZ^*_{i_b,K} = 1$. Let define $\bZ'^*$ same as $\bZ^*$ but with $\bZ'^*_{:,K}=\mathbf{0}$ and $\bZ'^*_{i_a,\eta},\ldots,\bZ'^*_{i_b,\eta} = 1$. In words, we ``move'' the $1$ values from the last column to its dominant one. We show that $\bZ'^*$ is an equivalent solution (same objective value than $\bZ^*$). $\bZ'^*$ is feasible since~\eqref{eq:dominate} ensures condition~\eqref{eq:iqp2_c1}, \eqref{eq:lm3_c3} is not affected by ``moving ones'' in the same row, and~\eqref{eq:remcol} is true for the definition of $\bZ'^*$. Finally we show that 
			\begin{equation}\label{eq:same}
			\sum_{i=1}^N \bZ^*_{i,K} \bM_{i,K} = \sum_{i=1}^N \bZ'^*_{i,\eta} \bM_{i,\eta} 
			\end{equation}
			therefore $\bZ'^*$ produces same objective value than~\ref{eq:iqp}. We prove \eqref{eq:same} by contradiction. Assume exists at least one $i' \not\in \{i_a, \ldots, i_b \}$ such that $\bZ^*_{i',\eta} = 1 \implies \bZ'^*_{i',\eta} = 1$. Then, $\bZ'^*$ produces a larger objective value than $\bZ^*$ which is a contradiction since problem~\eqref{eq:iqp2} is most restricted than~\ref{eq:iqp}. Thus, removing any arbitrary non-dominant column will not change the solution which implies this is also true if we remove all non-dominant columns (\ie, if we replace $\bM$ with $\bM'$). 

		\end{proof}
		
		\subsection{Proof of Lemma 4 in the main text}
		\begin{lemma}
			\begin{align}
			\overline{S}_d(\rbox) \ge \overline{S}^*_d(\rbox)
			\label{eq:lm4}
			\end{align}
			with equality achieved if $\rbox$ is singleton, i.e., $\rbox = \{ \vel \}$.
		\end{lemma}
		\begin{proof}
			To prove~\eqref{eq:lm4}, it is enough to show 
			\begin{align}\tag{R-IQP}\label{eq:rqip}
			\begin{aligned}
			\overline{S}_d(\rbox) = \max_{\bZ \in \{0,1\}^{N \times K^\prime}} \quad & \sum^{K^\prime}_{k=1} \left( \sum_{i=1}^N \bZ_{i,k}\bM^\prime_{i,k} \right)^2 \\
			\text{s.t.} \quad & \bZ_{i,k} \le \bM^\prime_{i,k}, \;\; \forall i,k, \\
			& \sum_{k=1}^{K^\prime} \sum^{N}_{i=1} \bZ_{i,k} = N,
			\end{aligned}
			\end{align}
			is a valid relaxation of~\ref{eq:iqp}. This is true as the constraint $\sum_{k=1}^{K^\prime}\sum^{N}_{i=1} \bZ_{i,k} = N$ in \ref{eq:rqip} is a necessary but not sufficient condition for the constraints $\sum_{k=1}^{K^\prime} \bZ_{i,k} = 1, \forall i$ in~\ref{eq:iqp}. If $\rbox$ collapse into $\vel$, every event $\bm{e}_i$ can intersect only one CC $\cG_k \implies \sum_{k=1}^{K^\prime} \bZ_{i,k} = 1$; hence, \ref{eq:rqip} is equivalent to~\ref{eq:iqp}.
		\end{proof}
		
		%Correctness of the continuous lower bound
		\subsection{Proof of lower bound~\eqref{eq:pixelvaluelwbnd} in the main text}
		\begin{lemma}\label{lem:pix_lwbnd}
			\begin{align}\label{eq:lem5}
			\underline{H}_c(\bx_j;\rbox) \le \min_{\vel\in \mathbb{B}}~H_c(\bx_j;\vel)
			\end{align}
			with equality achieved if $\rbox$ is singleton, i.e., $\rbox = \{ \vel \}$.
		\end{lemma}
		\begin{proof}
			Analogous to Lemma~\ref{lem:pix_upbnd}, we prove this Lemma by contraction. Let $\vel^*$ be the optimiser for the RHS of~\eqref{eq:lem5}. If
			\begin{equation} \label{eq:lem5_cont}
			H_c(\bx_j;\vel^*) < \underline{H}_c(\bx_j;\rbox),
			\end{equation}
			it follows from the definition of pixel intensity (Eq.~\eqref{eq:rt}) and its lower bound (Eq.~\eqref{eq:pixelvaluelwbnd}) that 
			\begin{equation}
			\label{eq:lem5_violet}
			\|\bx_j-f(\bu_i,t_i,\vel^*)\| > \|\bx_j-\bc_i(\rbox)\|+\|\by_i(\rbox)\|,
			\end{equation}
			for at least one $i=1,\ldots, N$.
			
			In words, the longest distance between $\bx_j$ and the disc $\cD_i(\rbox)$ is less than the distance between $\bx_j$ and the optimal position $f(\bu_i,t_i,\vel^*)$. However, $f(\bu_i,t_i,\vel^*)$ is always inside the disc $\cD_i(\rbox)$, and hence Eq.~\eqref{eq:lem5_violet} cannot hold. If $\rbox = \{ \vel \}$, then from definition~\eqref{eq:pixelvaluelwbnd} in the main text $\underline{H}_c(\bx_j;\rbox) = H_c(\bx_j;\vel)$.
			
			% then after considering the definitions~\eqref{eq:rt}
			% and 
			% \begin{align}\label{eq:pixelvaluelwbnd}
			% \underline{H}_c(\bx_j ; \rbox) = \sum_{i=1}^N \delta\left(\|\bx_j-\bc_i(\rbox)\|+\|\by_i(\rbox)\|\right),
			% \end{align}
			% of the two sides of~\eqref{eq:lem5_cont}, we can establish that 
			% \begin{equation}
			%     \label{eq:lem5_violet}
			%     \|\bx_j-f(\bu_i,t_i,\vel^*)\| > \|\bx_j-\bc_i(\rbox)\|+\|\by_i(\rbox)\|,
			% \end{equation}
			% for at least one $i=1,\ldots, N$.
			
			% In words, the longest distance between $\bx_j$ and the disc $\cD_i(\rbox)$ is less than the distance between $\bx_j$ and the optimal position $f(\bu_i,t_i,\vel^*)$. However, $f(\bu_i,t_i,\vel^*)$ is always inside the disc $\cD_i(\rbox)$, and hence Eq.~\eqref{eq:lem5_violet} cannot hold. If $\rbox = \{ \vel \}$, then from definition~\eqref{eq:pixelvaluelwbnd} $\underline{H}_c(\bx_j;\rbox) = H_c(\bx_j;\vel)$.

		\end{proof}
		
		%\begin{lemma}
		%\begin{equation}\label{eq:discrete_lwbnd}
		%    \underline{\mu}_c(\rbox) \le \min_{\vel\in \rbox}~\dfrac{1}{P}\sum_{j=1}^P H_c(\bx_j;\vel),
		%\end{equation}
		%with equality achieved if $\rbox$ is singleton, i.e., $\rbox = \{ \vel \}$.
		%\end{lemma}
		%\begin{proof}
		%\hl{should we proof that?}
		%\end{proof}
		
		\subsection{Proof of lower bound~\eqref{eq:lowerbound} in the main text}
		\begin{lemma}
			\begin{equation}\label{eq:discrete_lwbnd}
			\underline{\mu}_d(\rbox) \le \min_{\vel\in \rbox}~\dfrac{1}{P}\sum_{j=1}^P H_d(\bx_j;\vel),
			\end{equation}
			with equality achieved if $\rbox$ is singleton, i.e., $\rbox = \{ \vel \}$.
		\end{lemma}
		\begin{proof}
			This lemma can be demonstrated by contraction. Let $\vel^*$ be the optimiser of the RHS of~\eqref{eq:discrete_lwbnd}. If 
			\begin{equation}\label{eq:lowd_cont}
			\dfrac{1}{P}\sum_{j=1}^P H_d(\bx_j;\vel^*) < \underline{\mu}_d(\rbox),
			\end{equation}
			after replacing the pixel intensity and the lower bound pixel value with they definitions (Eqs.~\eqref{eq:rt_discrete} and~\eqref{eq:lowerbound}) in \eqref{eq:lowd_cont}, it leads to
			\begin{subequations}\label{eq:lowd_fcont}
				\begin{align}
				%&& \sum_{j=1}^P\sum_{i=1}^N \mathbb{I}(f(\bu_i, t_i ;\, \vel^*)~\text{lies in pixel}~\bx_j)  \label{eq:lowd_fcont1} \\
				&& 
				\sum_{i=1}^N\sum_{j=1}^P \mathbb{I}(f(\bu_i, t_i ;\, \vel^*)~\text{lies in pixel}~\bx_j)\label{eq:lowd_fcont2}\\
				&<& \sum_{i=1}^N\mathbb{I}(\cD_i~\text{fully lie in the image plane}).\label{eq:lowd_fcont3}
				\end{align}
			\end{subequations}
			In words, for every warped event $f(\bu_i, t_i ;\, \vel^*) \in \cD_i$ that lies in any pixel $\bx_j\in X$ of the image plane,  the discs $\cD_i$ must fully lie in the image plane. Since \eqref{eq:lowd_fcont2} is a less restricted problem than \eqref{eq:lowd_fcont3},  \eqref{eq:lowd_cont} cannot hold. If $\rbox = \{\vel\}$, $\cD_i = f(\bu_i, t_i ;\, \vel)$; therefore, the two sides in \eqref{eq:discrete_lwbnd} are equivalent.

		\end{proof}

		%%%%%%%%%%%%%%%%%%%%%%%%%%%%%%%%%%%%%%%%%%%%%%%%%%%%%%%%%%%%%%%%%%%%%%%%%%%%%%%%%%
		\section{Additional qualitative results}
		Figs.~\ref{fig:boxes}, \ref{fig:dynamic} and \ref{fig:poster} show additional motion compensation results (Sec.~\ref{sec:global_local_exp} in the main text) for subsequences from \emph{boxes}, \emph{dynamic}  and \emph{poster}. 
	\end{multicols} %
	
	%To show all used sequences () in Sec.~4.2, qualitative results for motion compensation on \emph{boxes}, \emph{dynamic} are plotted in  and \emph{poster}, respectively.
	
	\begin{figure}
		%\begin{minipage}{\textwidth}
		\renewcommand{\arraystretch}{0.7}
		\setlength{\tabcolsep}{1mm}
		\footnotesize
		\begin{tabularx}{\textwidth} { 
				>{\raggedright\arraybackslash}C{.05} 
				| >{\centering\arraybackslash}X 
				| >{\centering\arraybackslash}X 
				| >{\centering\arraybackslash}X 
				| >{\centering\arraybackslash}X  }
			& CMBnB1
			& CMBnB2
			& CMGD1
			& CMGD2\\
			%\hline
			% \RotText{~~~~ Subseq 1}
			%& \includegraphics[width=\linewidth]{figure/exp1_BNB1.eps}  
			%& \includegraphics[width=\linewidth]{figure/exp1_BNB2.eps}
			% & \includegraphics[width=\linewidth]{figure/exp1_GN1.eps}
			%& \includegraphics[width=\linewidth]{figure/exp1_GN2.eps}\\
			%\hline
			%&&&&\\[-.15cm]
			%  \RotText{~~~~ Subseq 2}
			% & \includegraphics[width=\linewidth]{figure/exp2_BNB1.eps}  
			% & \includegraphics[width=\linewidth]{figure/exp2_BNB2.eps}
			% & \includegraphics[width=\linewidth]{figure/exp2_GN1.eps}
			% & \includegraphics[width=\linewidth]{figure/exp2_GN2.eps}\\
			% \hline
			% &&&&\\[-.15cm]
			\RotText{~~~~ Subseq 3}
			& \includegraphics[width=\linewidth]{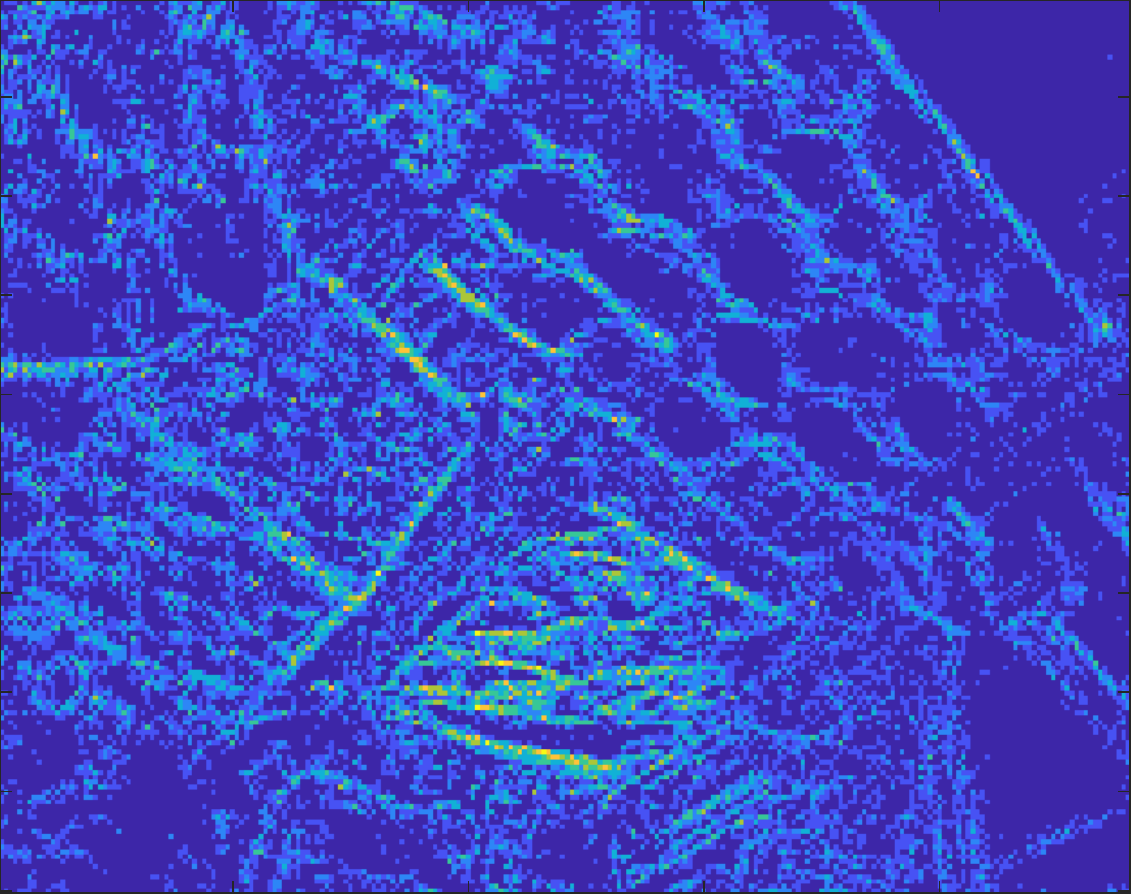}  
			& \includegraphics[width=\linewidth]{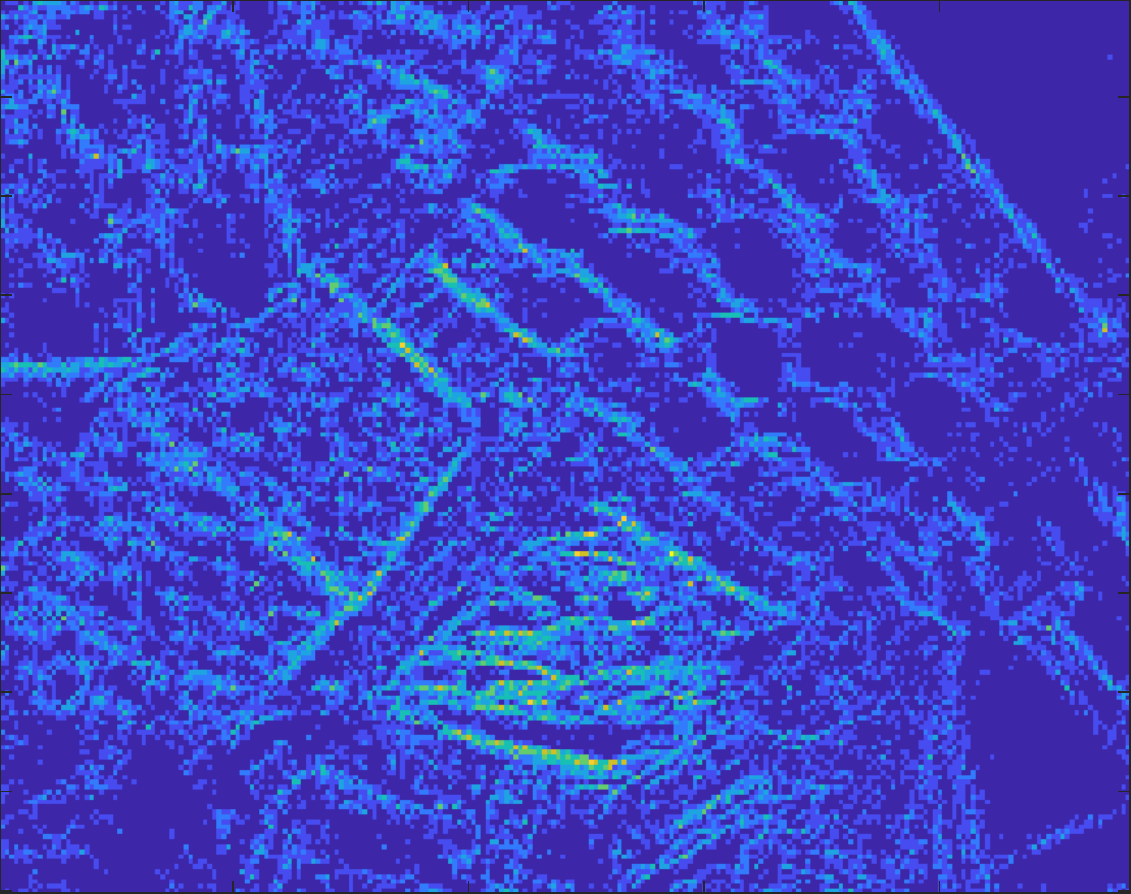}
			& \includegraphics[width=\linewidth]{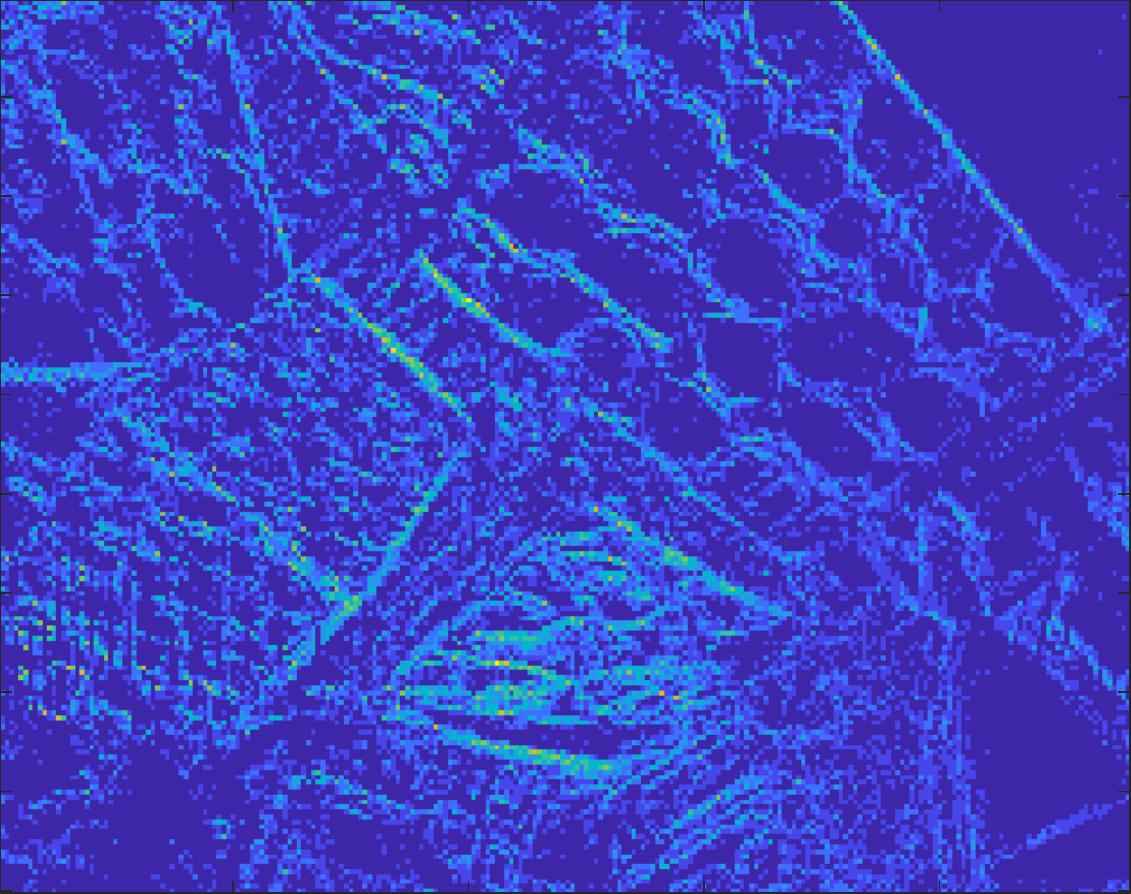}
			& \includegraphics[width=\linewidth]{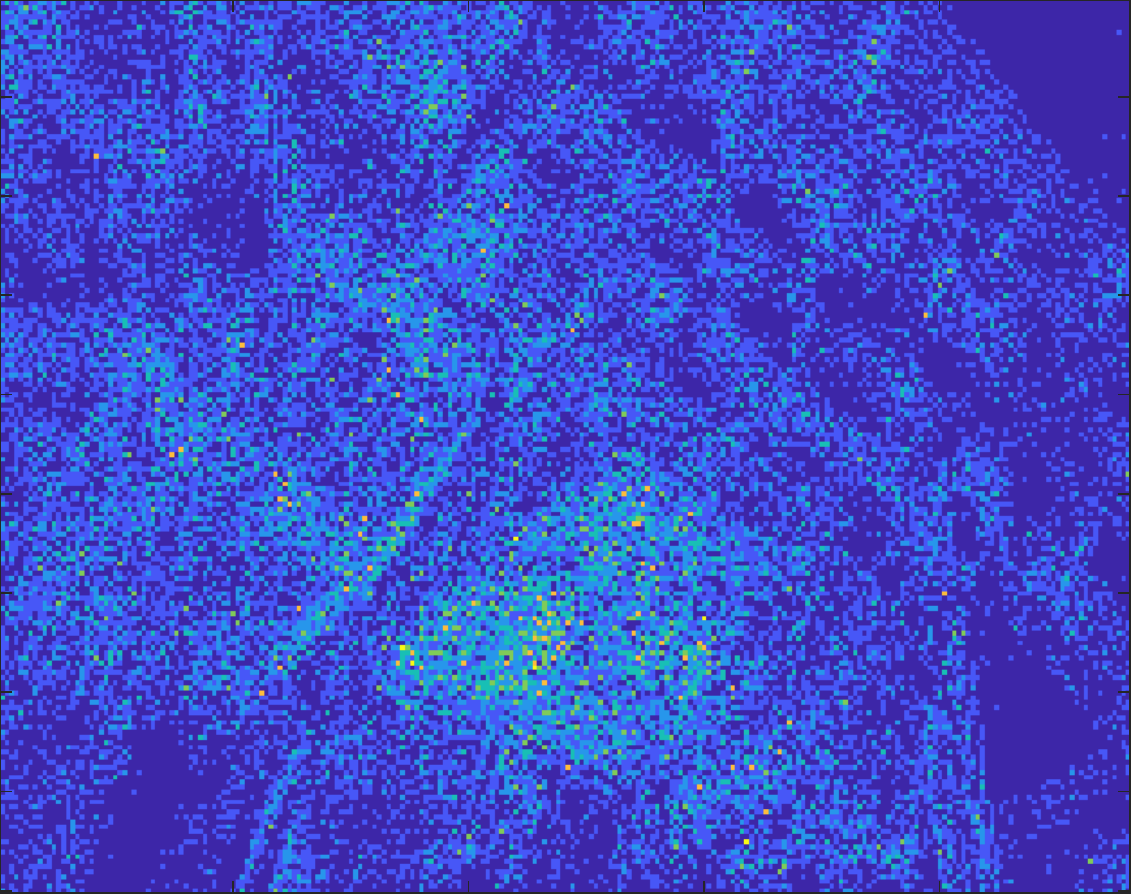}\\
			\hline
			&&&&\\[-.15cm]
			\RotText{~~~~ Subseq 4}
			& \includegraphics[width=\linewidth]{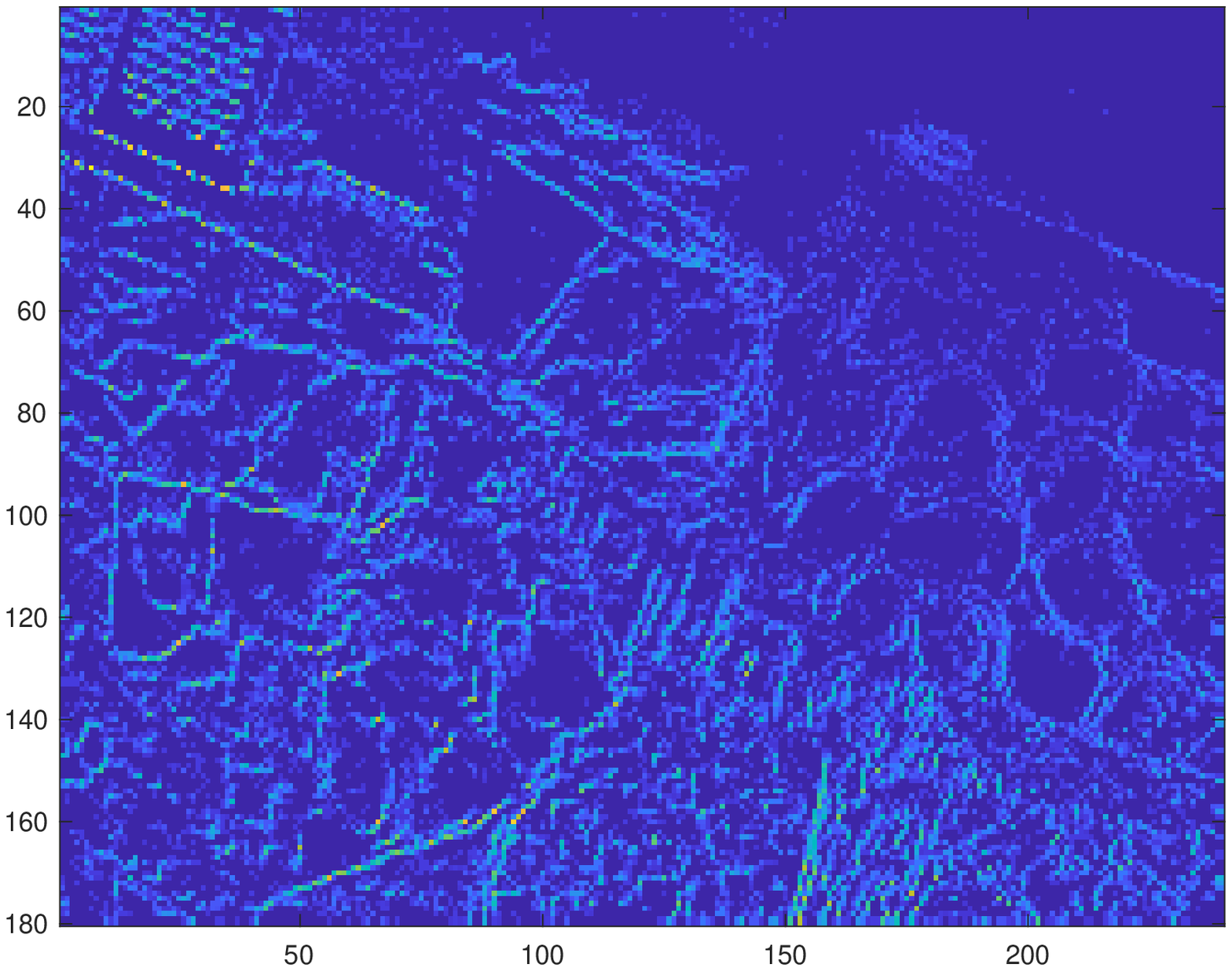}  
			& \includegraphics[width=\linewidth]{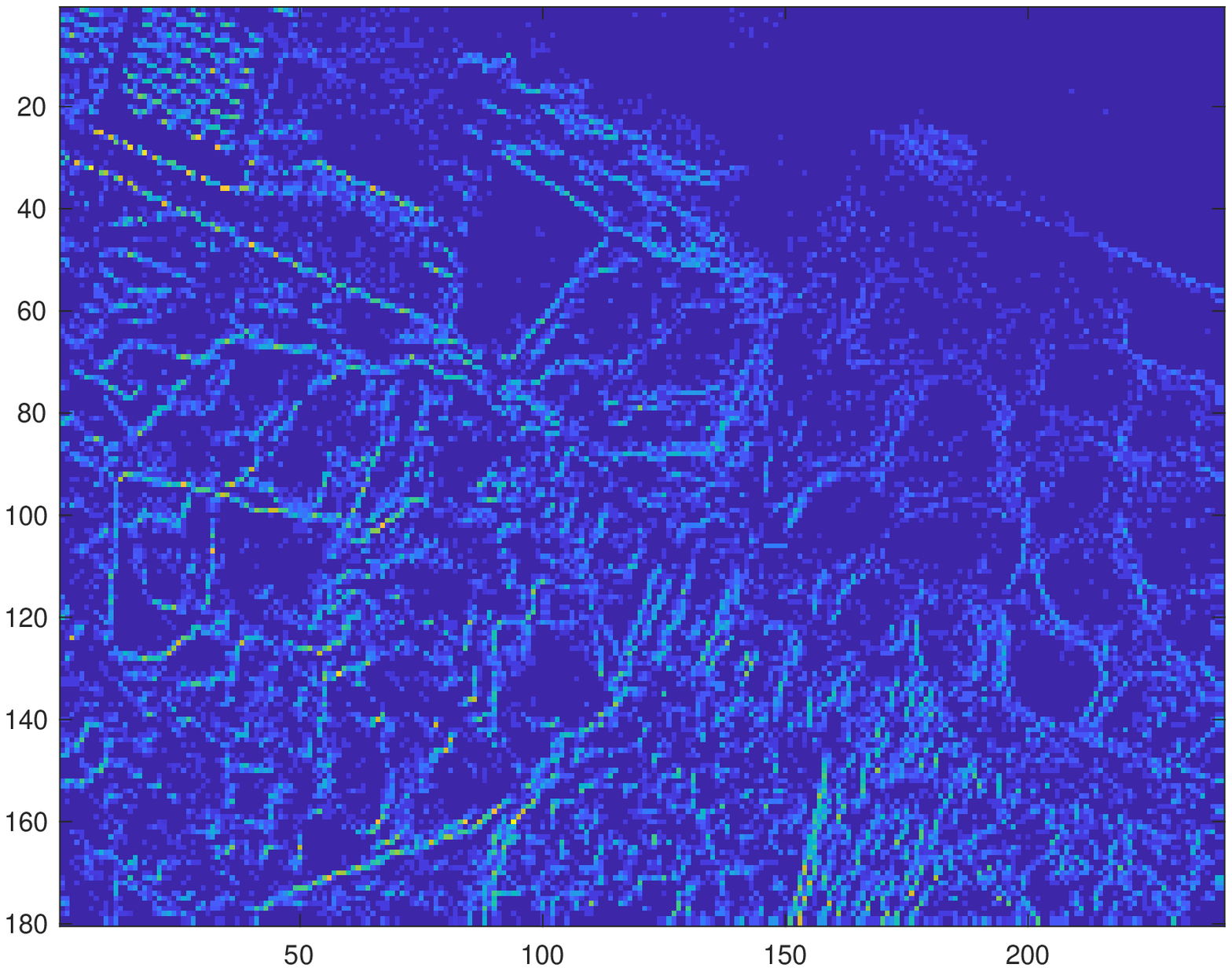}
			& \includegraphics[width=\linewidth]{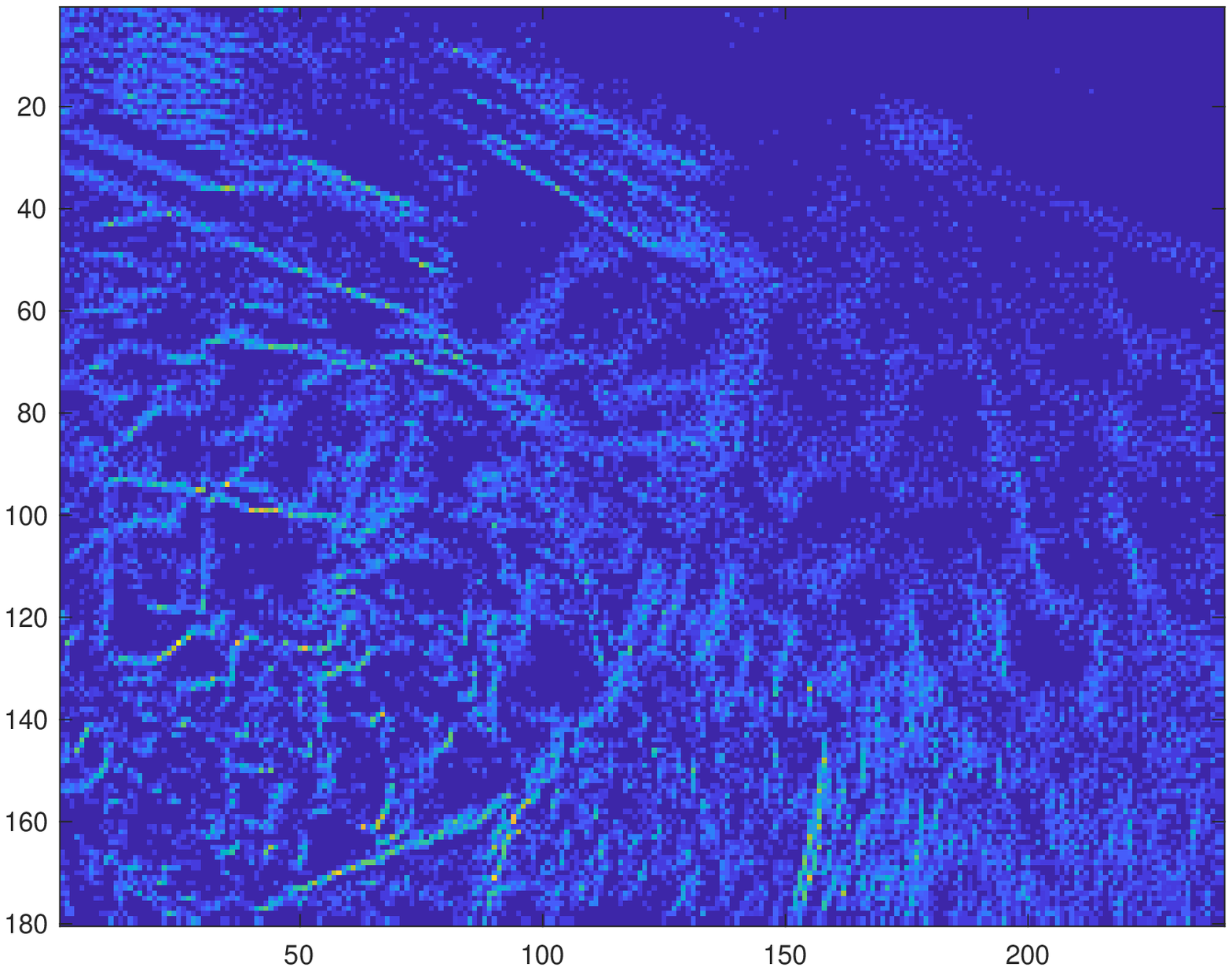}
			& \includegraphics[width=\linewidth]{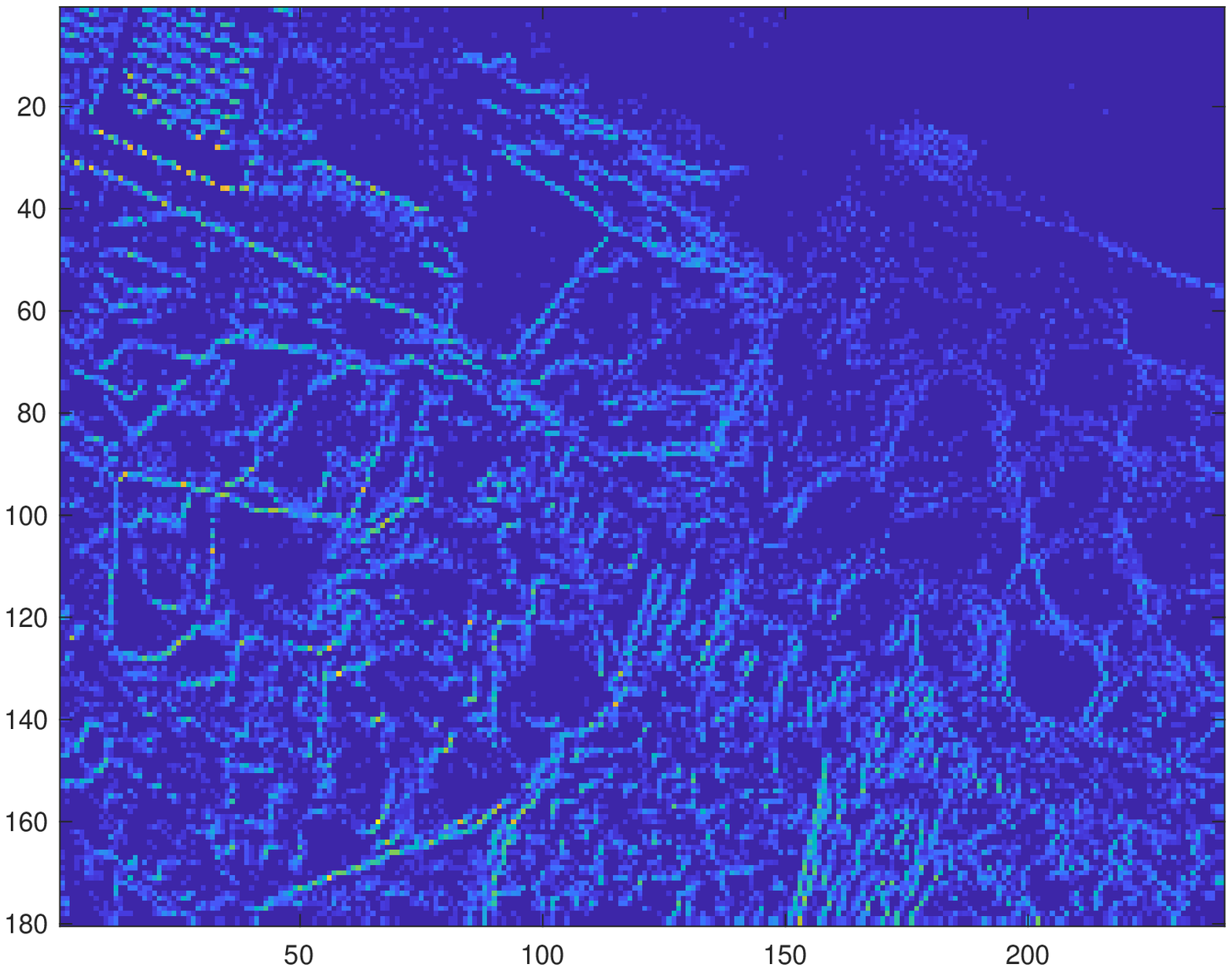}\\
			\hline
			&&&&\\[-.15cm]
			\RotText{~~~~Subseq 5}
			& \includegraphics[width=\linewidth]{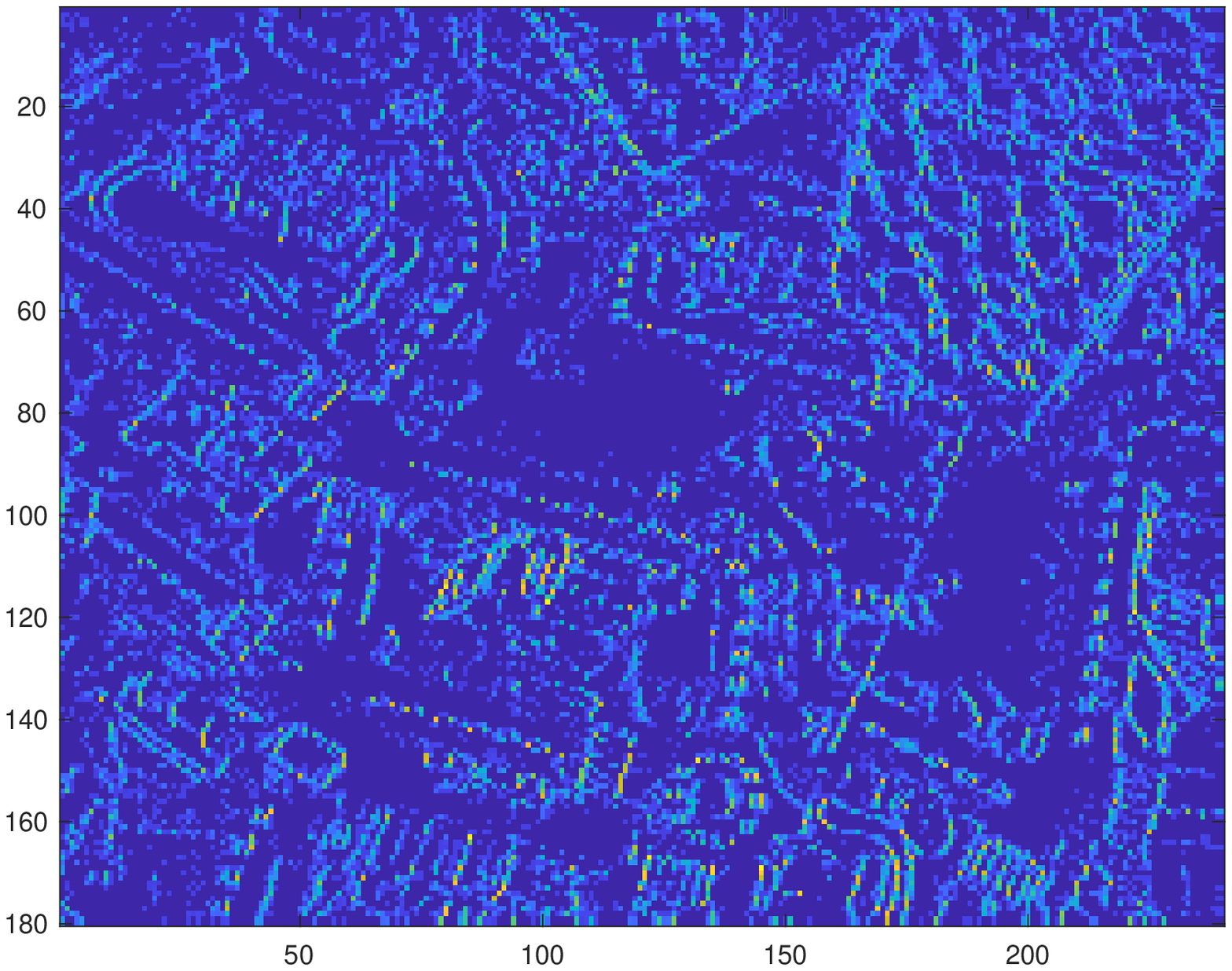}  
			& \includegraphics[width=\linewidth]{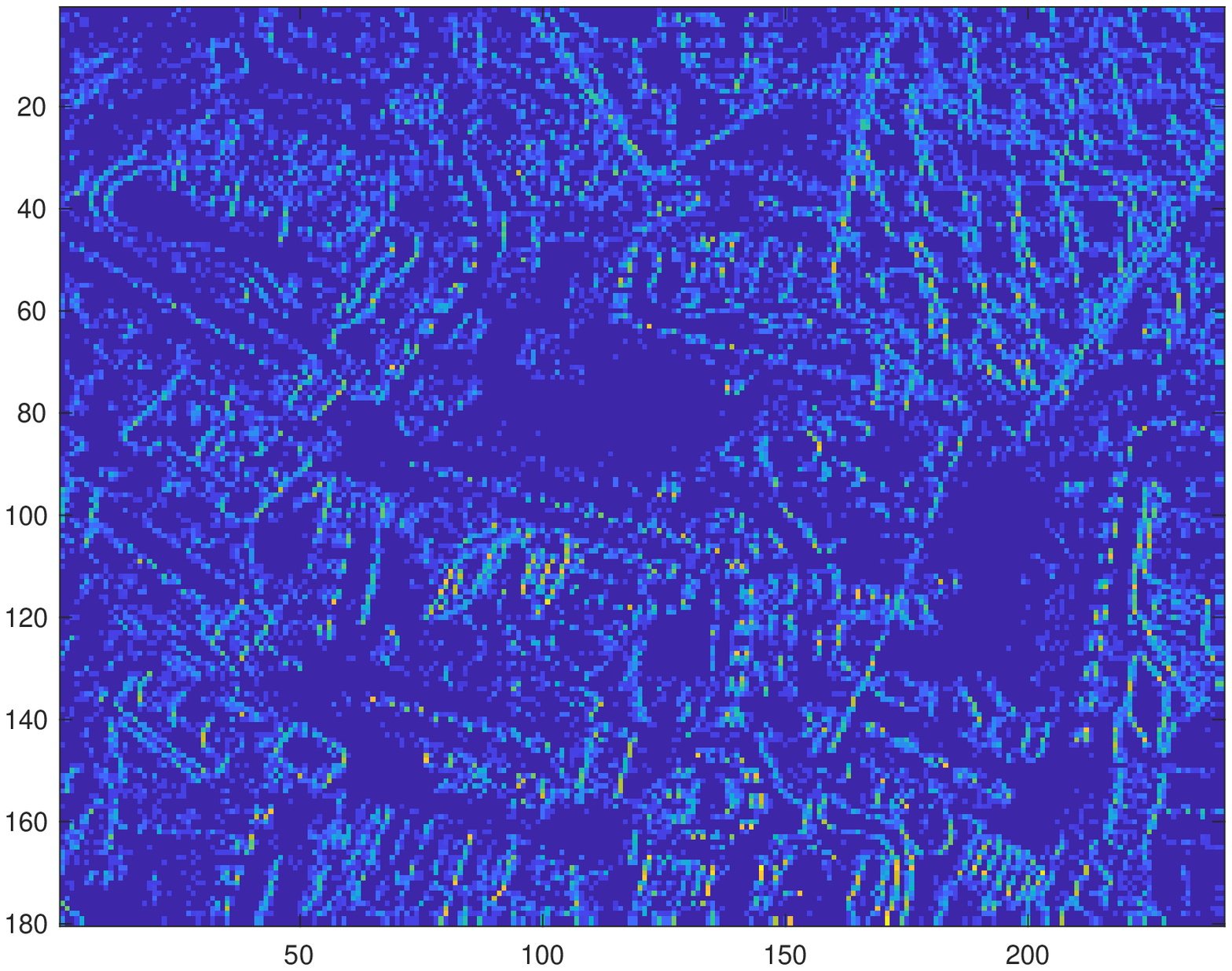}
			& \includegraphics[width=\linewidth]{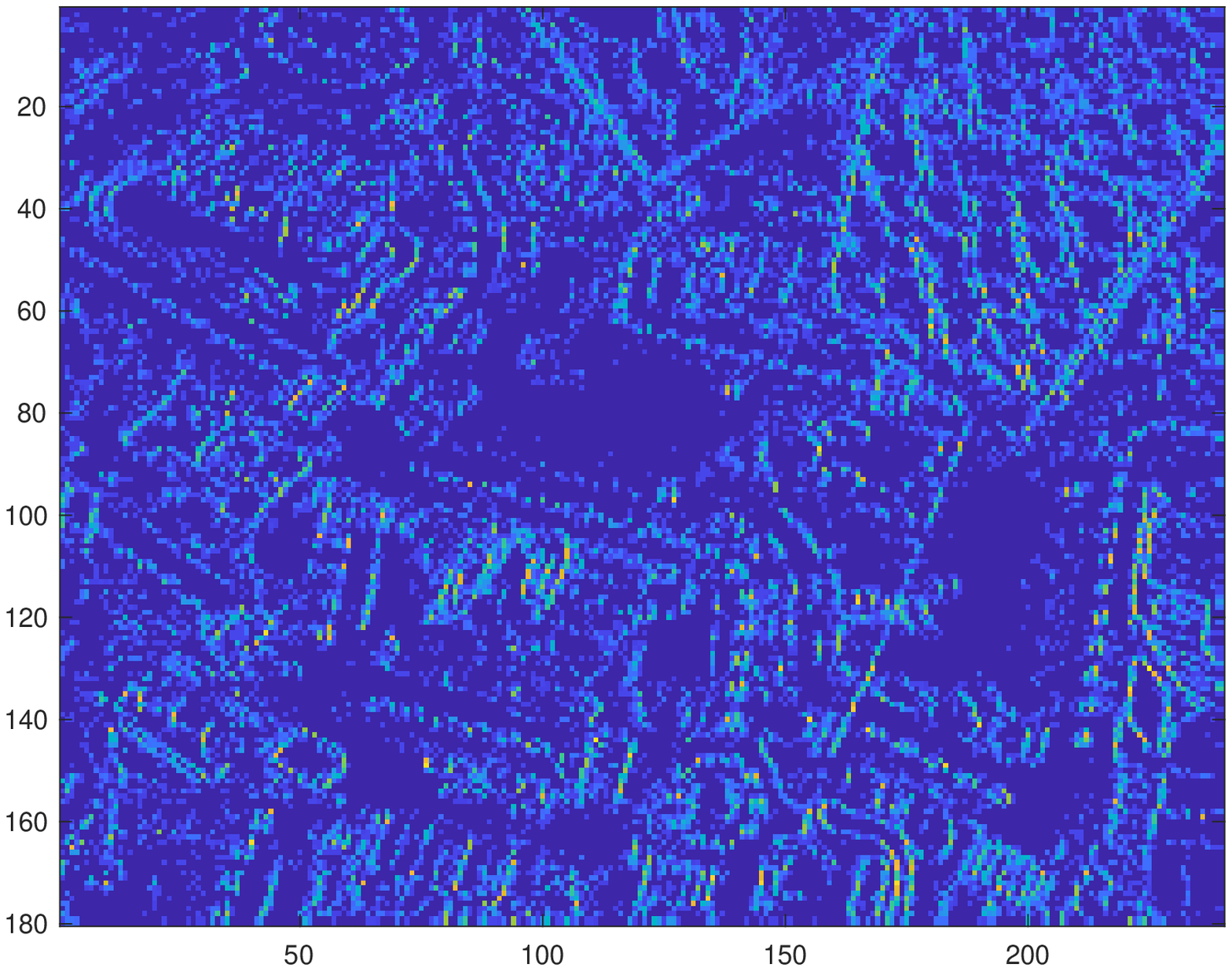}
			& \includegraphics[width=\linewidth]{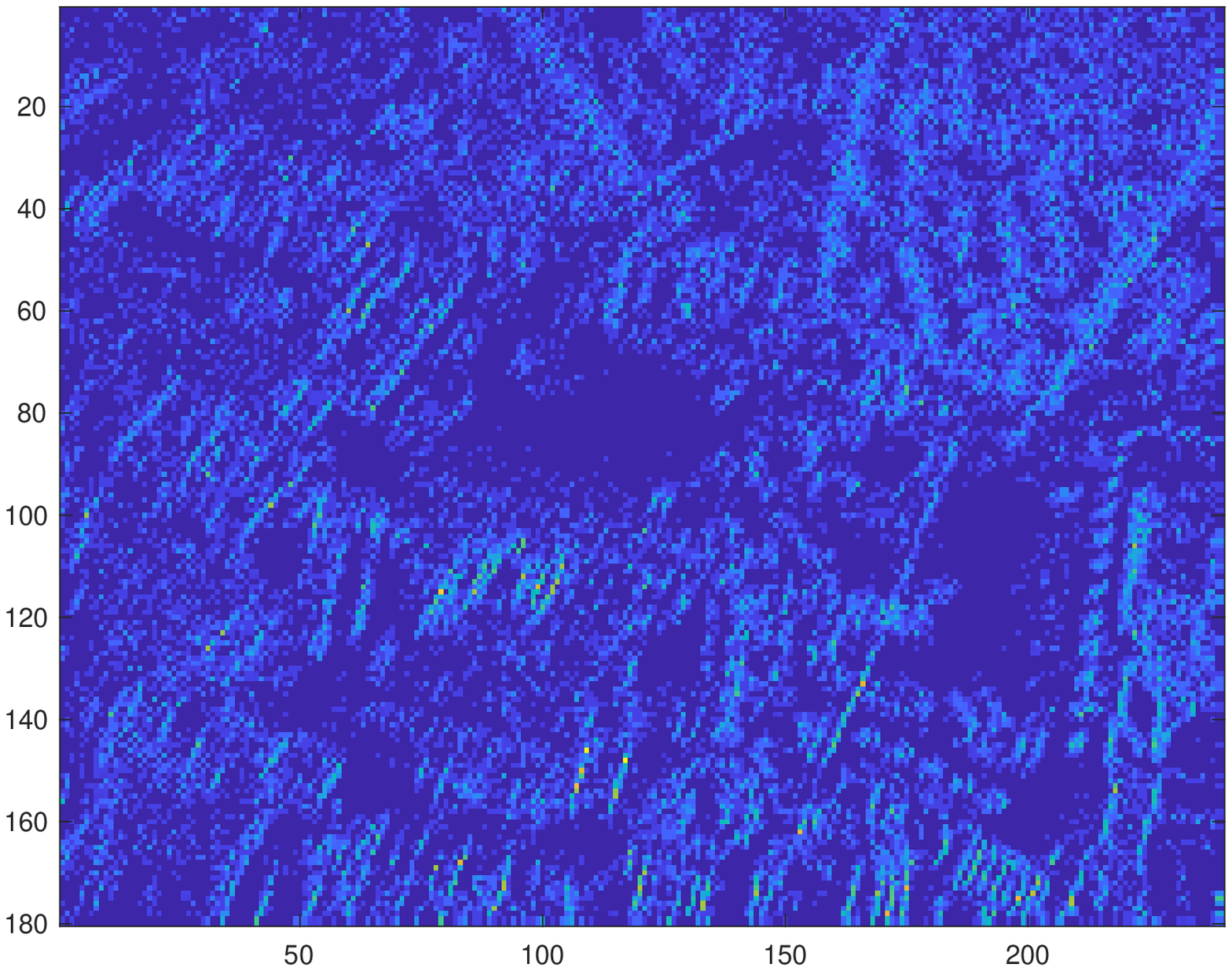}
		\end{tabularx}
		%\captionof{figure}{Qualitative results (motion compensated event images) for \emph{boxes}.}
		\caption{Qualitative results (motion compensated event images) for \emph{boxes}.}
		\label{fig:boxes}
	\end{figure}
	%\end{minipage}
	
	%\clearpage

	\begin{figure}
		\renewcommand{\arraystretch}{0.7}
		\setlength{\tabcolsep}{1mm}
		\footnotesize
		\begin{tabularx}{\textwidth} { 
				>{\raggedright\arraybackslash}C{.05} 
				| >{\centering\arraybackslash}X 
				| >{\centering\arraybackslash}X 
				| >{\centering\arraybackslash}X 
				| >{\centering\arraybackslash}X  }
			& CMBnB1
			& CMBnB2
			& CMGD1
			& CMGD2\\
			%\hline
			\RotText{~~ Subseq 1}
			& \includegraphics[width=\linewidth]{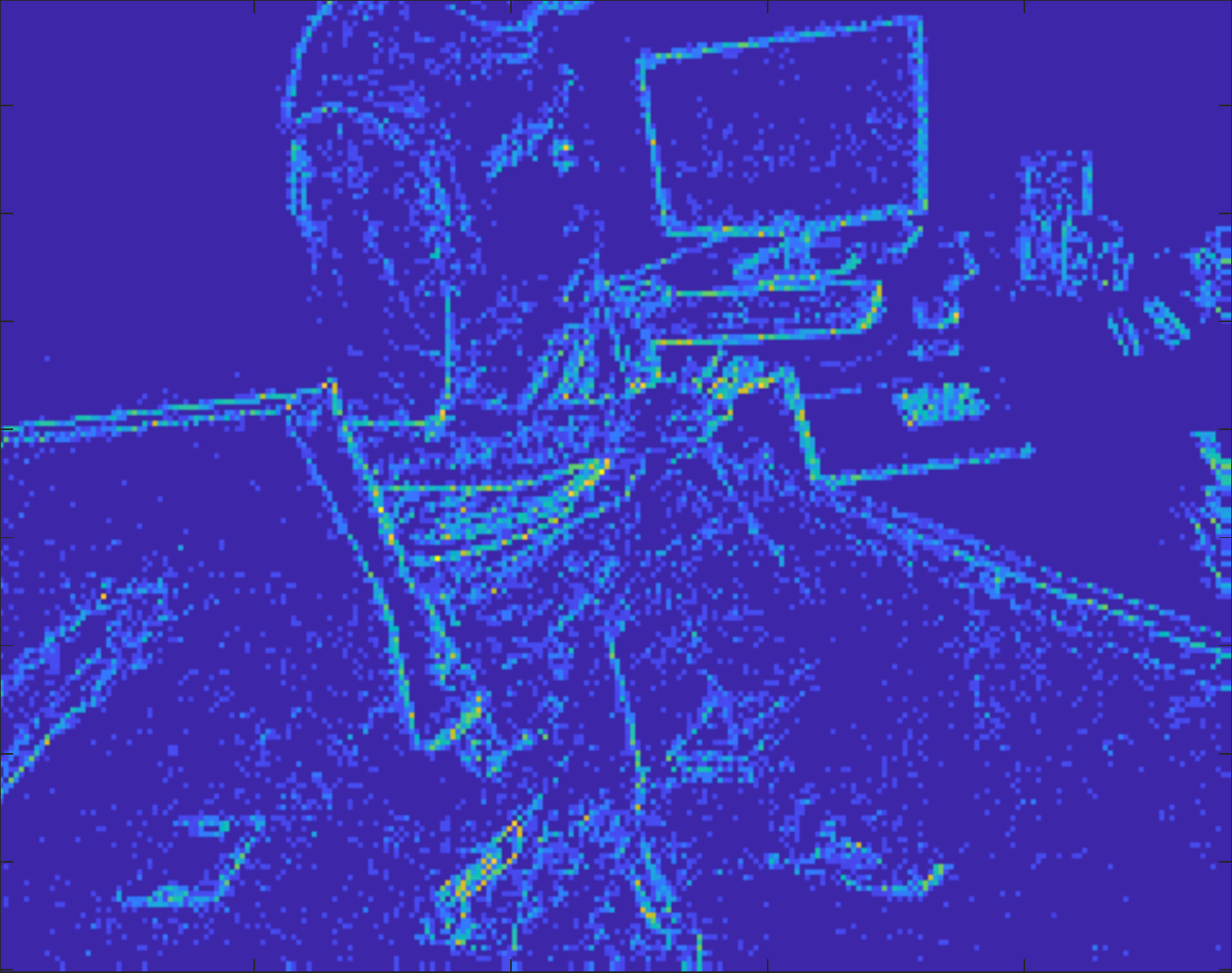}  
			& \includegraphics[width=\linewidth]{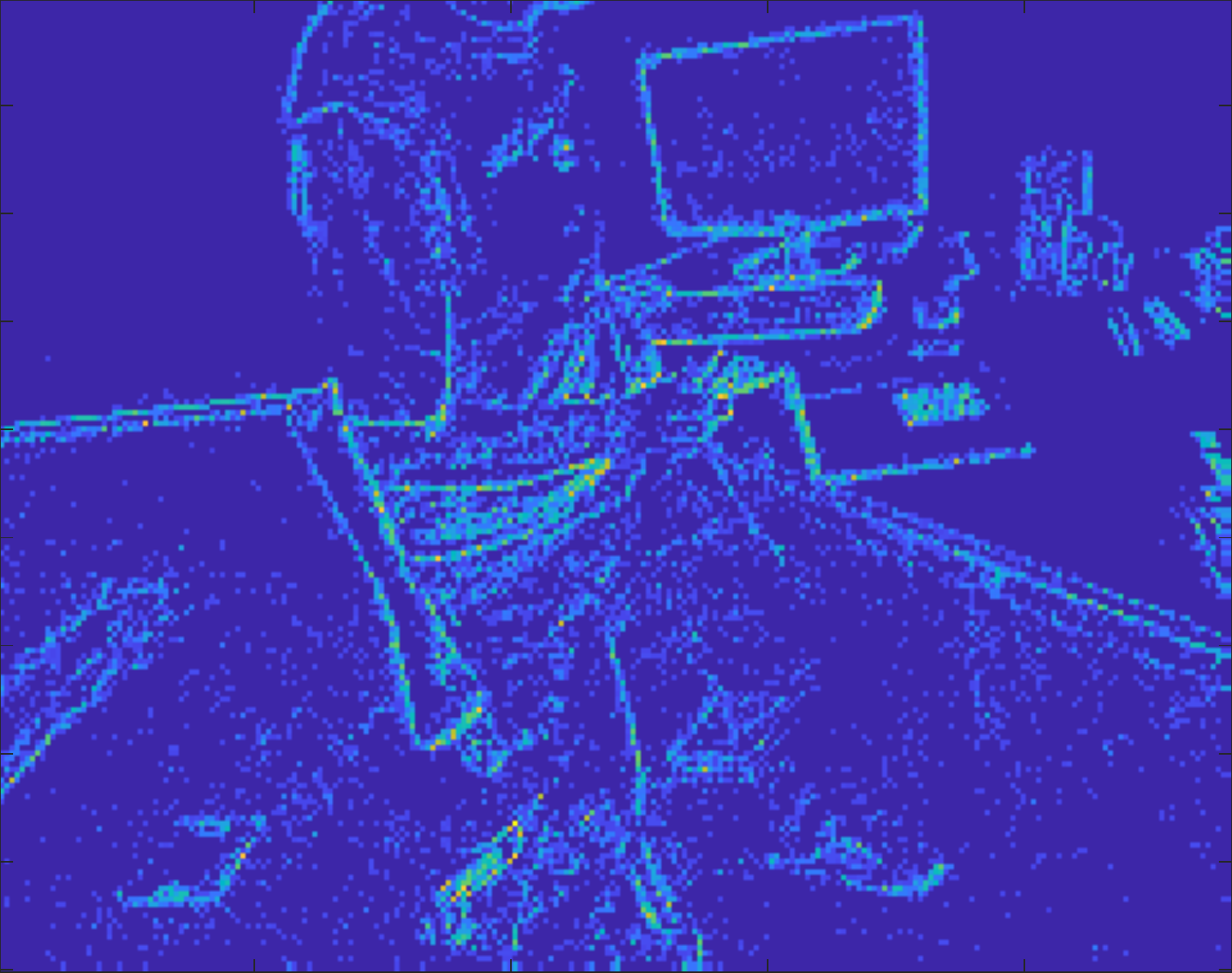}
			& \includegraphics[width=\linewidth]{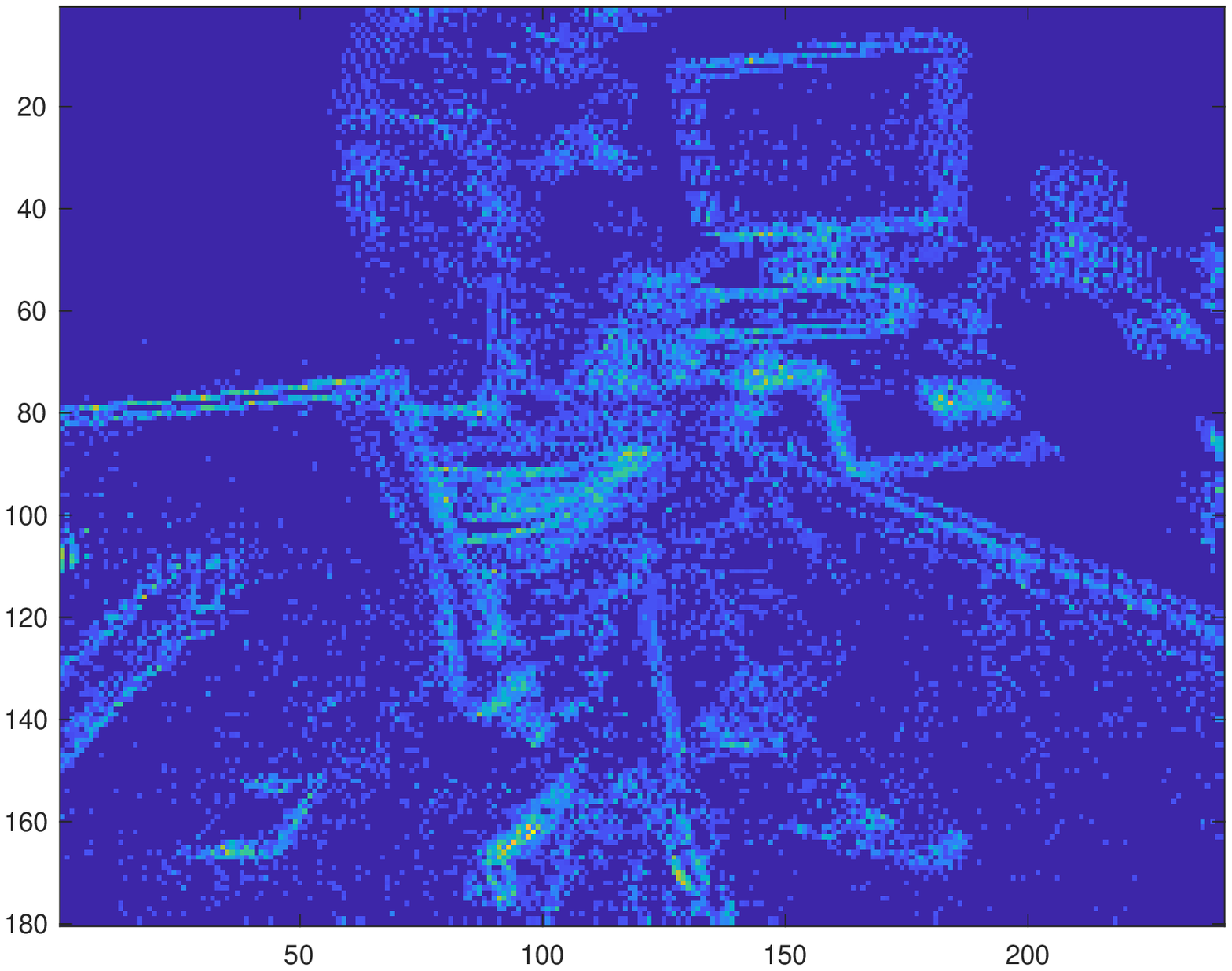}
			& \includegraphics[width=\linewidth]{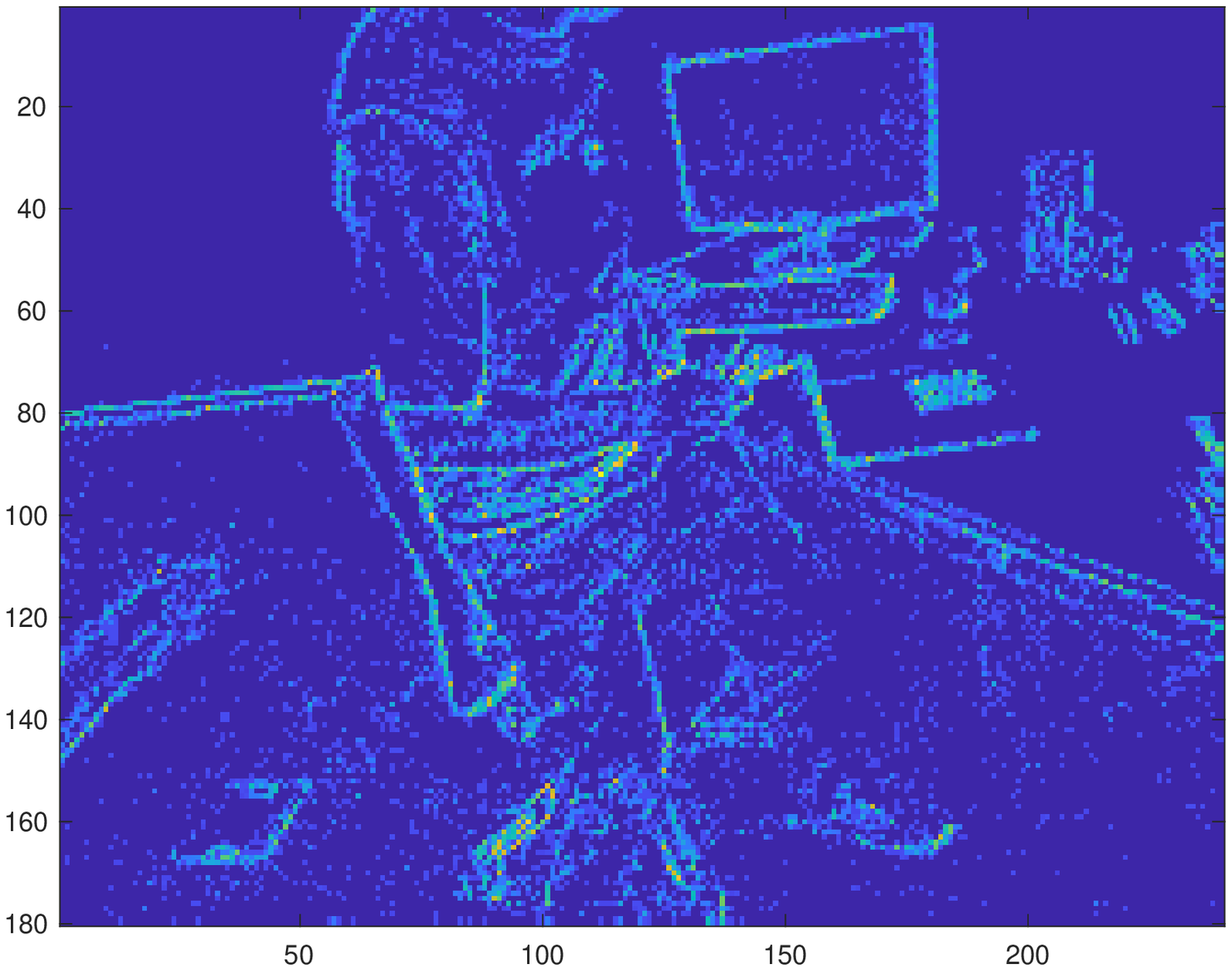}\\
			\hline
			&&&&\\[-.15cm]
			\RotText{~~ Subseq 2}
			& \includegraphics[width=\linewidth]{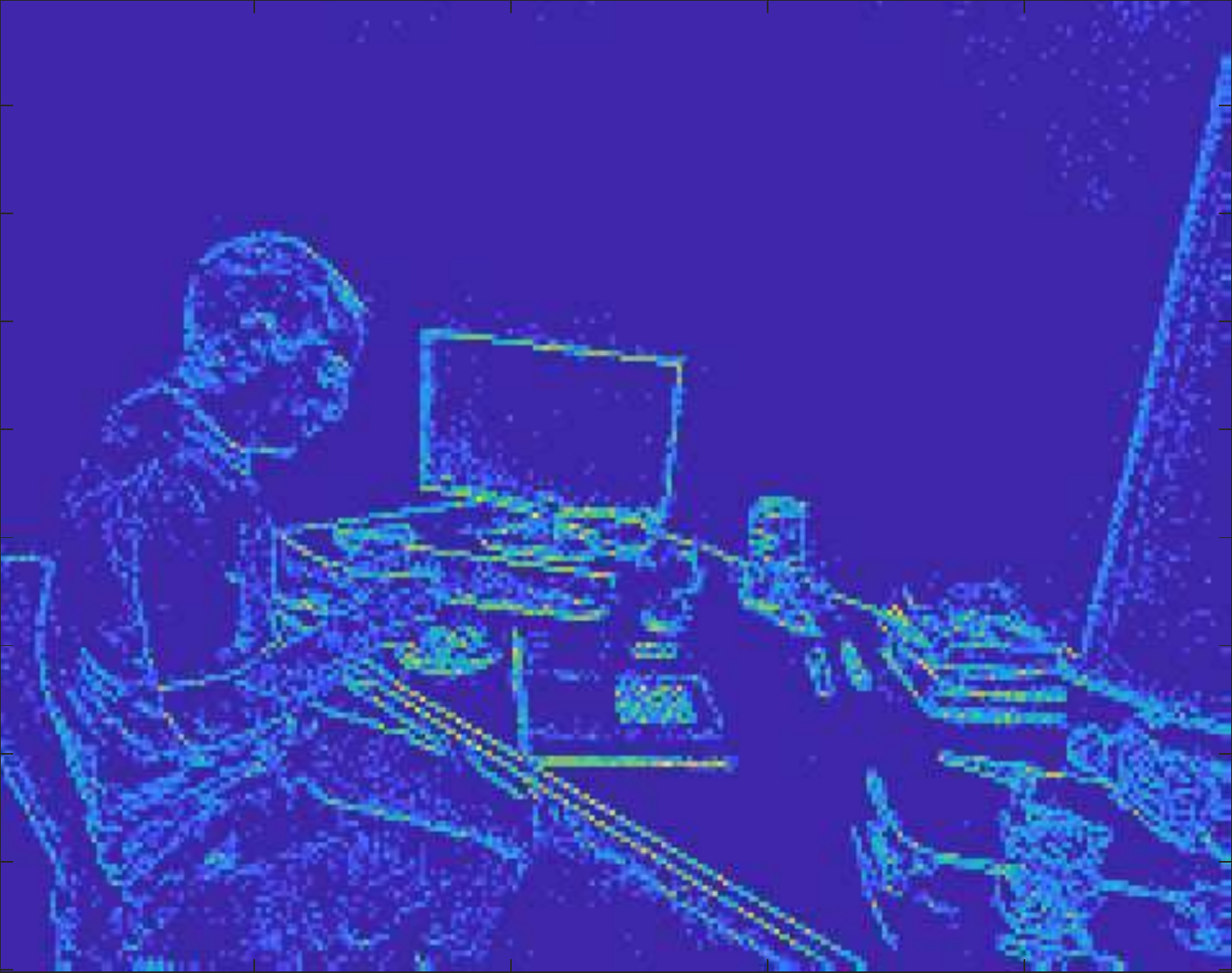}  
			& \includegraphics[width=\linewidth]{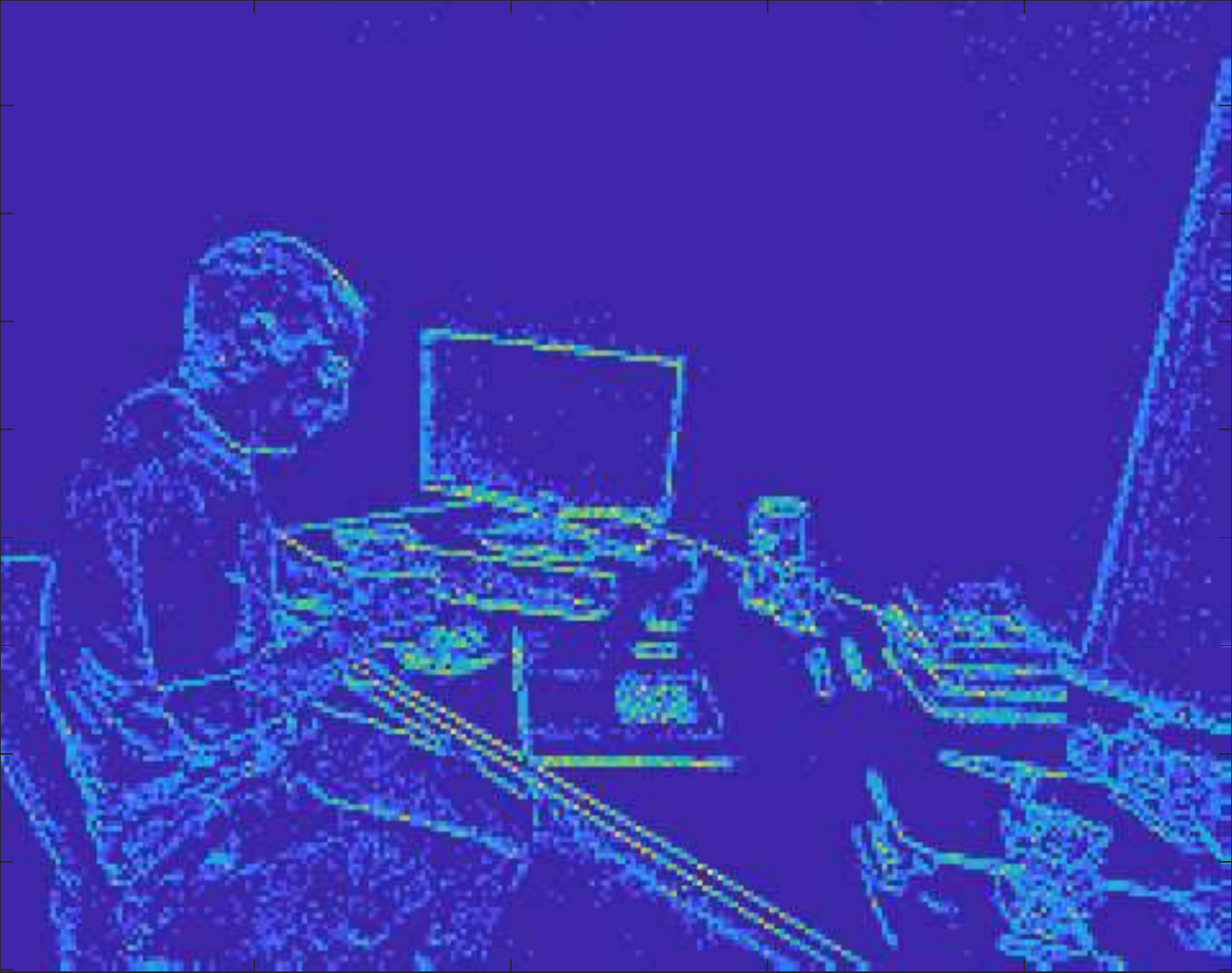}
			& \includegraphics[width=\linewidth]{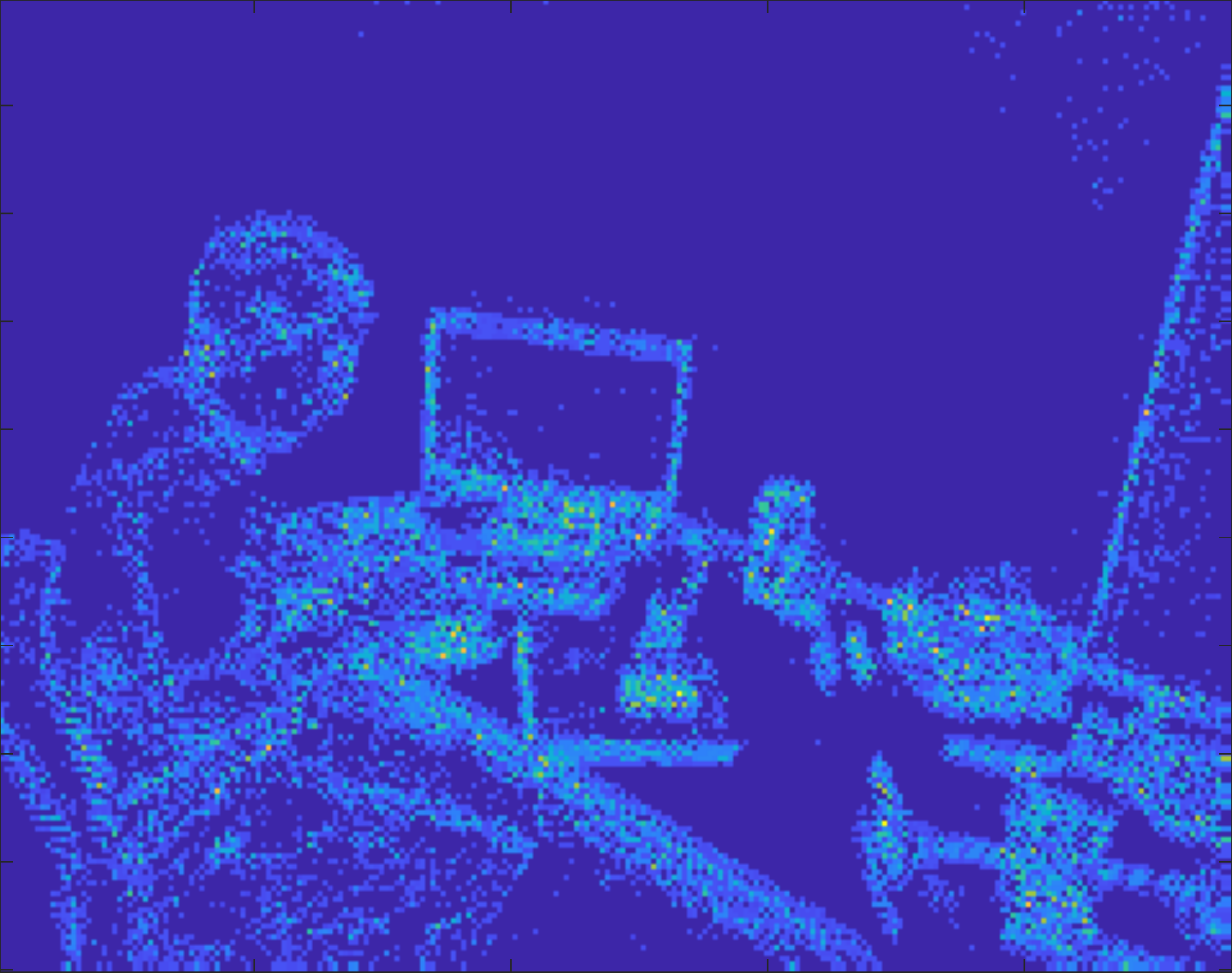}
			& \includegraphics[width=\linewidth]{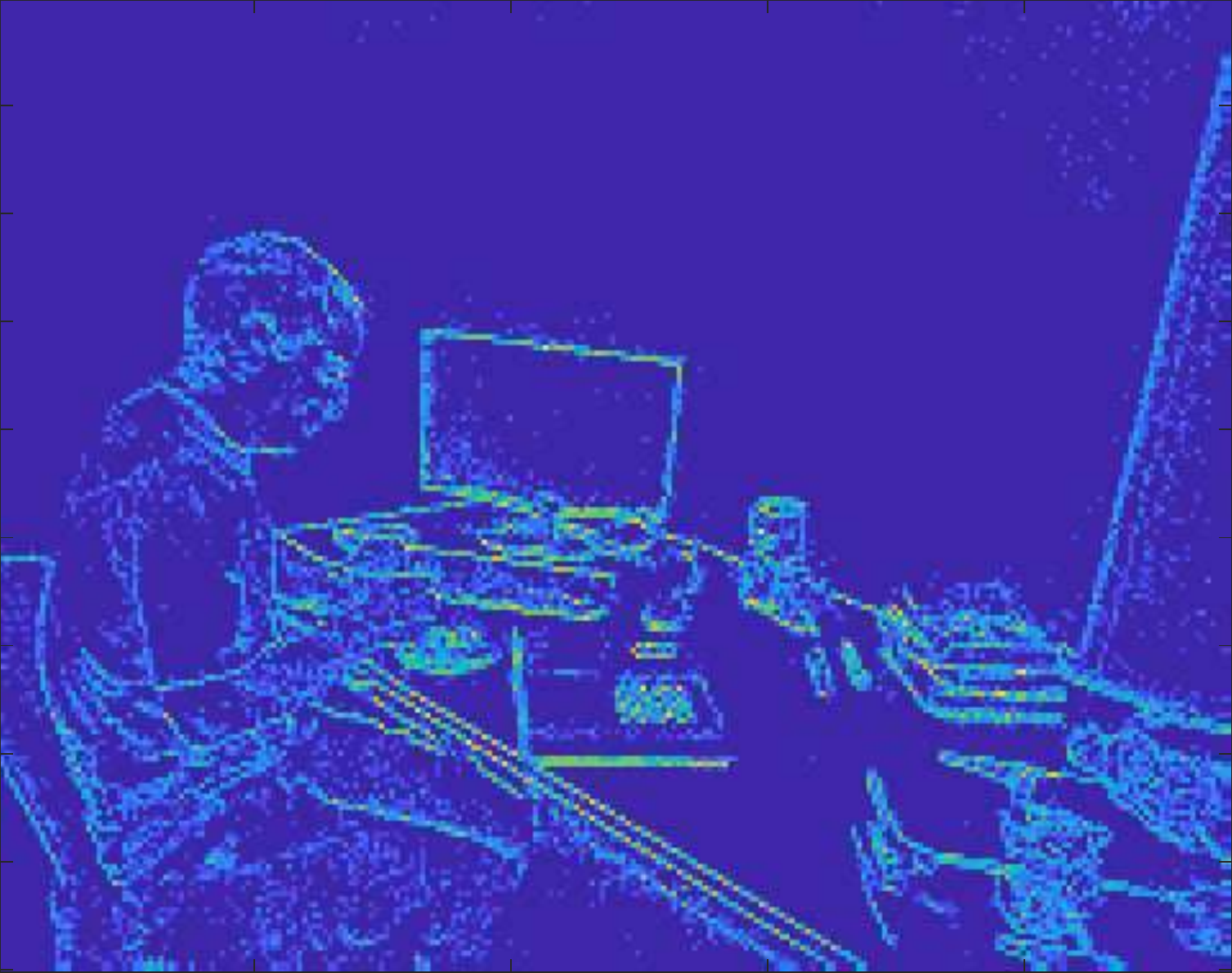}\\
			\hline
			&&&&\\[-.15cm]
			\RotText{~~ Subseq 3}
			& \includegraphics[width=\linewidth]{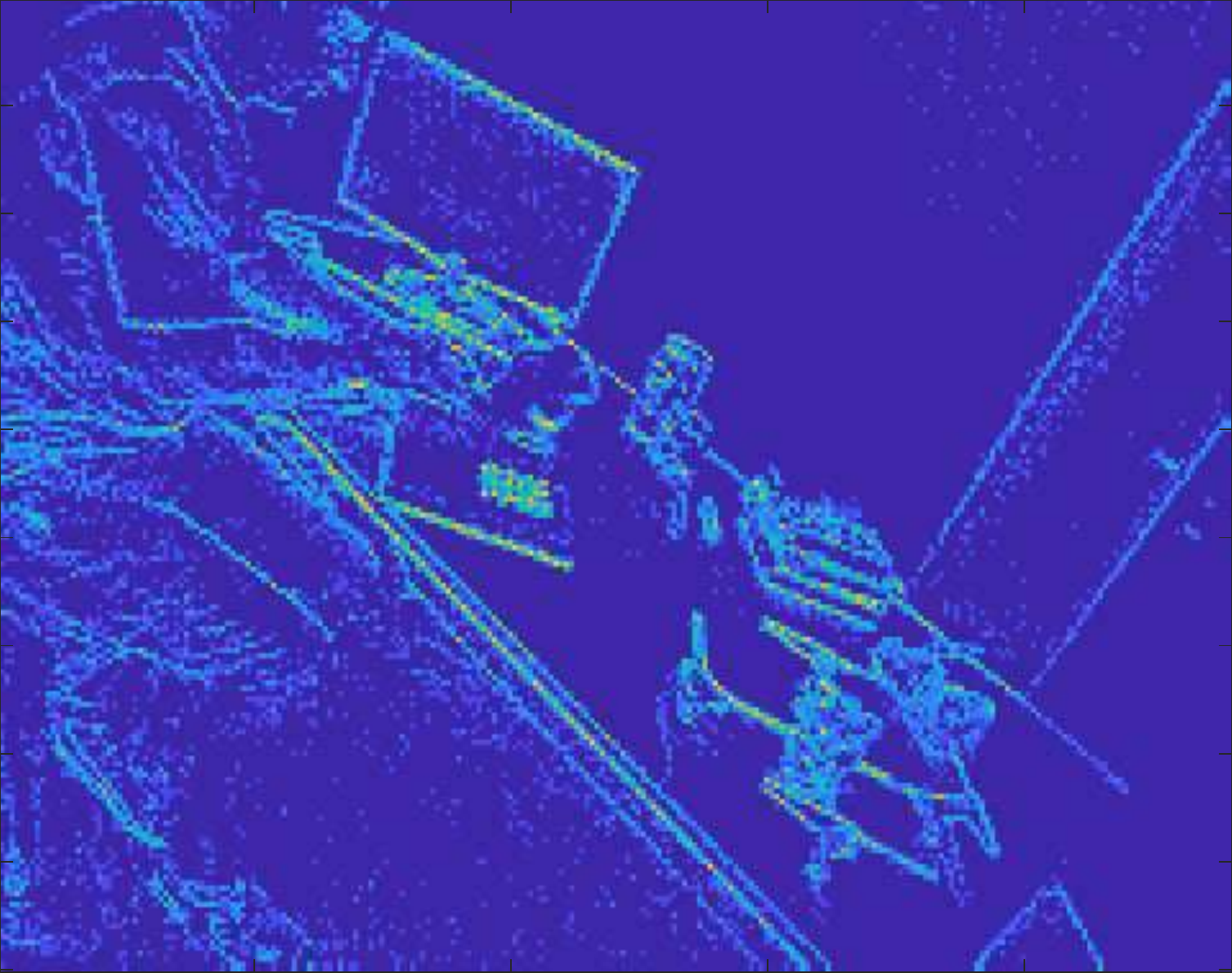}  
			& \includegraphics[width=\linewidth]{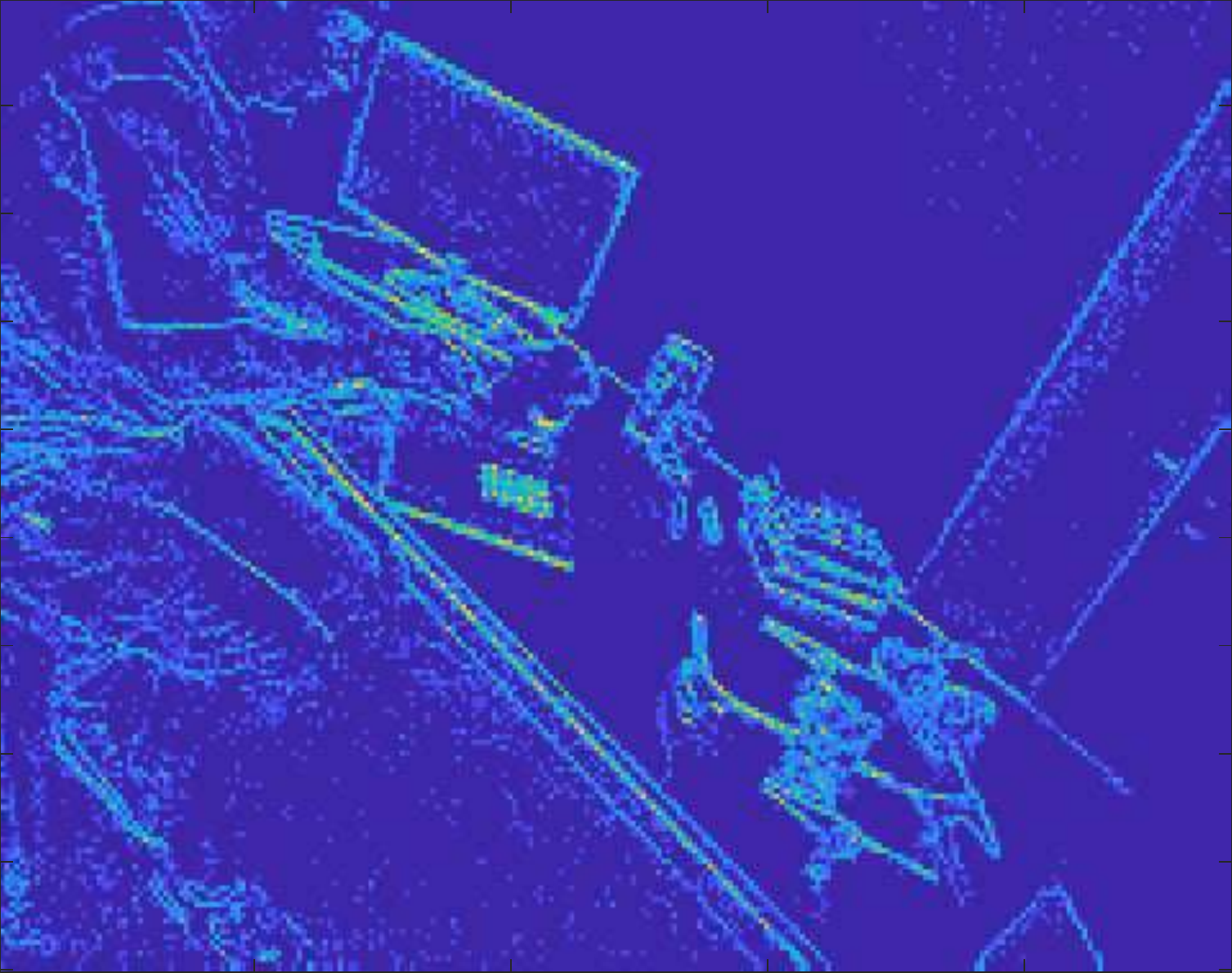}
			& \includegraphics[width=\linewidth]{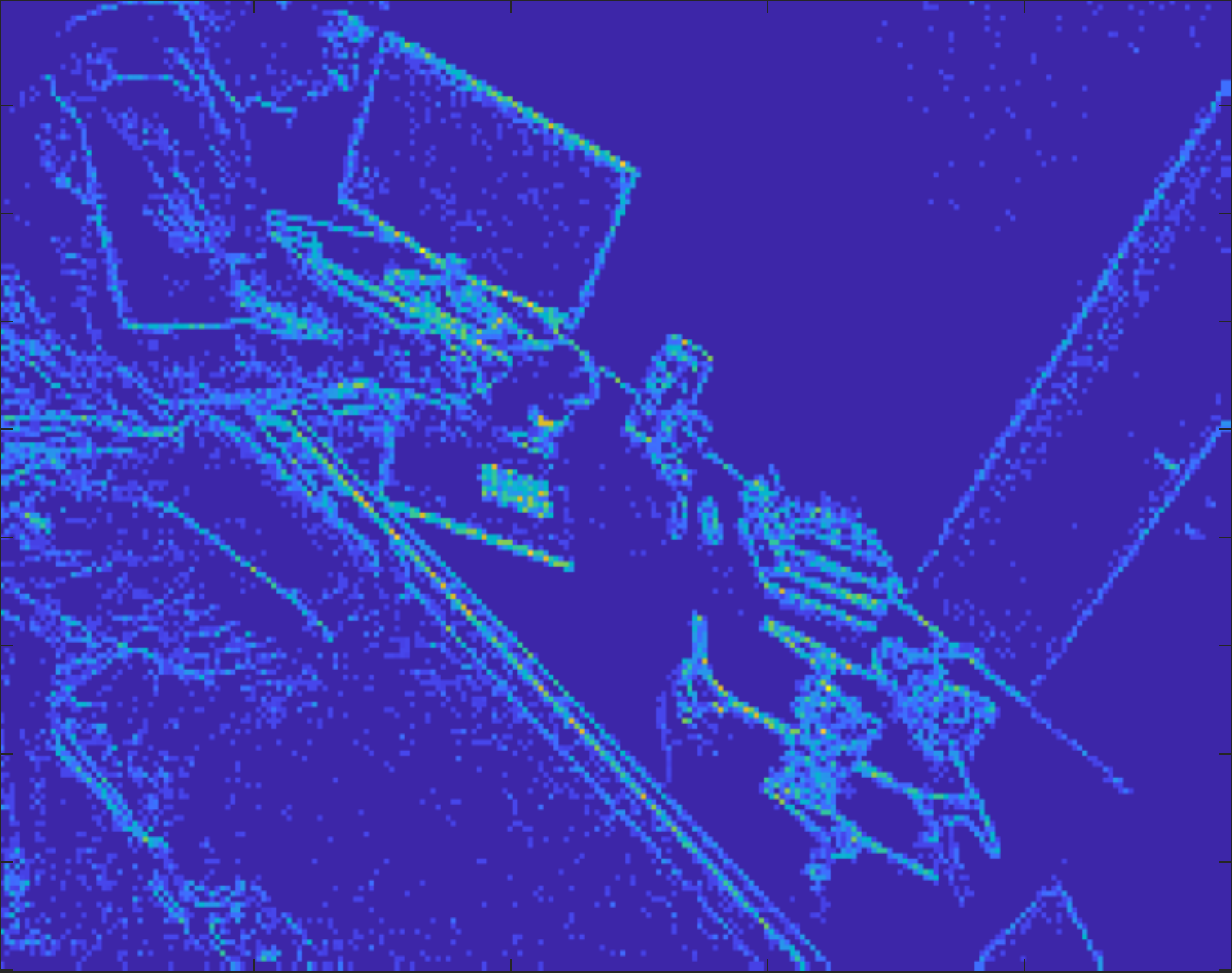}
			& \includegraphics[width=\linewidth]{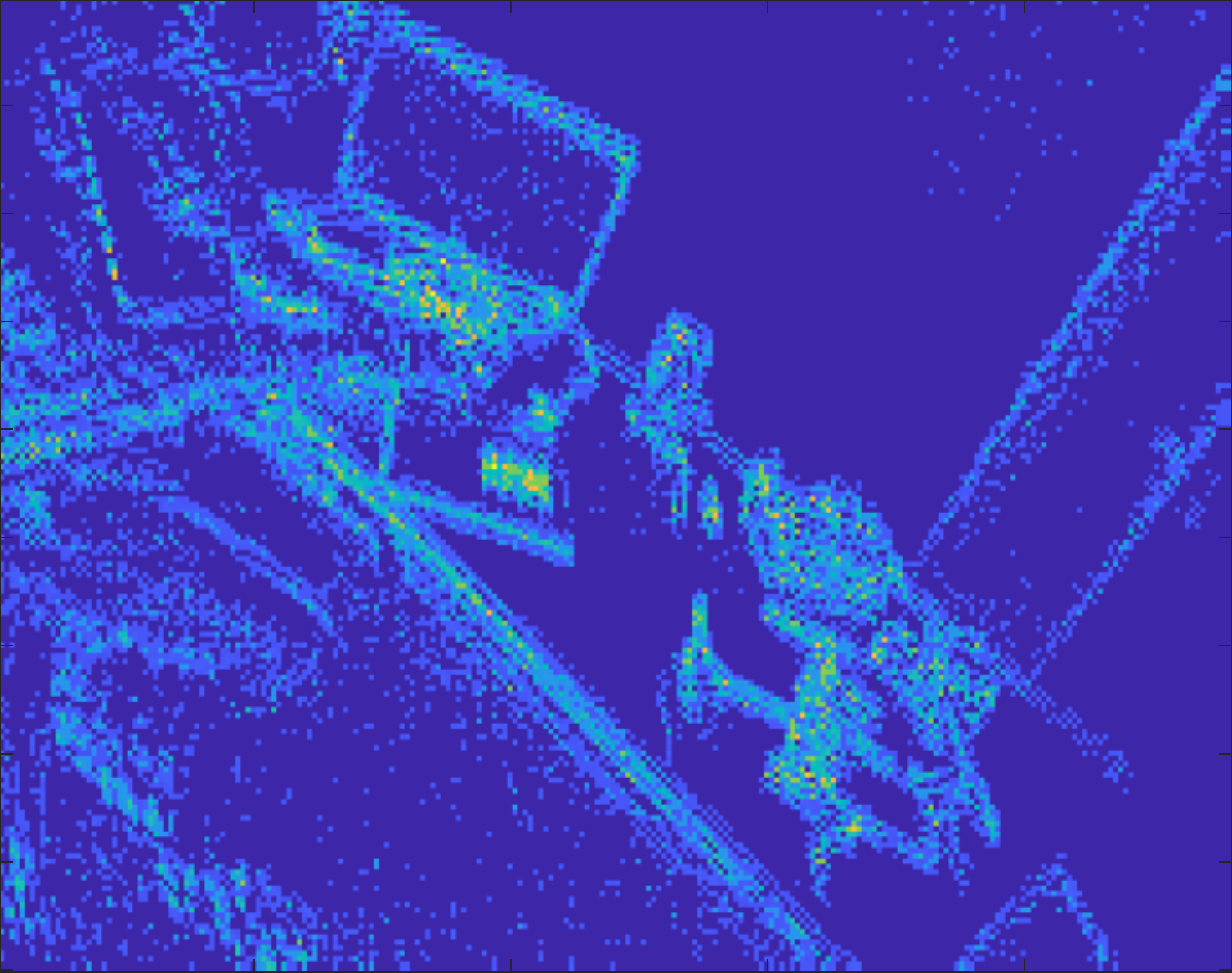}\\
			\hline
			&&&&\\[-.15cm]
			\RotText{~~ Subseq 4}
			& \includegraphics[width=\linewidth]{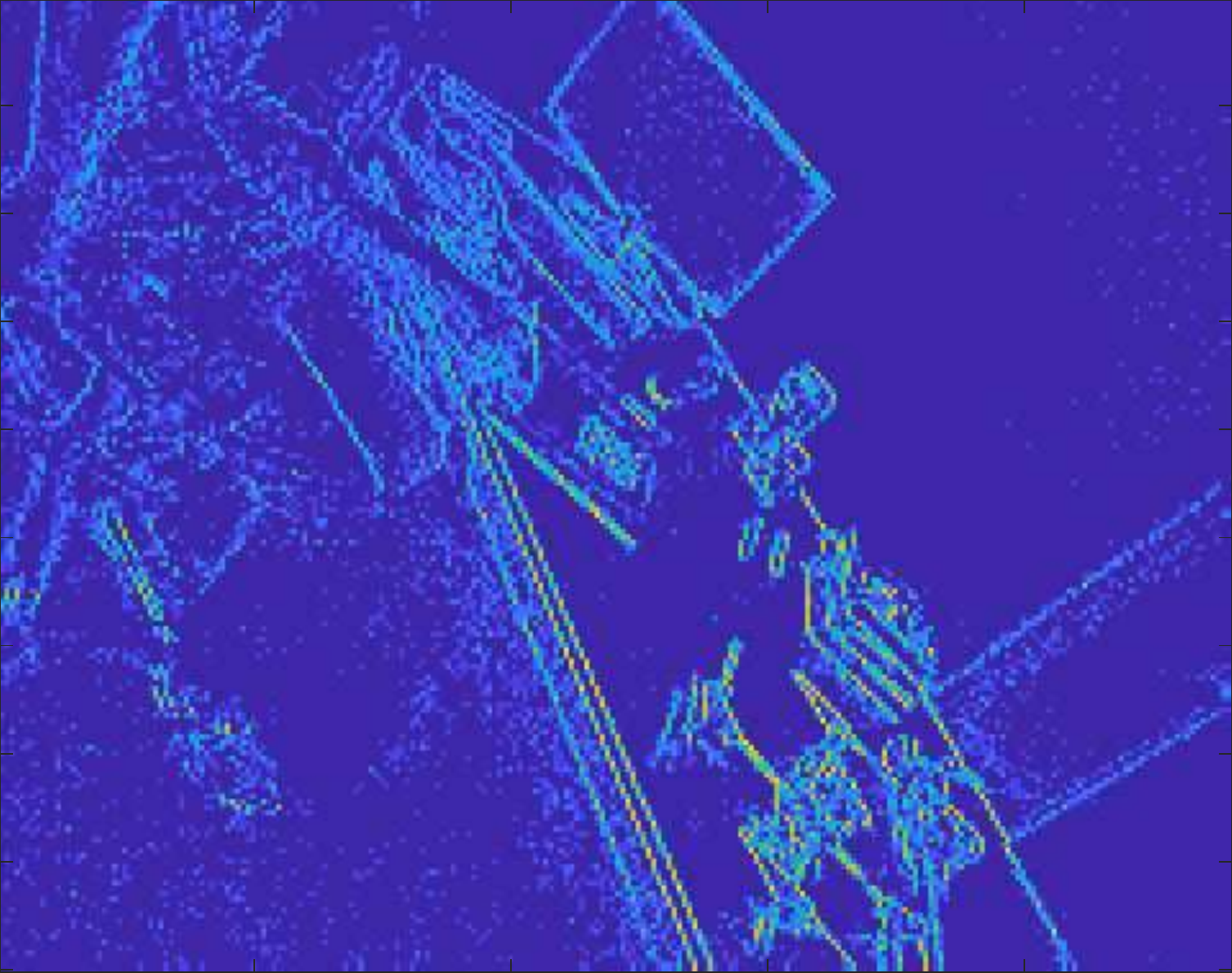}  
			& \includegraphics[width=\linewidth]{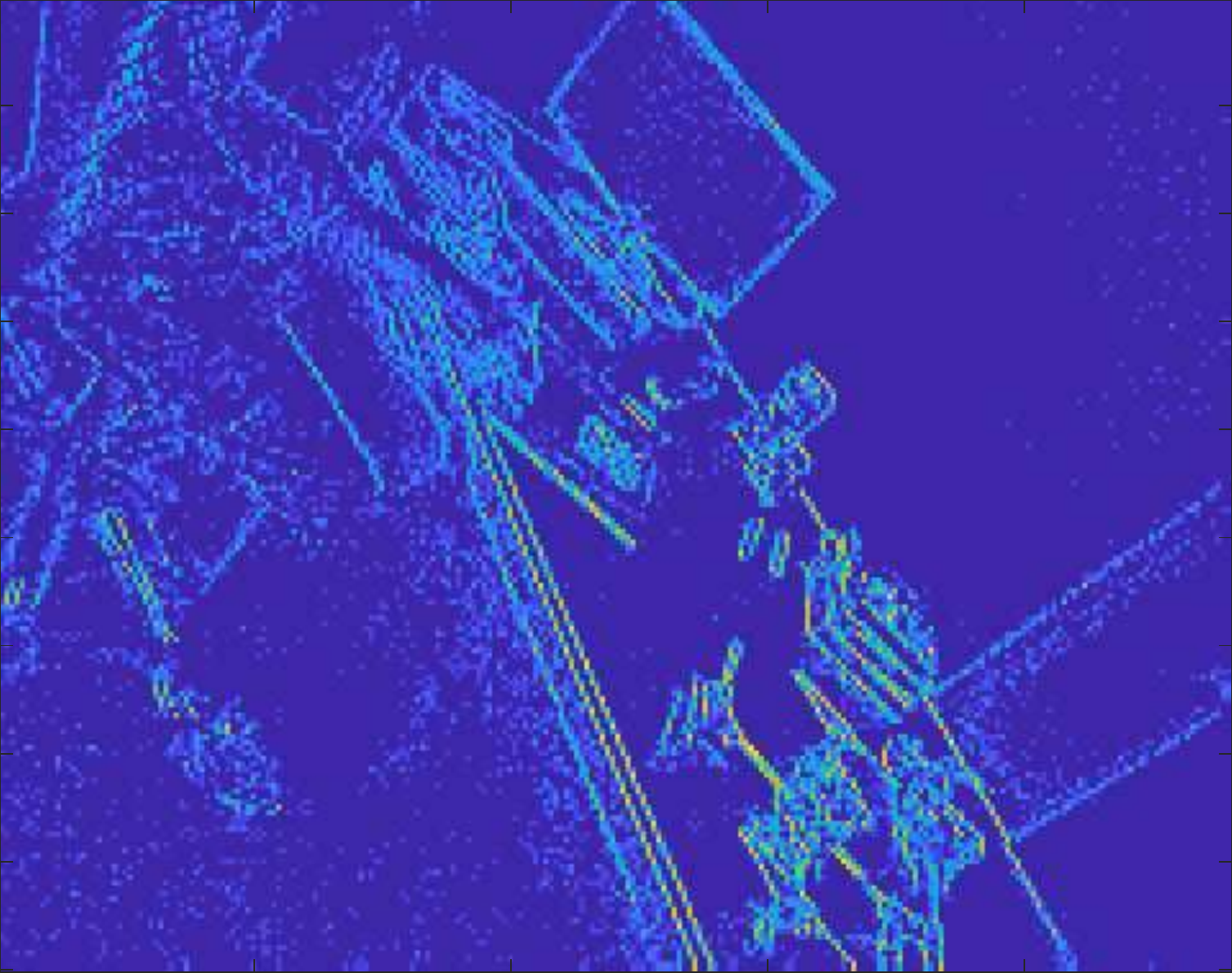}
			& \includegraphics[width=\linewidth]{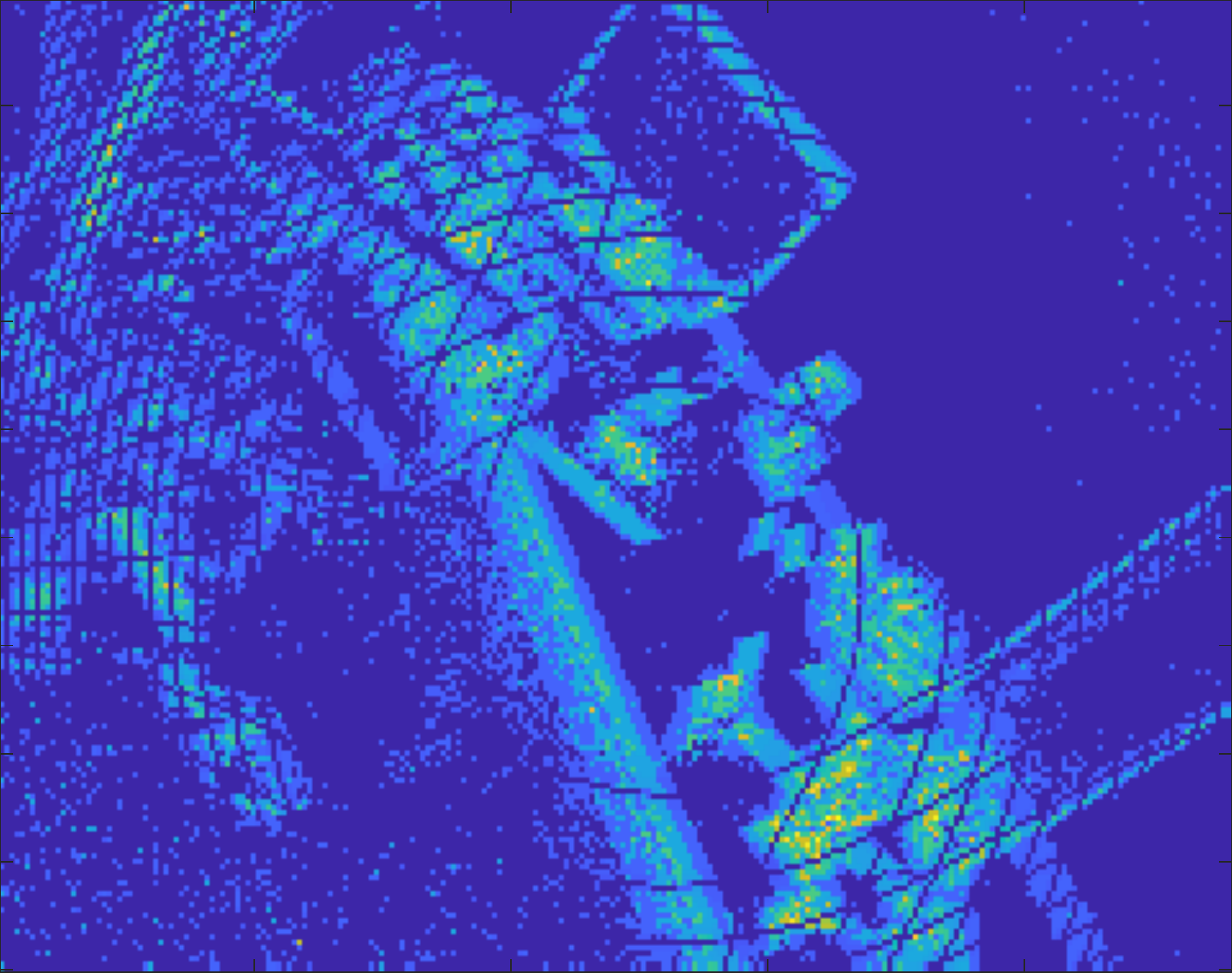}
			& \includegraphics[width=\linewidth]{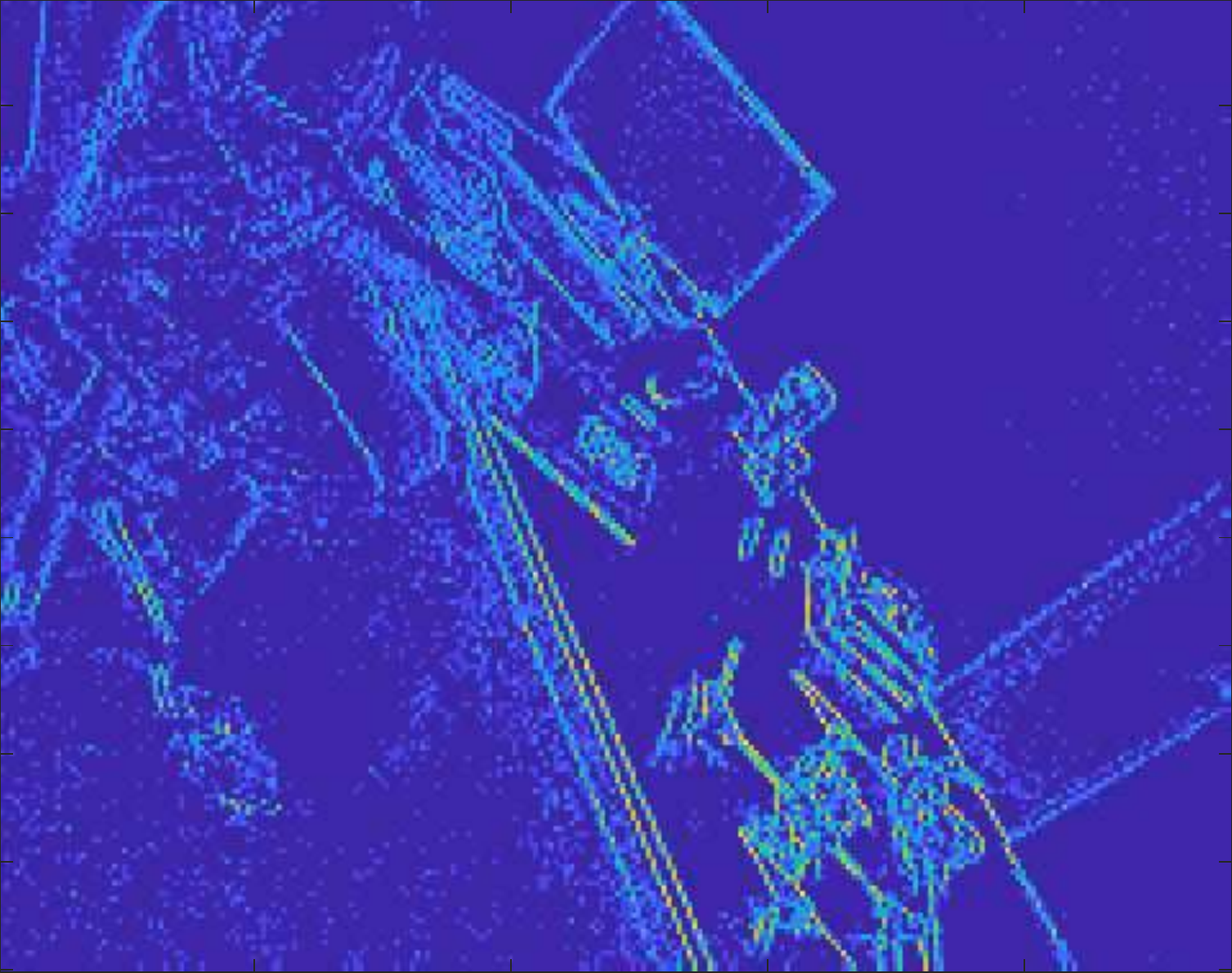}\\
			\hline
			&&&&\\[-.15cm]
			\RotText{~~ Subseq 5}
			& \includegraphics[width=\linewidth]{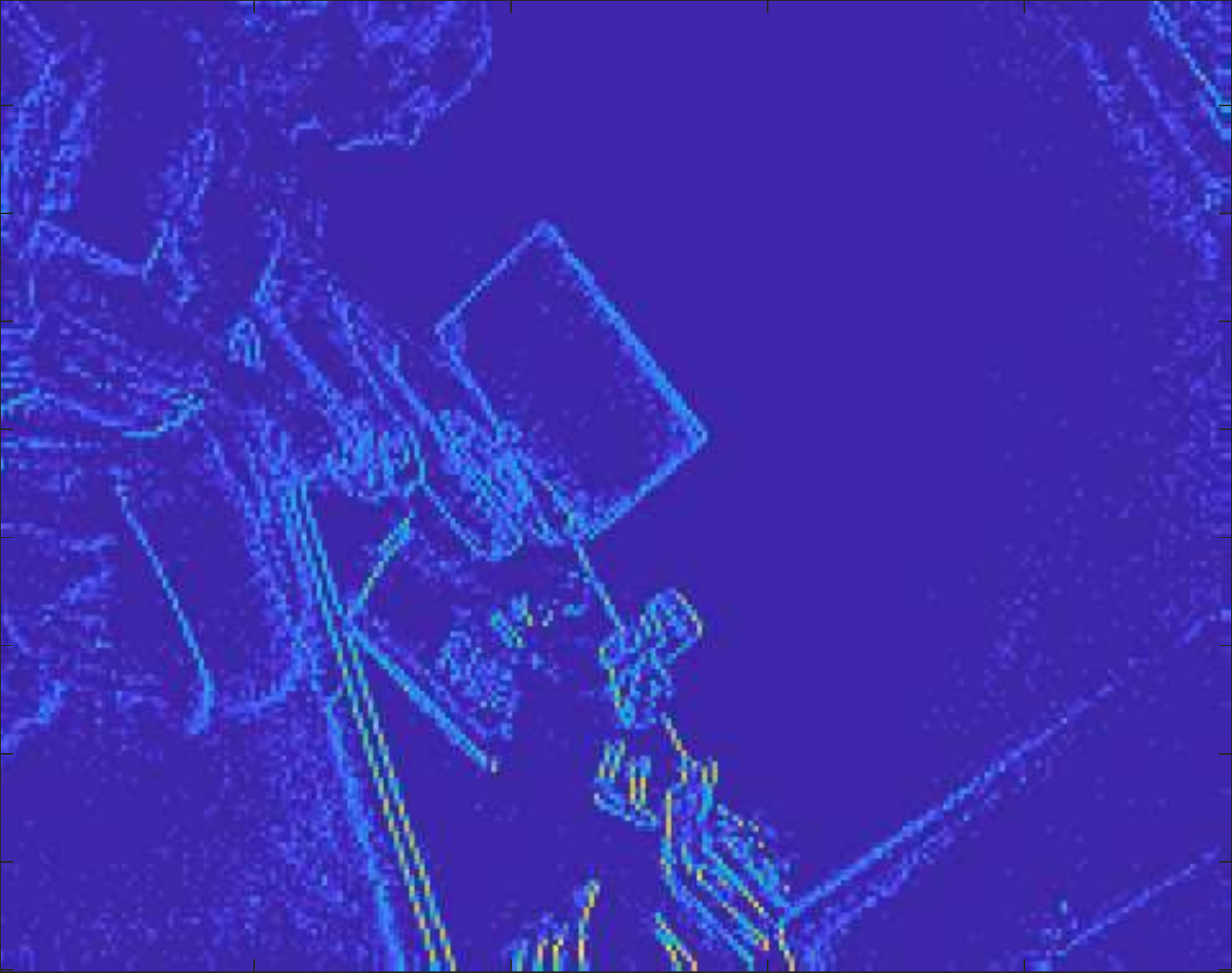}  
			& \includegraphics[width=\linewidth]{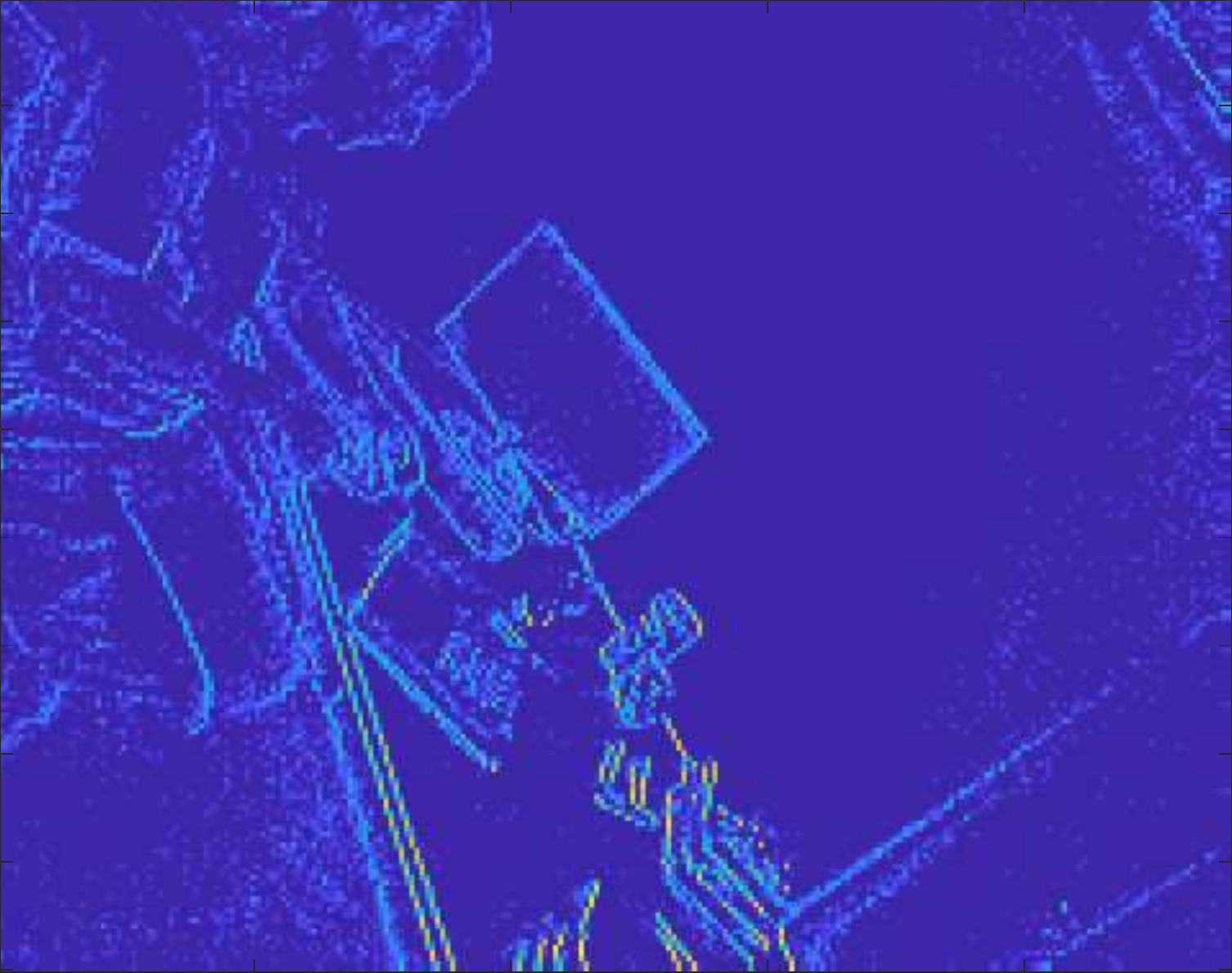}
			& \includegraphics[width=\linewidth]{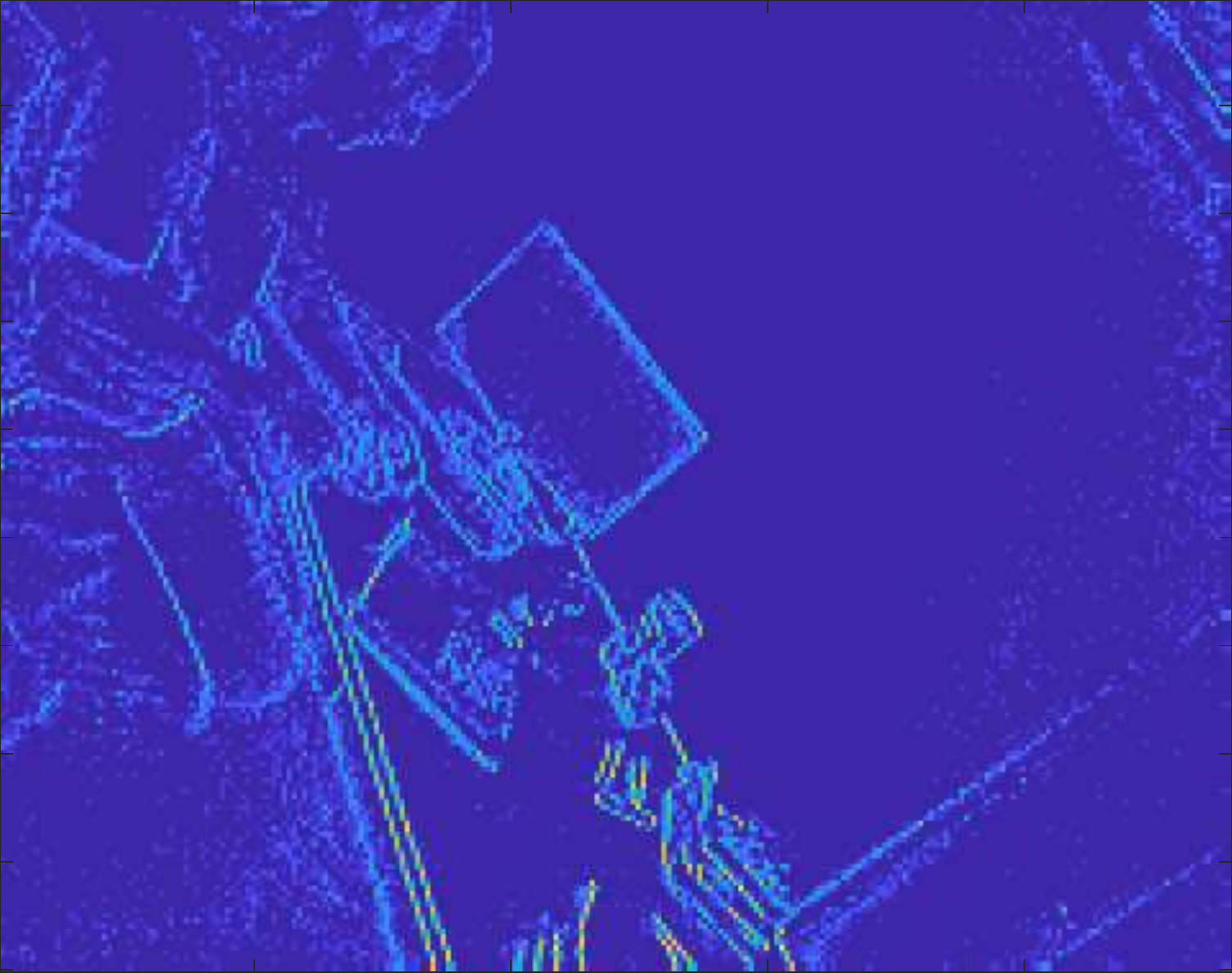}
			& \includegraphics[width=\linewidth]{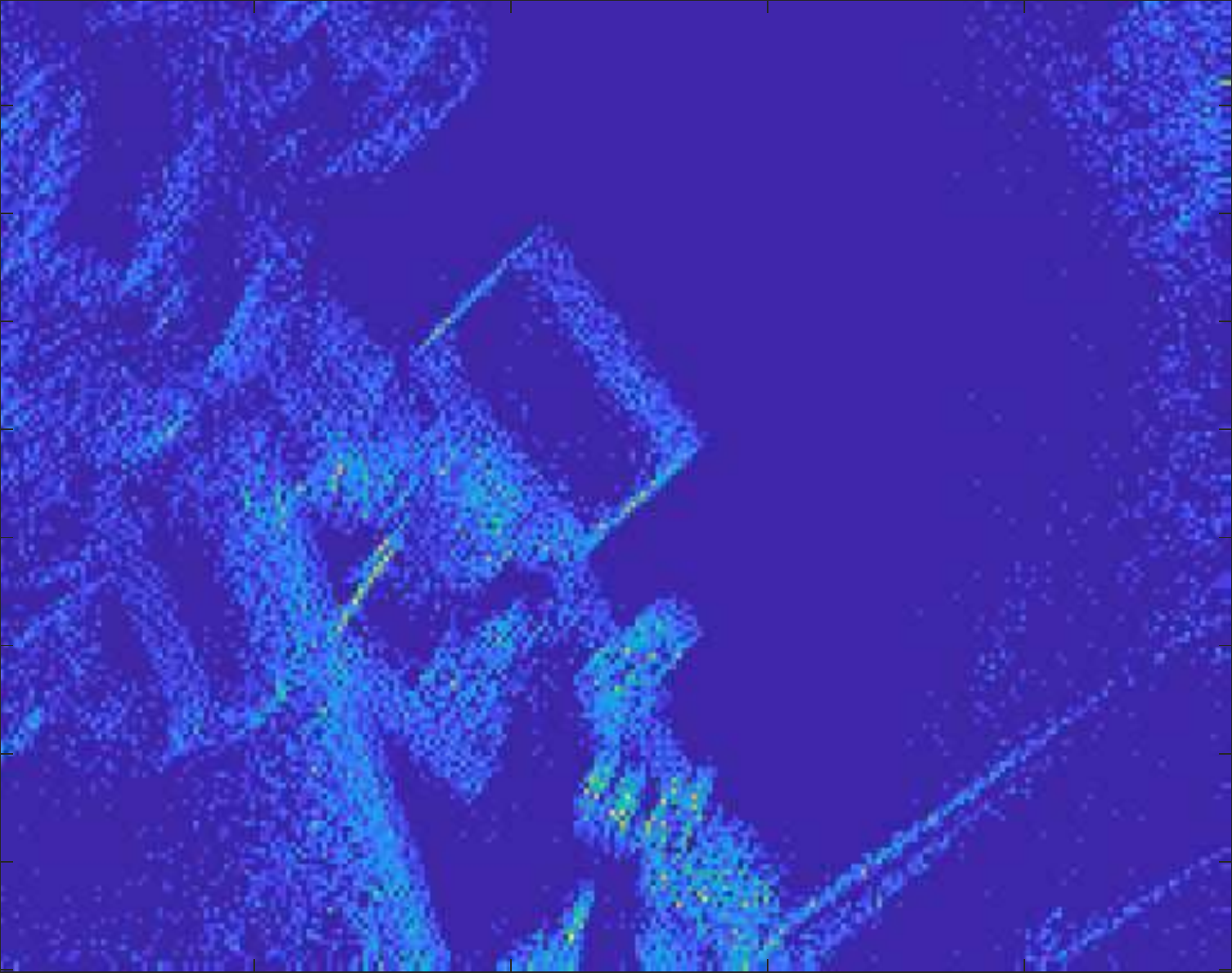}\\
		\end{tabularx}
		\caption{Qualitative results (motion compensated event images) for \emph{dynamic}.}
		\label{fig:dynamic}
	\end{figure}
	
	\begin{figure}
		\renewcommand{\arraystretch}{0.7}
		\setlength{\tabcolsep}{1mm}
		\footnotesize
		\begin{tabularx}{\textwidth} { 
				>{\raggedright\arraybackslash}C{.05} 
				| >{\centering\arraybackslash}X 
				| >{\centering\arraybackslash}X 
				| >{\centering\arraybackslash}X 
				| >{\centering\arraybackslash}X  }
			& CMBnB1
			& CMBnB2
			& CMGD1
			& CMGD2\\
			%\hline
			\RotText{~~ Subseq 1}
			& \includegraphics[width=\linewidth]{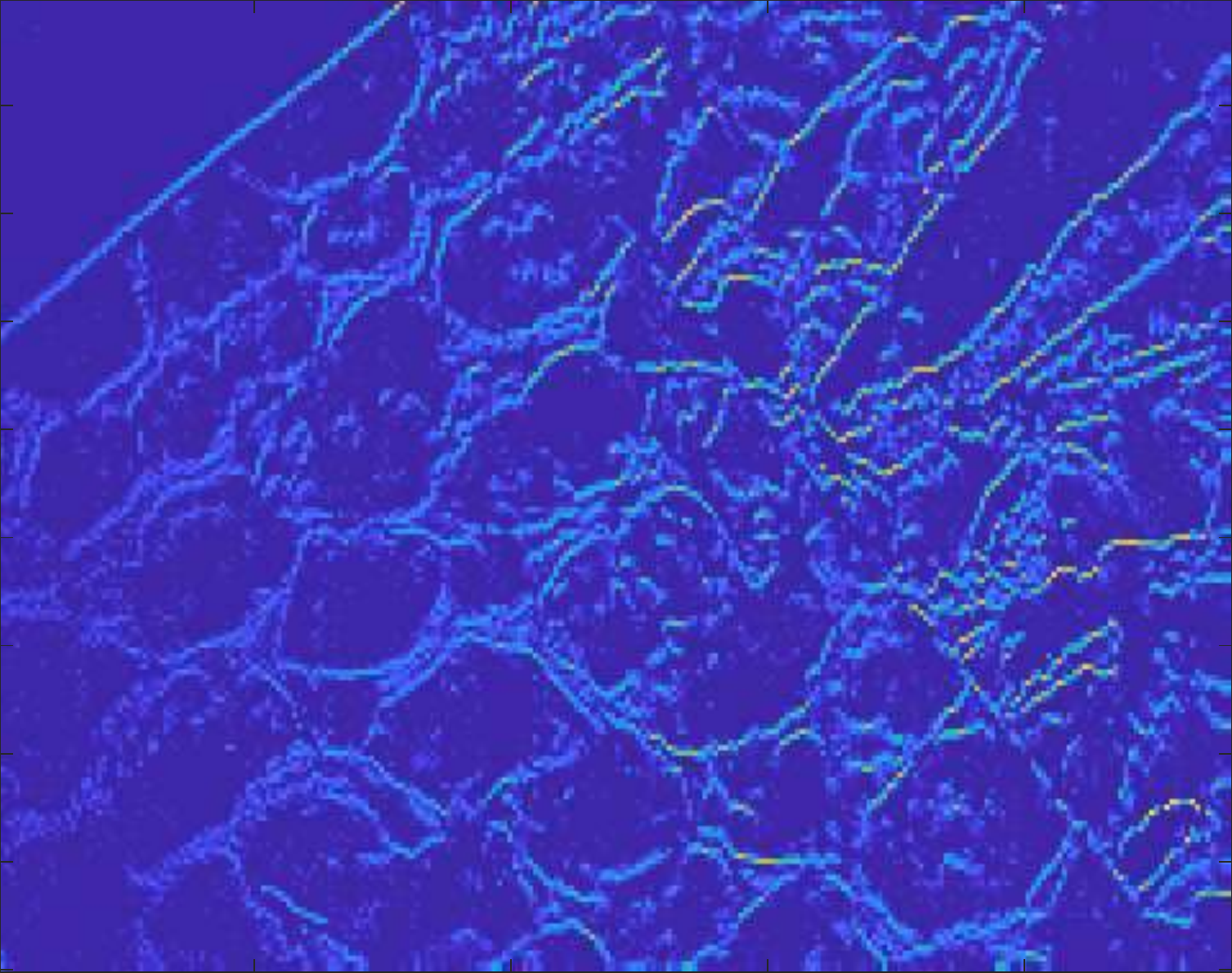}  
			& \includegraphics[width=\linewidth]{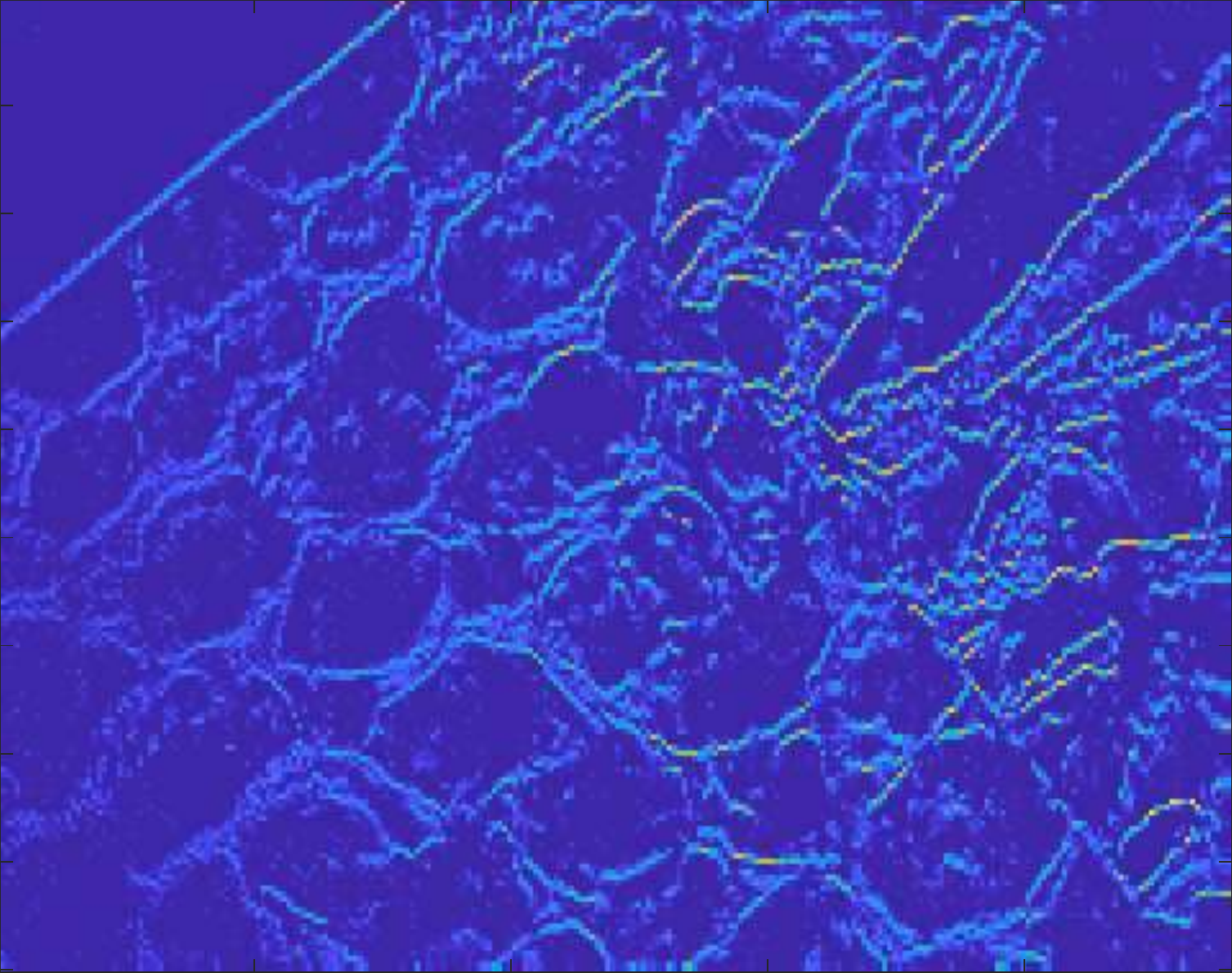}
			& \includegraphics[width=\linewidth]{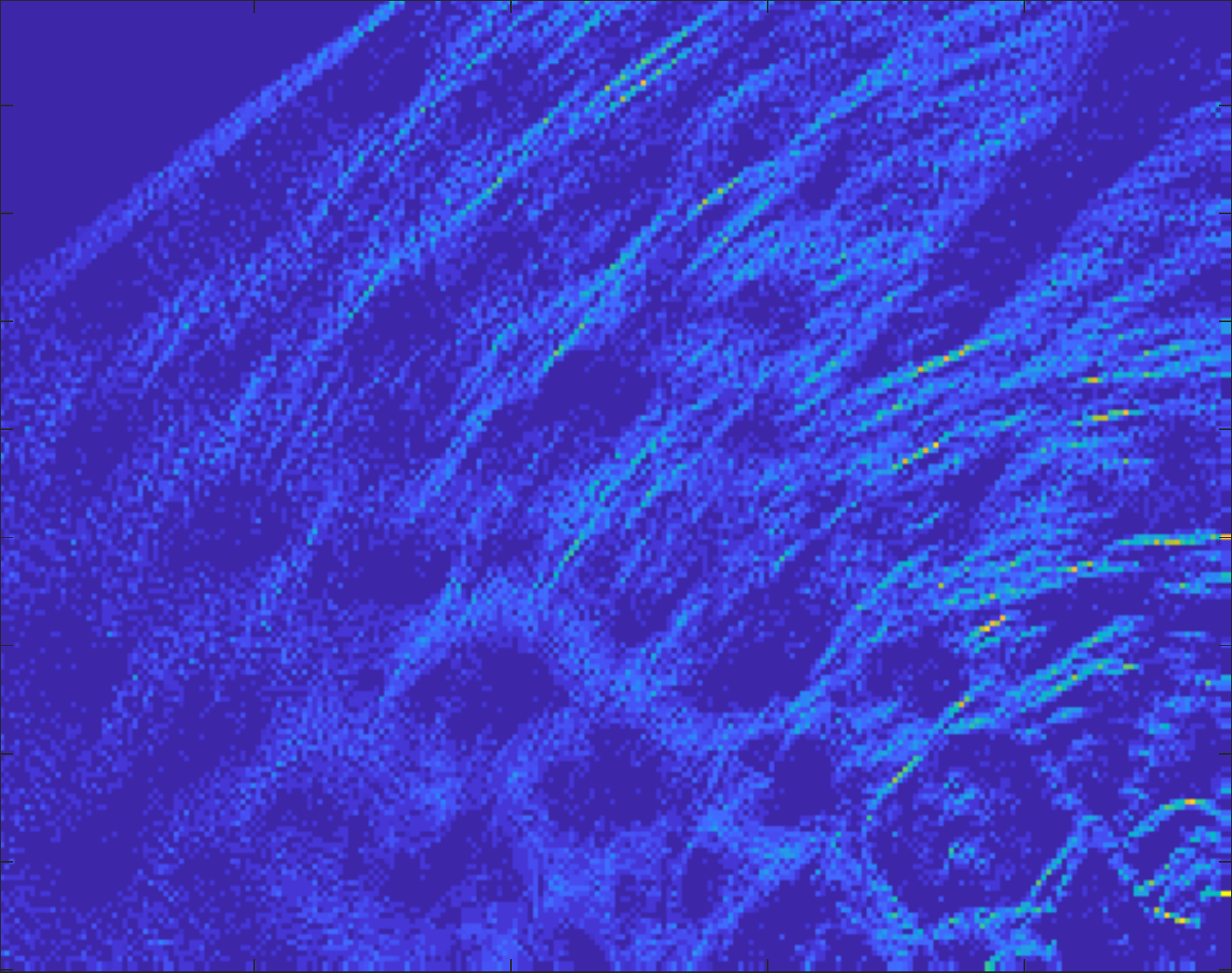}
			& \includegraphics[width=\linewidth]{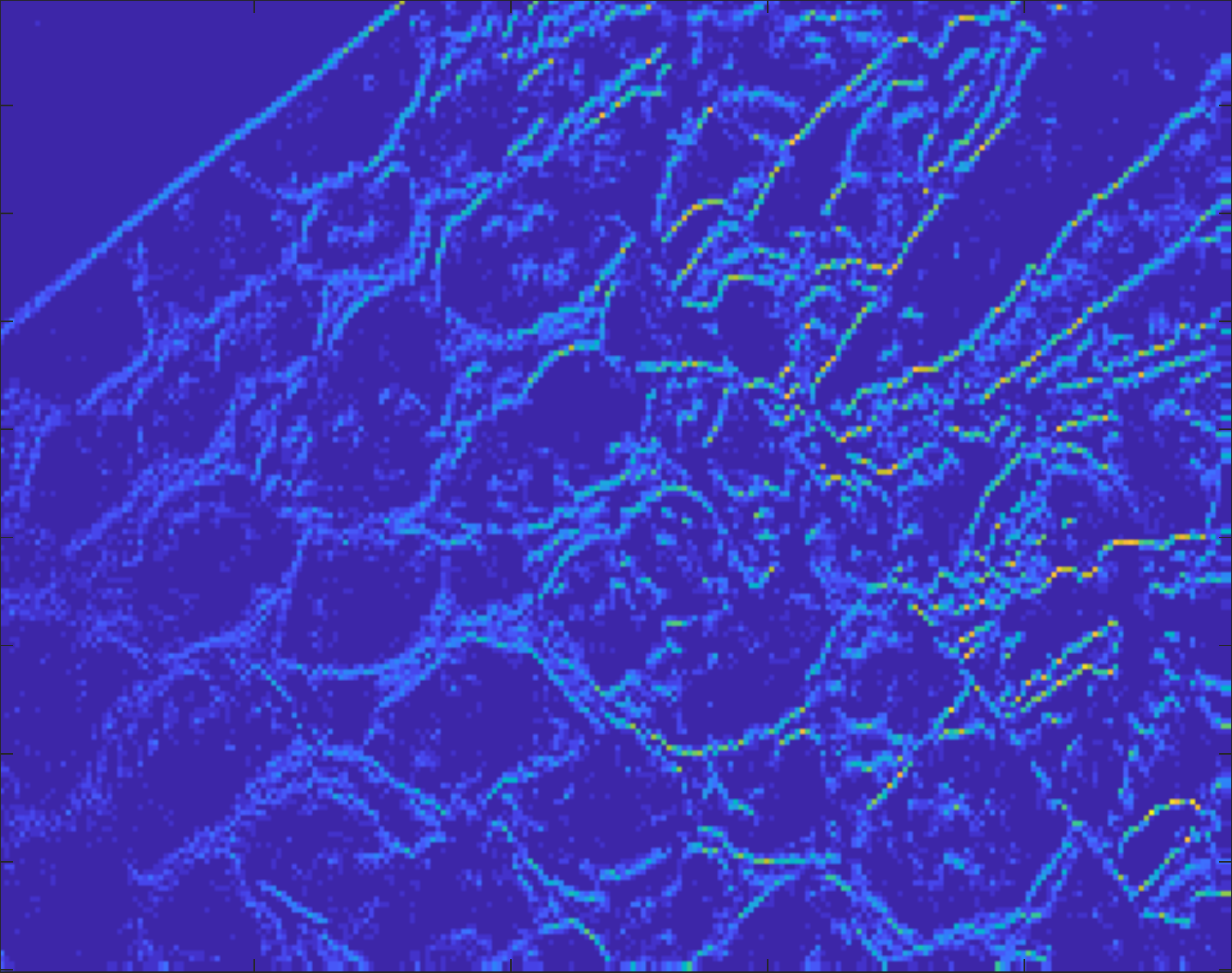}\\
			\hline
			&&&&\\[-.15cm]
			\RotText{~~ Subseq 2}
			& \includegraphics[width=\linewidth]{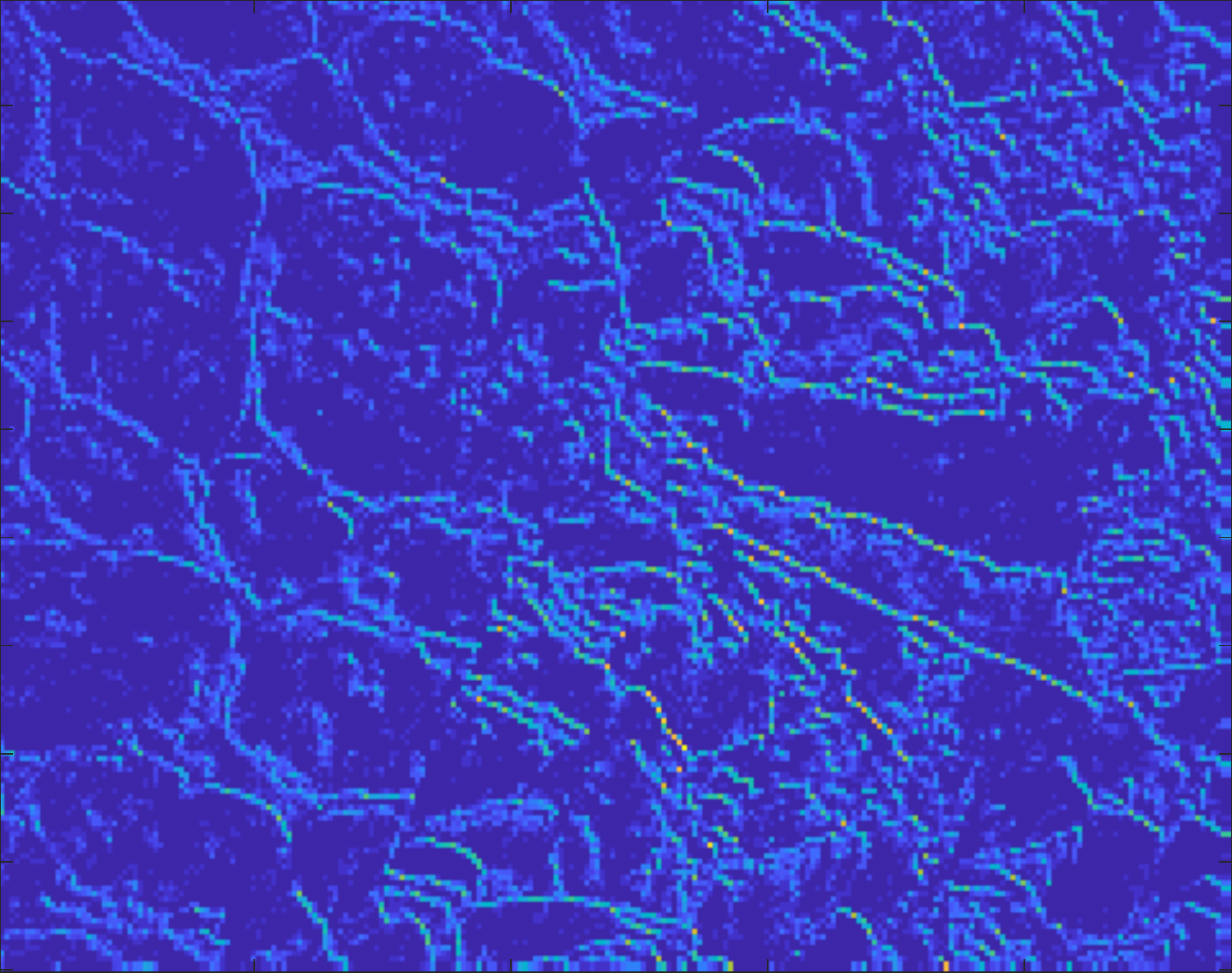}  
			& \includegraphics[width=\linewidth]{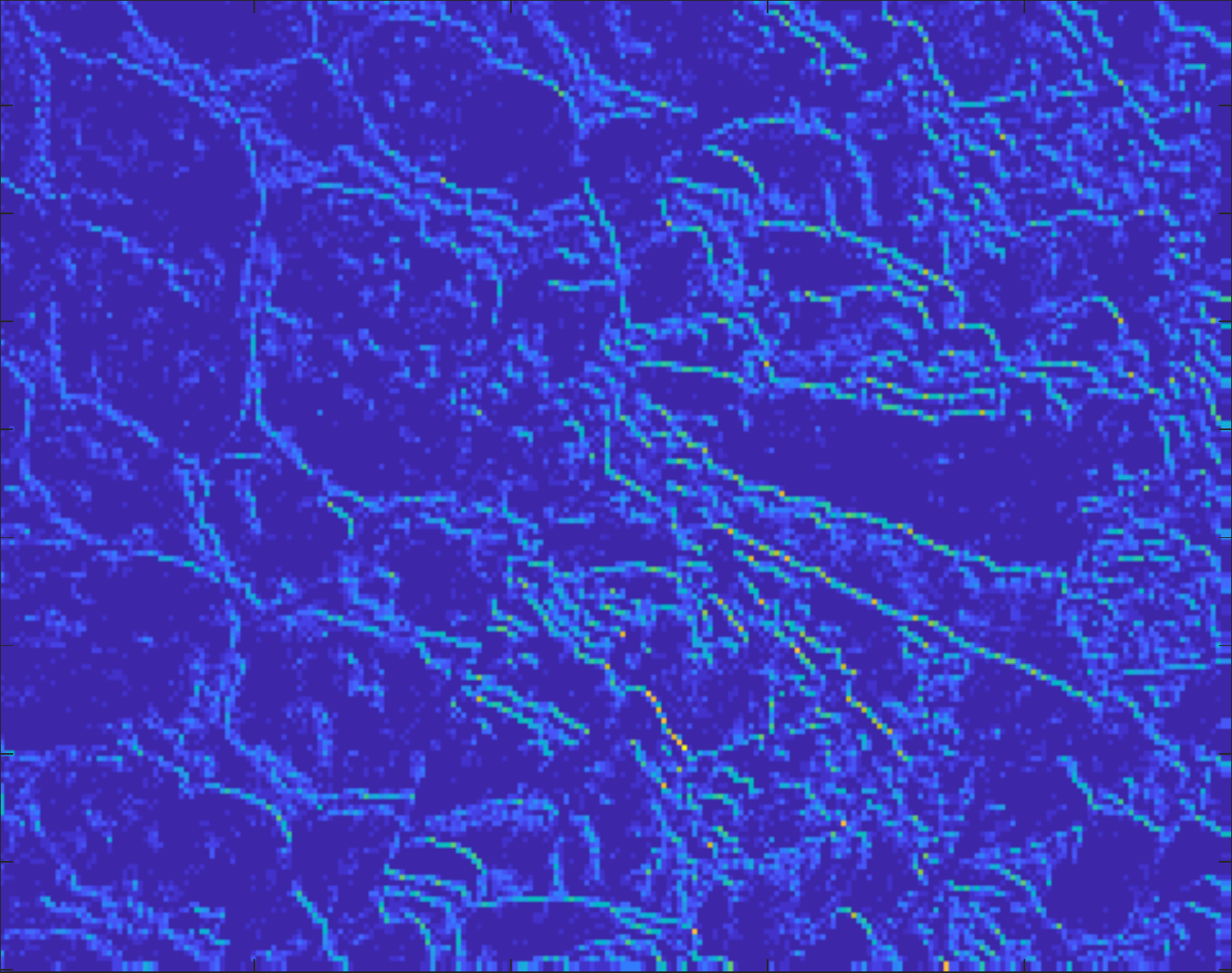}
			& \includegraphics[width=\linewidth]{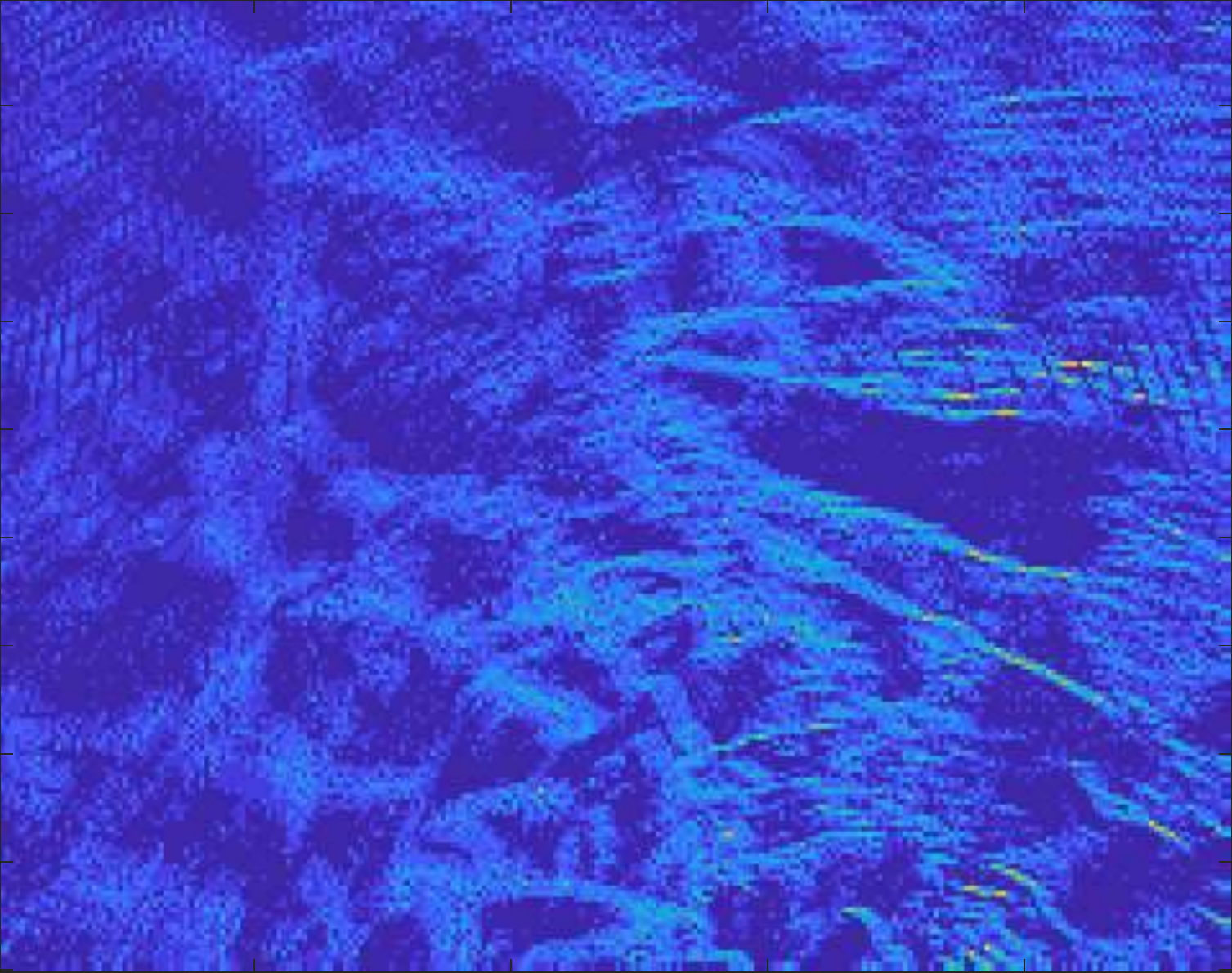}
			& \includegraphics[width=\linewidth]{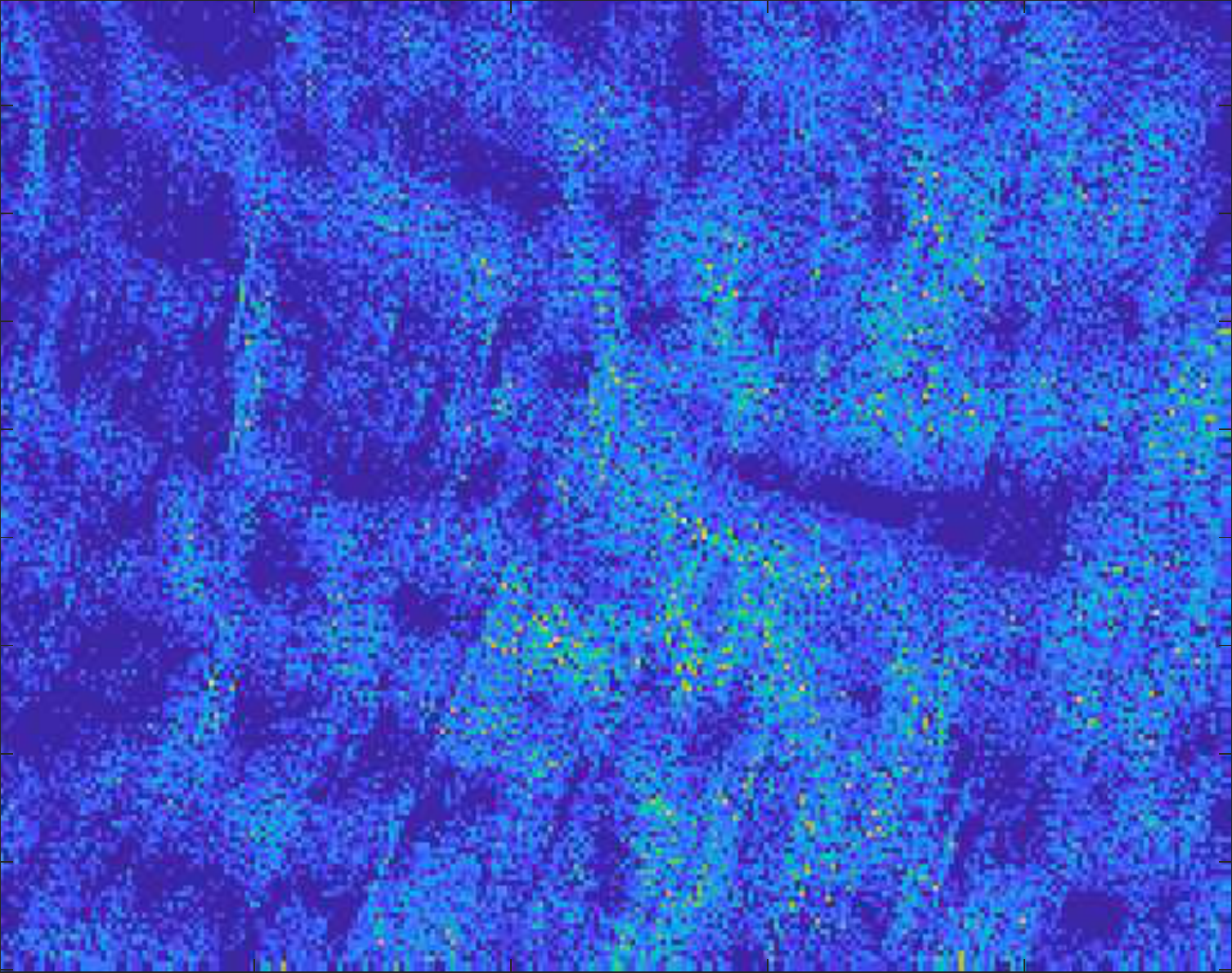}\\
			\hline
			&&&&\\[-.15cm]
			\RotText{~~ Subseq 3}
			& \includegraphics[width=\linewidth]{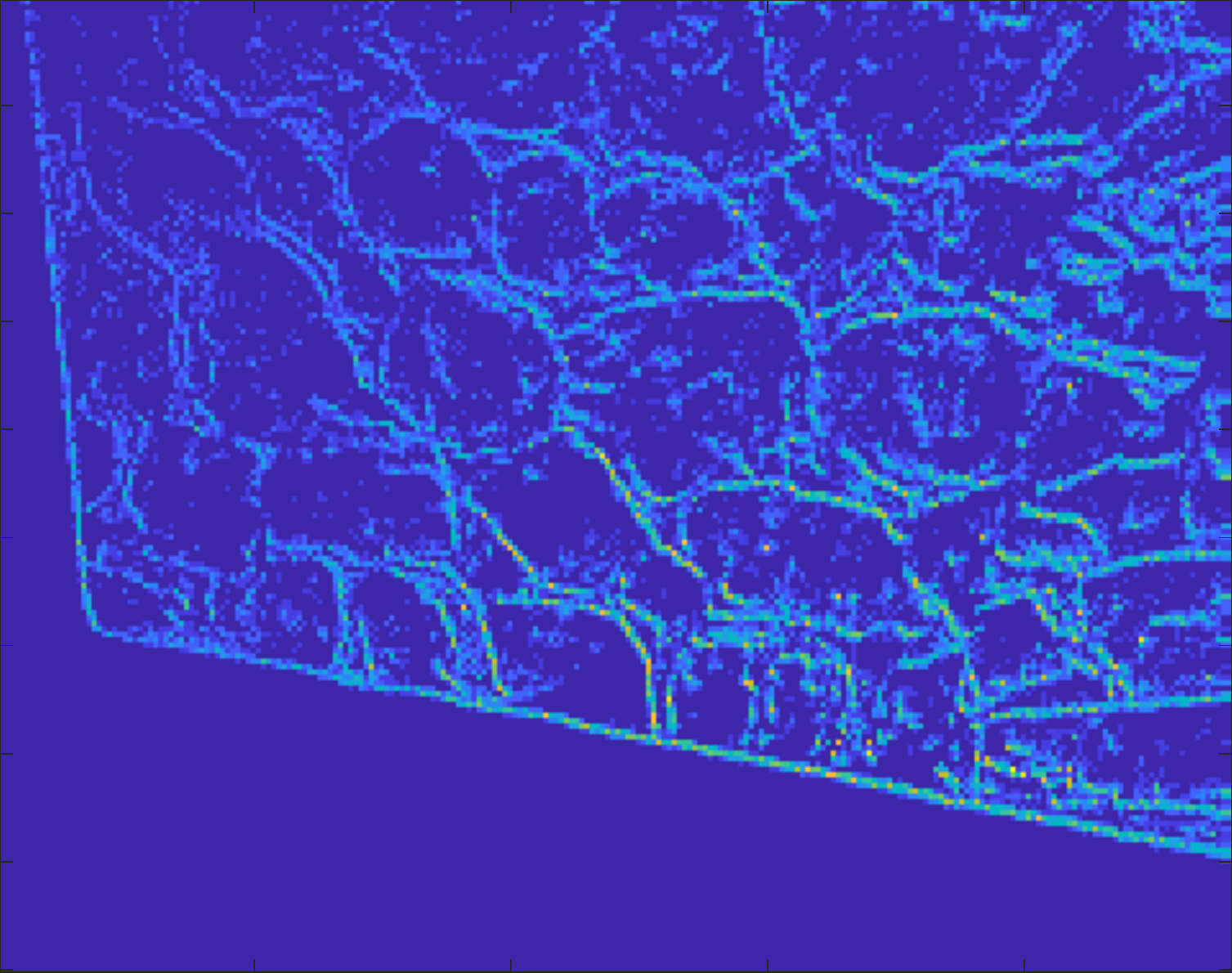}  
			& \includegraphics[width=\linewidth]{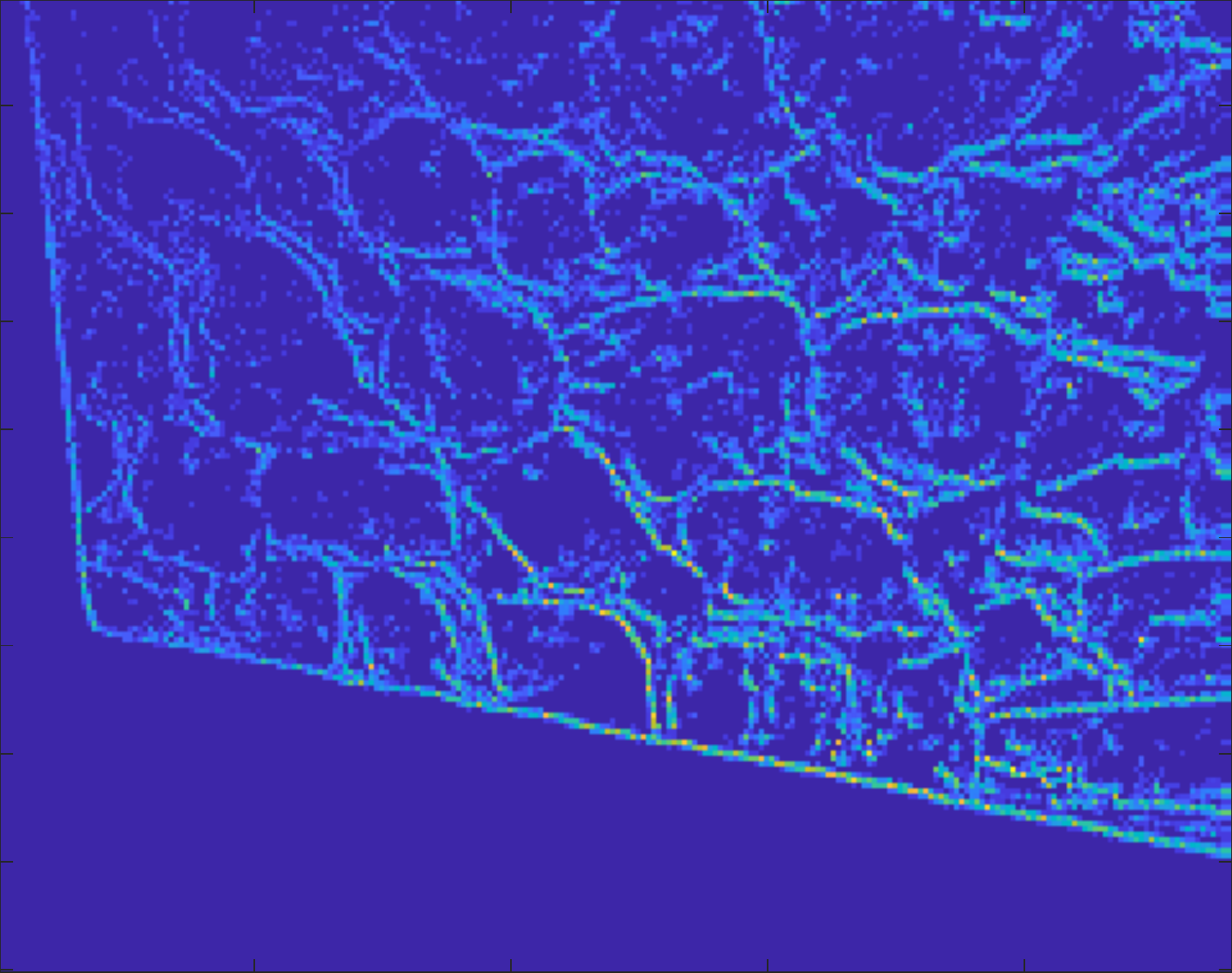}
			& \includegraphics[width=\linewidth]{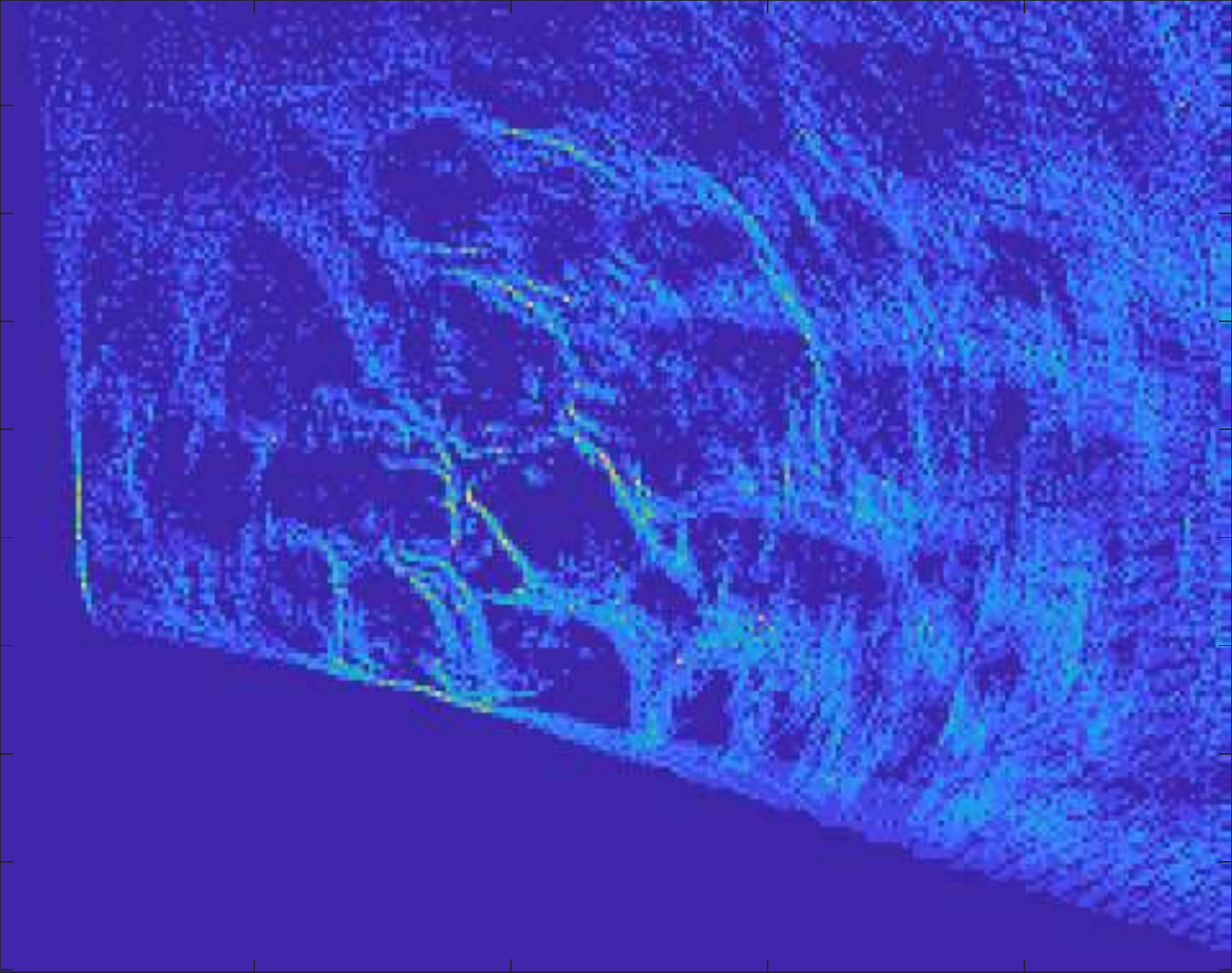}
			& \includegraphics[width=\linewidth]{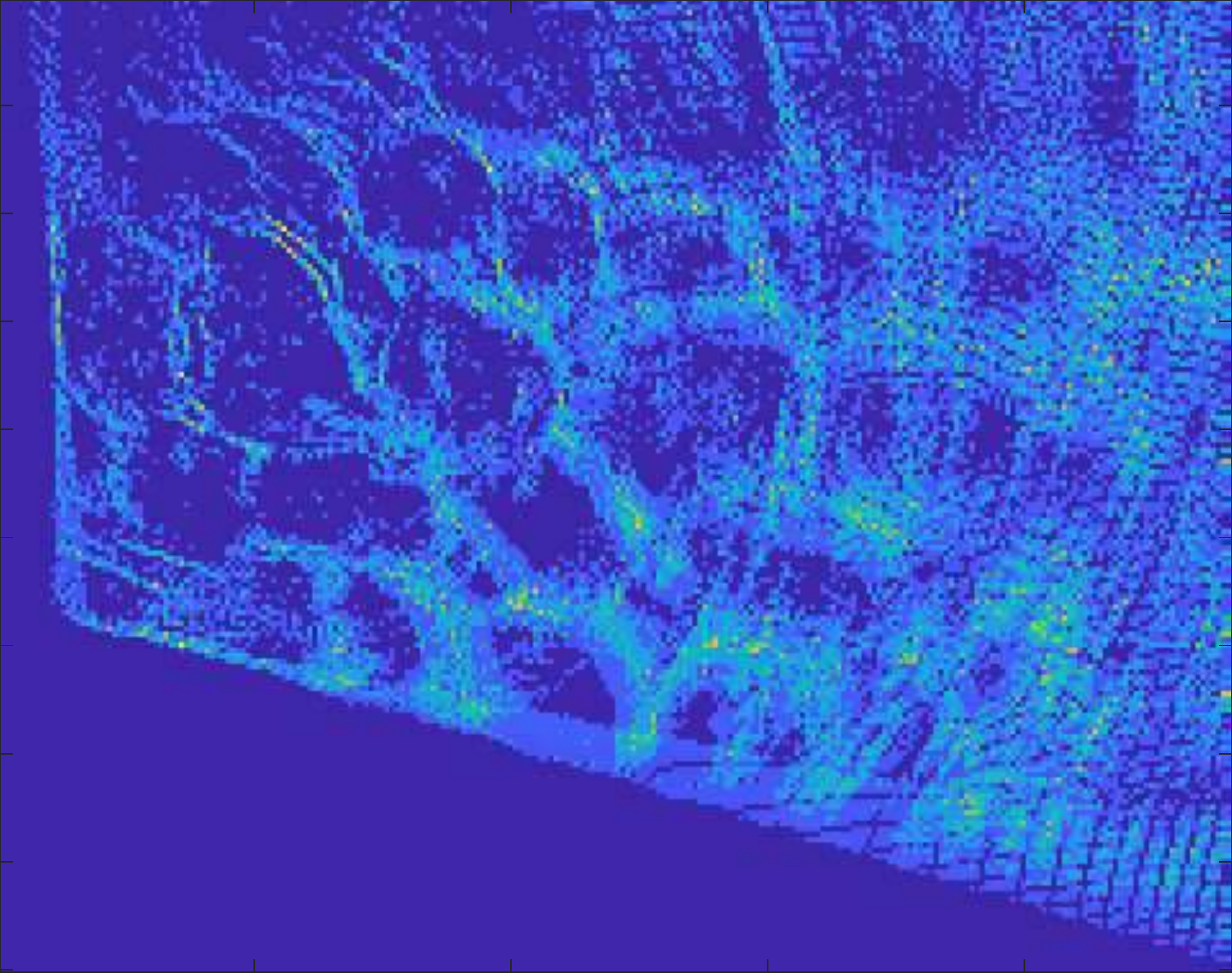}\\
			\hline
			&&&&\\[-.15cm]
			\RotText{~~ Subseq 4}
			& \includegraphics[width=\linewidth]{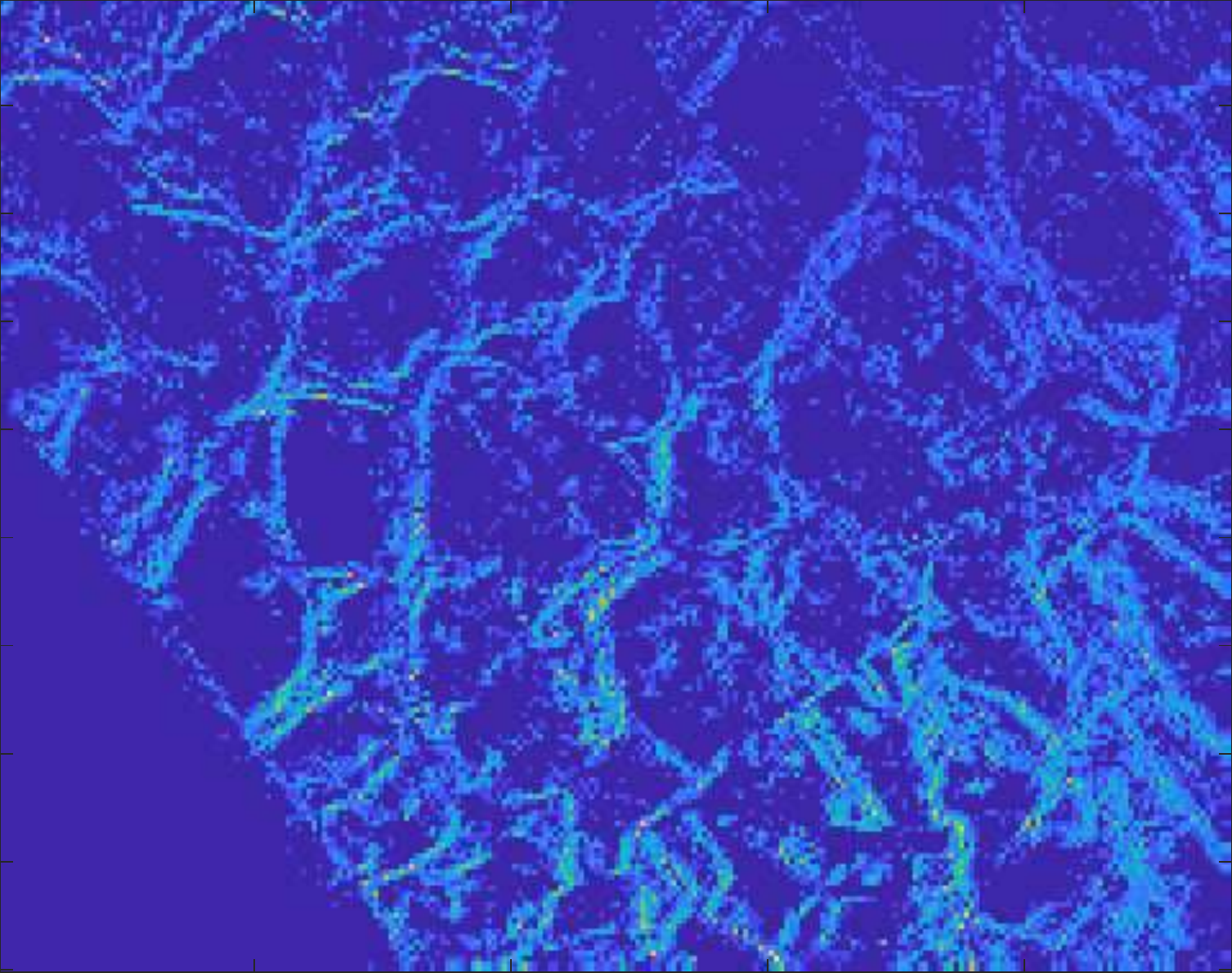}  
			& \includegraphics[width=\linewidth]{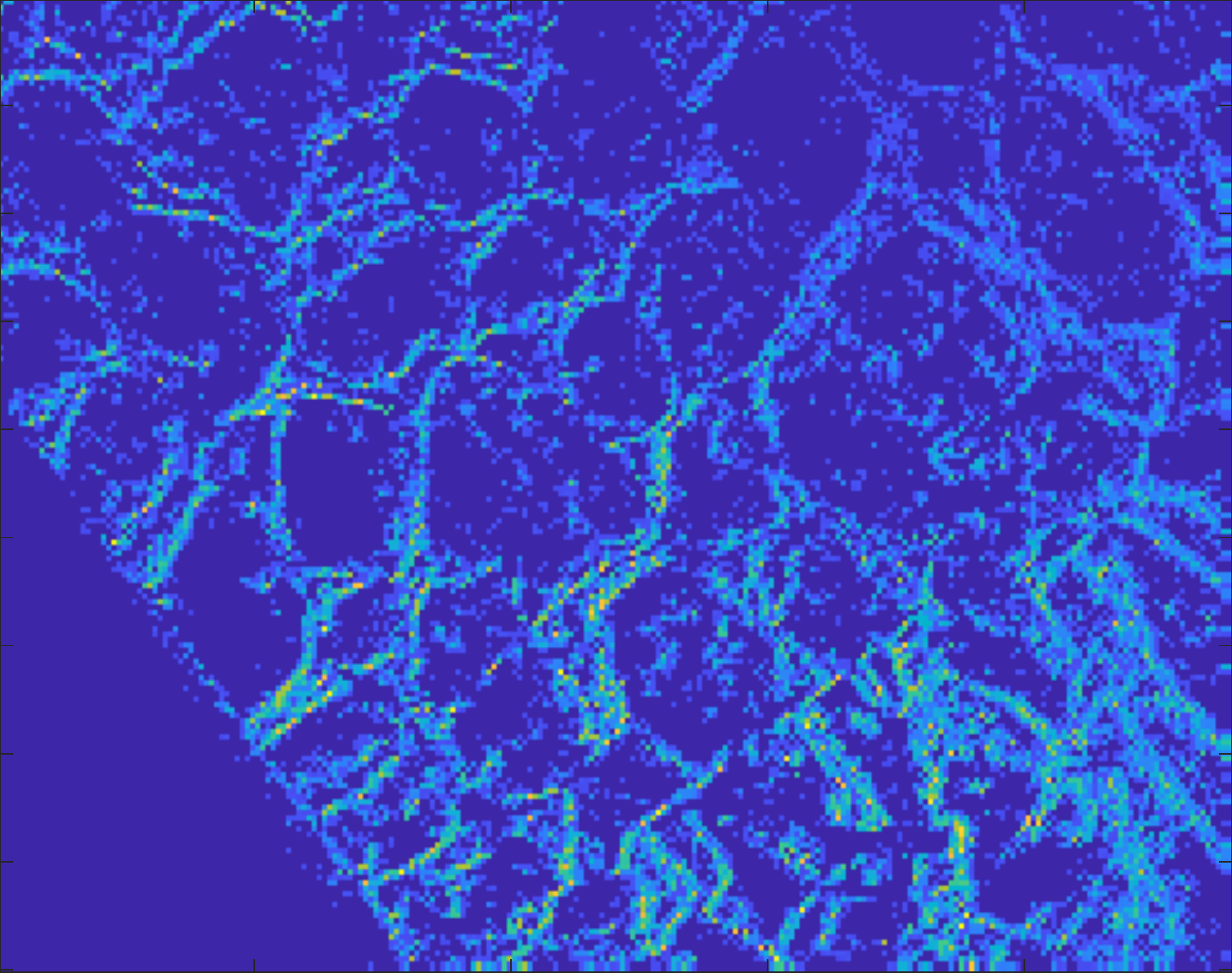}
			& \includegraphics[width=\linewidth]{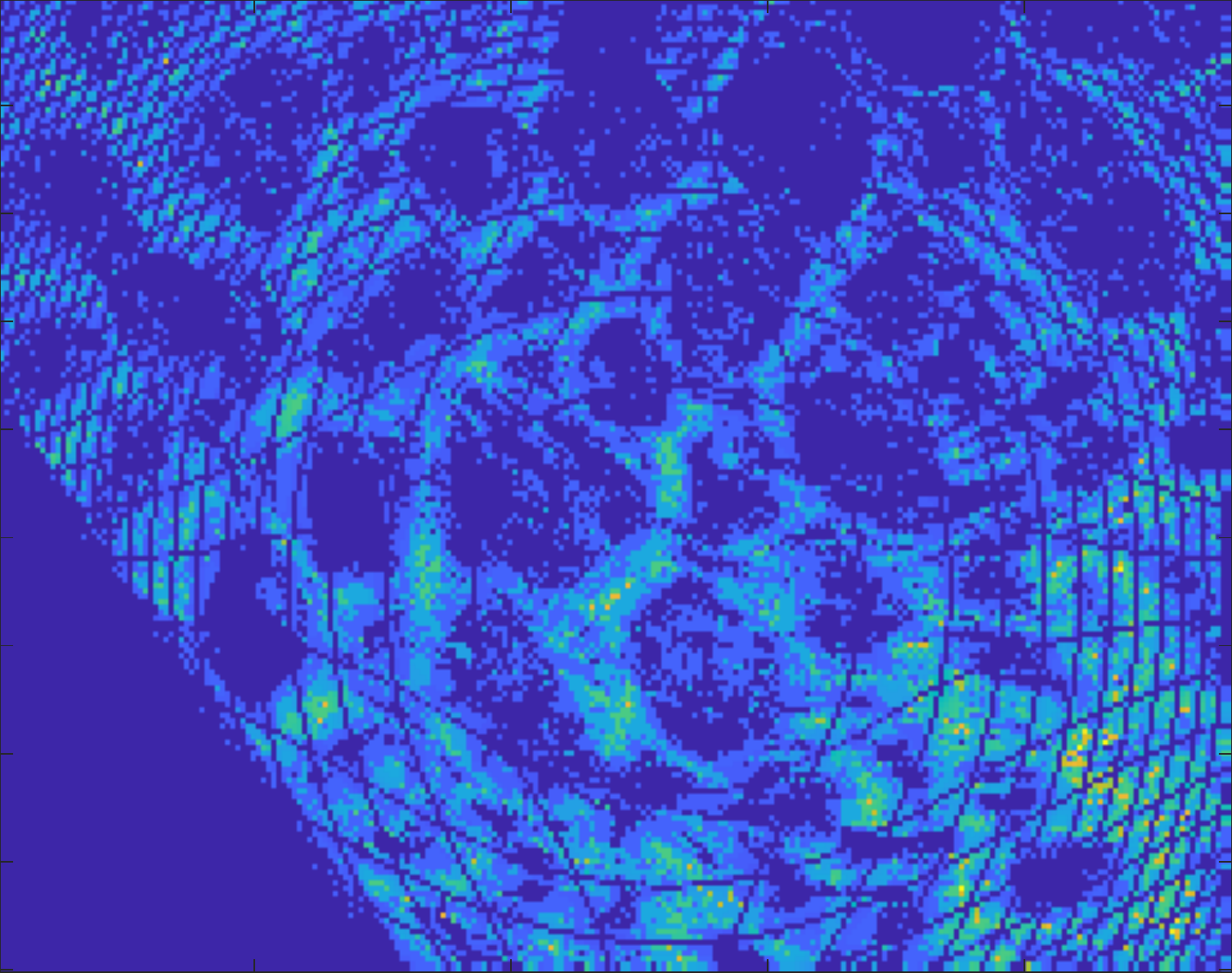}
			& \includegraphics[width=\linewidth]{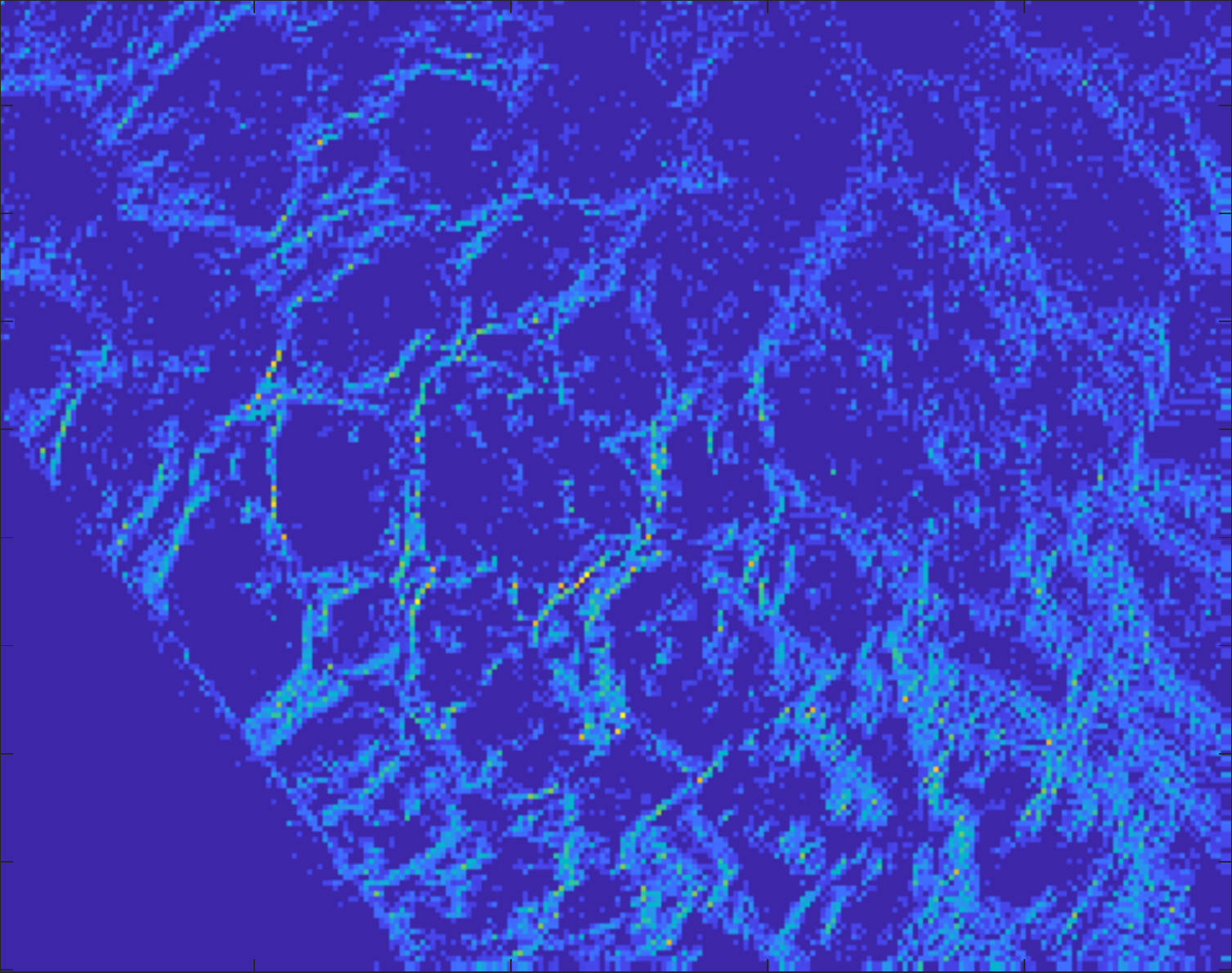}\\
			\hline
			&&&&\\[-.15cm]
			\RotText{~~ Subseq 5}
			& \includegraphics[width=\linewidth]{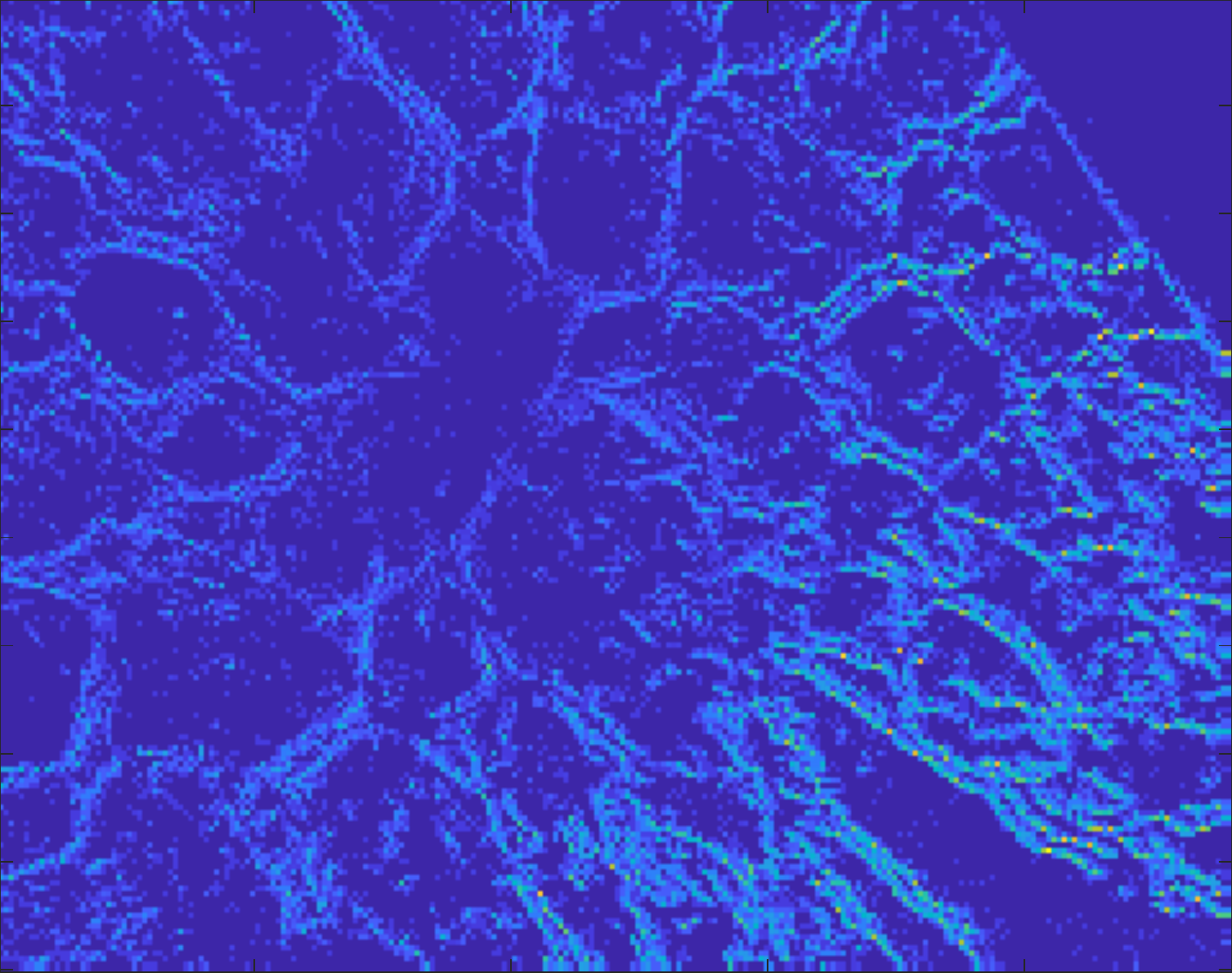}  
			& \includegraphics[width=\linewidth]{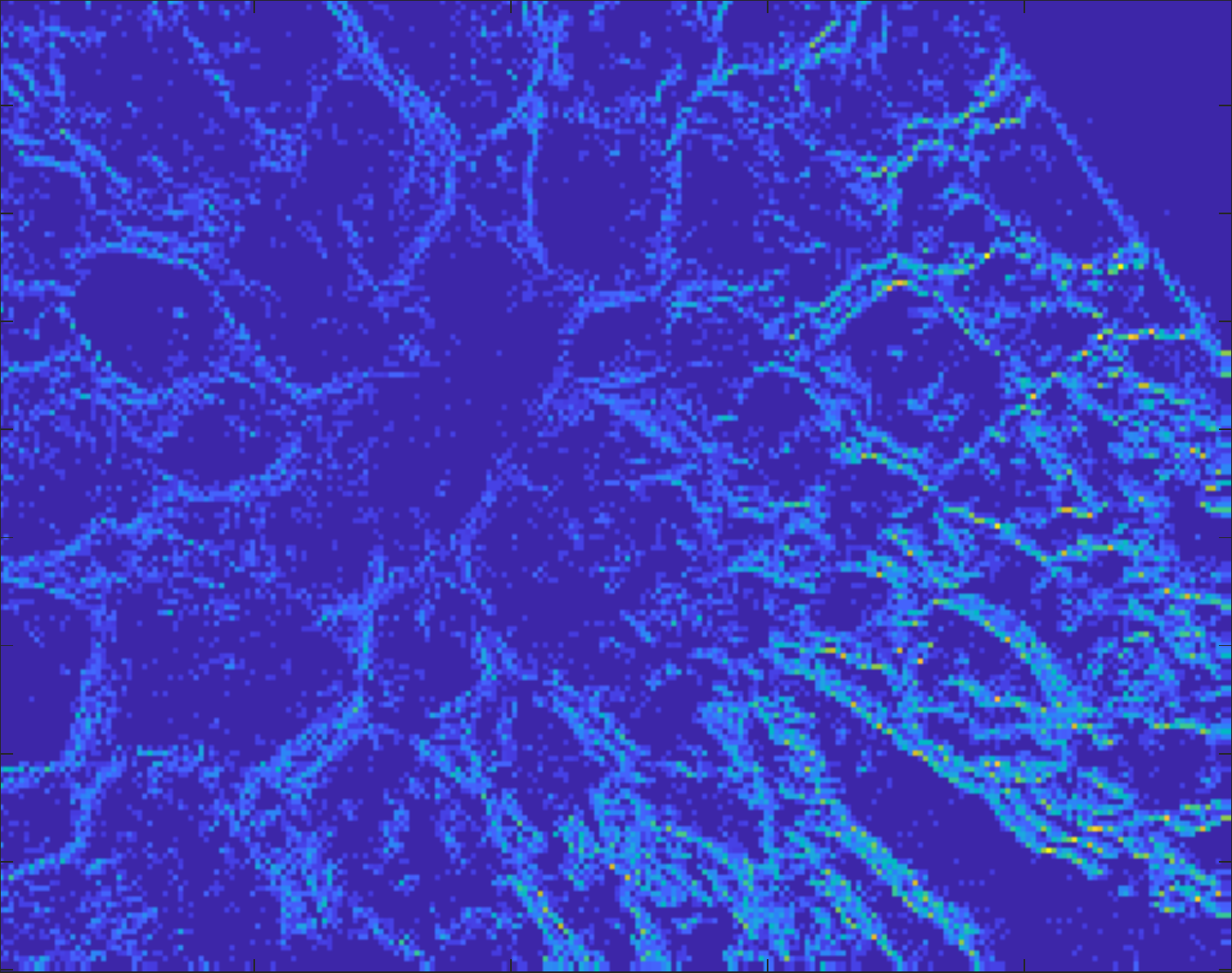}
			& \includegraphics[width=\linewidth]{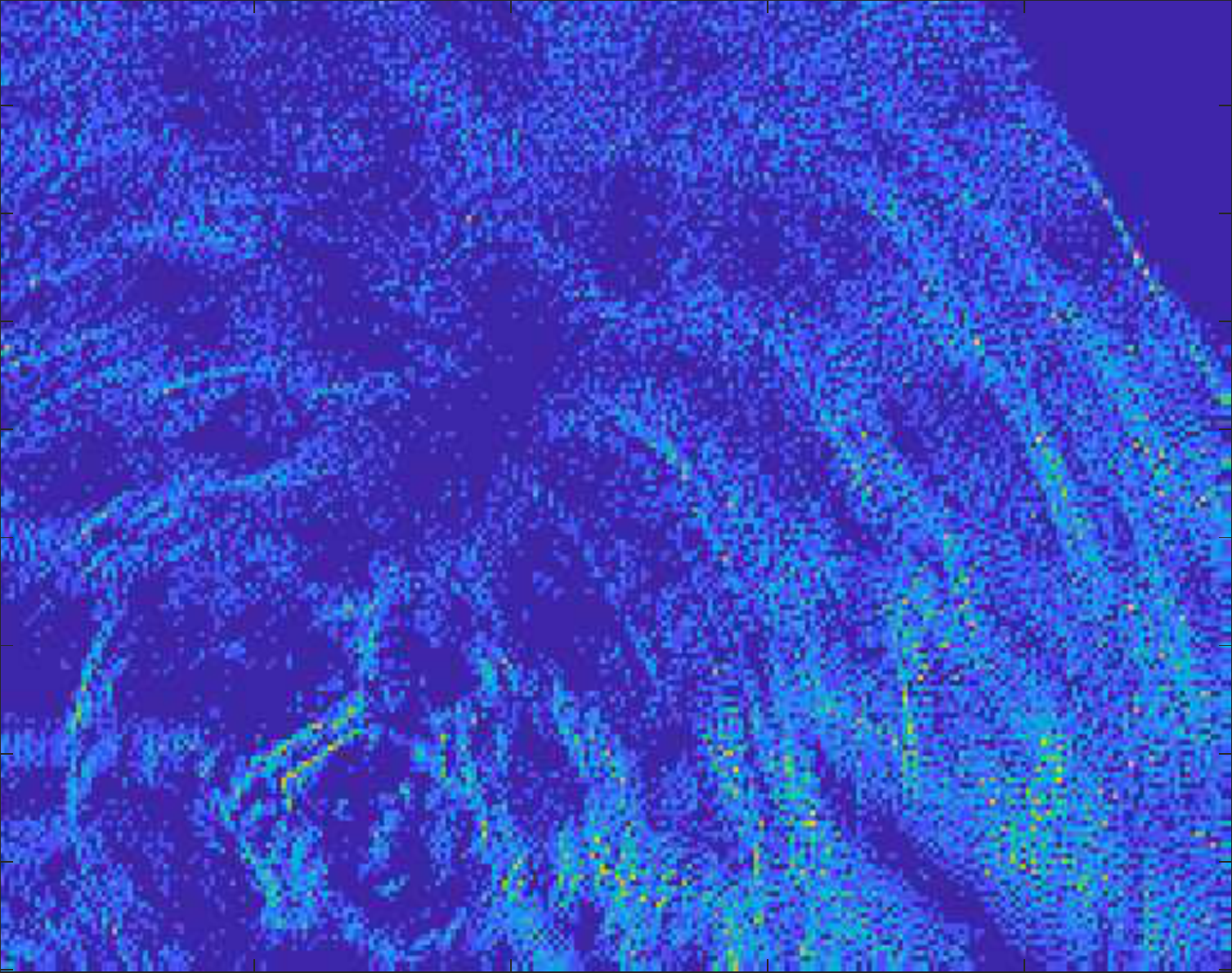}
			& \includegraphics[width=\linewidth]{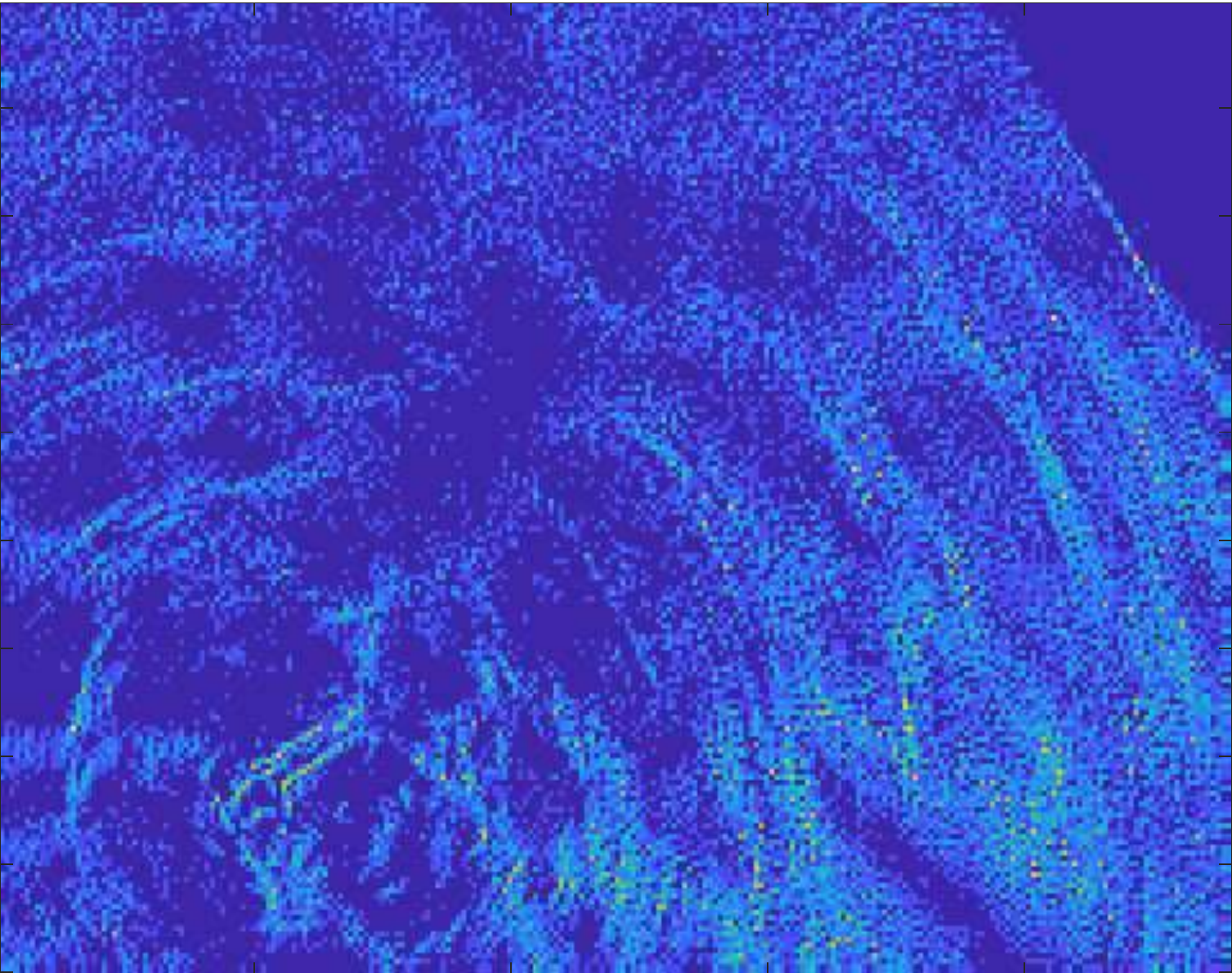}\\
		\end{tabularx}
		\caption{Qualitative results (motion compensated event images) for \emph{poster}.}
		\label{fig:poster}
	\end{figure}
	\clearpage
	\begin{multicols}{2}
		
		{\small
			\bibliographystyle{ieee_fullname}
			\bibliography{references}
		}
	\end{multicols} %

\end{document}